%% file: main.tex
\documentclass[runningheads]{llncs}

\usepackage{eccv}

\usepackage{eccvabbrv}
\usepackage{wrapfig}
\usepackage{multirow}
\usepackage{url}
\usepackage{graphicx}
\usepackage{booktabs}
\usepackage{tcolorbox}
\tcbuselibrary{breakable}

\usepackage[accsupp]{axessibility}  
\usepackage{bm}

\newcommand{\vx}{\mathbf{x}}
\newcommand{\va}{\mathbf{a}}
\newcommand{\vr}{\mathbf{r}}
\newcommand{\vomega}{\bm{\omega}}
\newcommand{\inner}[2]{\langle #1, #2 \rangle}

\usepackage{hyperref}

\usepackage{orcidlink}

\begin{document}

\title{Disentangling Hallucinations: Orthogonal Semantic Projection for Robust Interpretability}

\titlerunning{Orthogonal Semantic Projection for Robust Interpretability}

\author{Emirhan Bilgi{\c{c}}\inst{1,2} \and
Baptiste Caramiaux\inst{2} \and
Zhi Yan\inst{1} \and
Gianni Franchi\inst{3}}

\authorrunning{E.~Bilgi{\c{c}} et al.}

\institute{U2IS, ENSTA, Institut Polytechnique de Paris, Palaiseau \and
ISIR, Université Sorbonne, Pierre et Marie Curie, Paris \and
AMIAD, Pôle Recherche, Palaiseau}

\maketitle

\begin{abstract}
As Vision-Language Models are increasingly deployed in safety-critical applications, the trustworthiness of their explanations becomes crucial. Explainable AI (XAI) methods for Vision-Language Models often suffer from \textit{semantic hallucination}, where attribution maps highlight prominent image regions even when prompted with incorrect text descriptions (e.g., highlighting a dog when prompted ``cat''). Although this problem is widespread, a formal mathematical analysis of XAI methods and CLIP embeddings is largely missing in the literature. We demonstrate that this phenomenon is not specific to a single architecture but is a fundamental consequence of Linear Semantic Leakage in high-dimensional embedding spaces. We propose a unified theoretical framework, Linear Semantic Attribution (LSA), which generalizes across discriminative methods. We introduce OSP, a geometric intervention that utilizes the residual property of OMP to disentangle unique semantic signals from shared concepts. We prove theoretically and demonstrate empirically that OSP minimizes hallucination by orthogonalizing the query vector against distractor concepts, rendering the attribution model blind to shared features while preserving fidelity for correct prompts. Our code is available at: \url{https://github.com/emirhanbilgic/Orthogonal-Semantic-Projection}
\end{abstract}

\section{Introduction}
\label{sec:intro}

Vision-Language Models (VLMs) and multimodal diffusion models have become central to modern computer vision. Well-known examples include CLIP~\cite{radford2021learning} and SigLIP~\cite{zhai2023sigmoid} for image--text understanding, as well as generative architectures such as Stable Diffusion~\cite{rombach2022high}, Flux~\cite{esser2024scaling}, and PixArt-$\alpha$~\cite{chen2024pixart}. These systems have enabled significant progress in areas such as autonomous driving~\cite{zhou2024vision}, medical image synthesis~\cite{kazerouni2023diffusion}, and zero-shot robotics~\cite{nasiriany2024pivot}. However, as they move from research prototypes to real-world deployment, a critical challenge emerges: they are difficult to interpret. Their internal representations function largely as ``black boxes,'' undermining trust, especially in safety-critical environments.

To better understand these models, researchers commonly apply post-hoc explainability methods originally developed for standard Deep Neural Networks (DNNs). These methods produce saliency maps or attribution masks that highlight important regions in an image. Yet when applied to multimodal models, a new and severe problem appears: \textit{semantic hallucination}. Saliency methods often highlight objects that are completely absent from the image, for instance, highlighting a dog when prompted with ``cat.'' This undermines the trustworthiness of Explainable AI (XAI) and, more broadly, reduces the possibility of effectively debugging AI models and makes AI less safe.

\begin{figure}[t]
\centering
\includegraphics[width=0.95\textwidth]{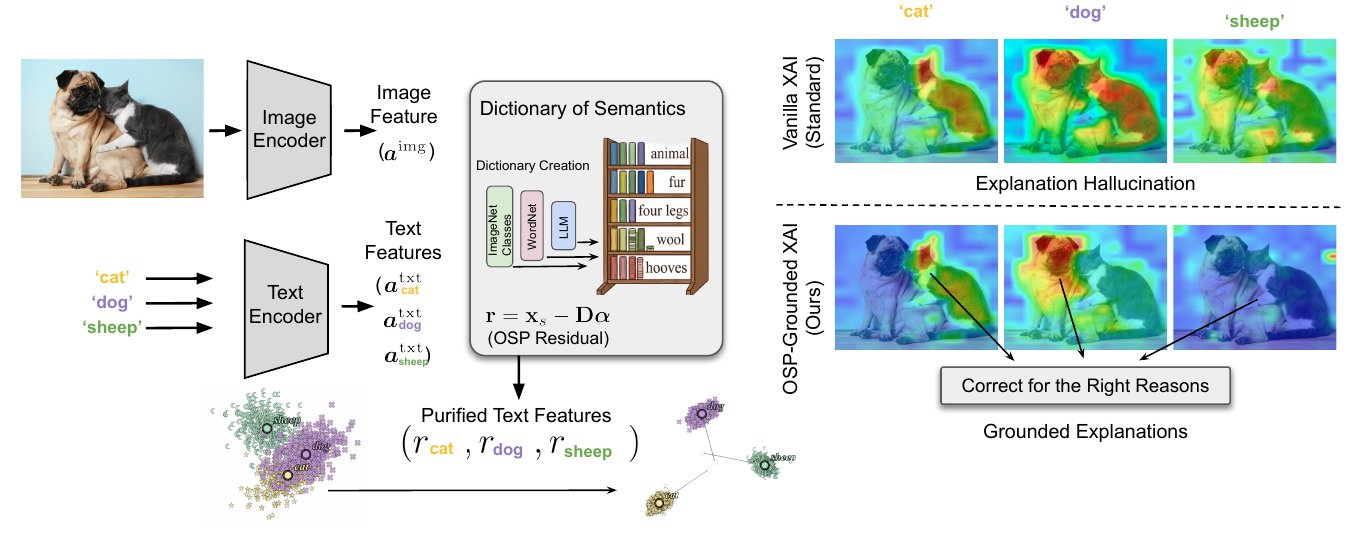}
\caption{Overall pipeline of OSP. Our method performs a geometric intervention in the multimodal embedding space to disentangle shared semantic components that cause hallucination, resulting in more faithful and grounded explanations.}
\label{fig:pipeline}

\end{figure}

Hallucination is commonly defined as generated content that is incorrect or not grounded in the input~\cite{huang2025survey}. In this work, we show that hallucination does not only affect model \emph{outputs}: it also affects \emph{explanations}. We define \textbf{explanation hallucination} as a situation where a saliency map highlights image regions that are not truly related to the given text prompt, as illustrated in Fig.~\ref{fig:pipeline}, where the explanation hallucinates the presence of sheep. This problem is critical because the explanation appears convincing but is not correctly grounded in the text: the model may seem correct, but it is ``right for the wrong reasons.''
\begin{wrapfigure}{r}{0.55\textwidth}
\vspace{-20pt}
\centering
\includegraphics[width=\linewidth]{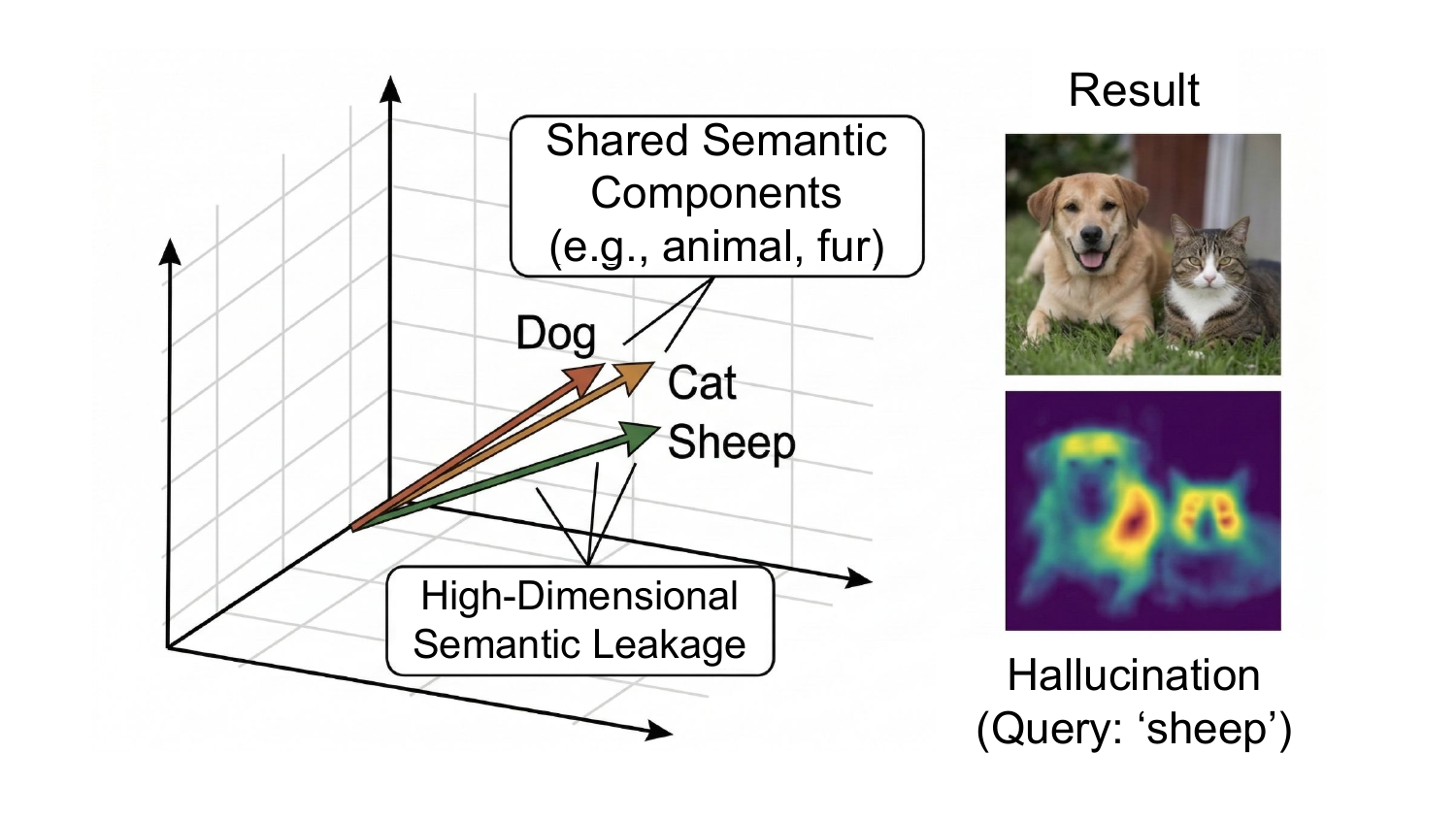}
\caption{\scriptsize 3D text embeddings and hallucination. Without disentanglement, the distractor is "distracting" the heatmaps.}
\vspace{-20pt}
\label{fig:3Dtextembedding}
\end{wrapfigure}

We argue that this issue stems from the geometry of the multimodal embedding space. Image and text features are represented in a shared high-dimensional space. Because these embeddings are not orthogonal, semantically related but distinct concepts share common directional components (see Fig.\ref{fig:3Dtextembedding}). 
When saliency maps are computed using these embeddings, these shared components cause the method to highlight irrelevant regions, and especially "wrong objects": producing hallucinated explanations. While not addressing the hallucination of attribution maps, there have been multiple works on image embeddings~\cite{papadimitriou2025interpreting, fel2025archetypal, pach2025sparse}, but to the best of our knowledge, we are the first to provide such work on text embedding to improve attribution XAI methods.

In this paper, we introduce \textbf{Orthogonal Semantic Projection (OSP)}, a geometric intervention designed to reduce hallucinations in saliency-based explanations. OSP is grounded in a simple geometric idea: before computing the saliency map, we remove the shared semantic components in the text embedding that cause hallucination.

To achieve this, we use a dictionary learning framework based on Orthogonal Matching Pursuit (OMP)~\cite{pati1993orthogonal}. This sparse decomposition allows us to isolate and remove the parts of the embedding that introduce semantic leakage. We then compute the saliency map using a refined, orthogonalized representation of the text query. Importantly, OSP acts as a \textbf{generalized framework} and is \textit{plug-and-play}: it can be applied to a wide range of Vision-Language models and saliency methods without retraining the underlying model.

Crucially, this represents a fundamentally new paradigm for interpretability: instead of decomposing the image embedding space~\cite{textspan2024} to find visual concepts, we mathematically decompose the \textit{text} embedding space to isolate pure semantic directions. We summarize our contributions as follows: 
\begin{enumerate}
    \item \textbf{Geometric Analysis of Hallucination:} We provide a formal mathematical framework that bridges the gap between grounding failures and hallucinatory outputs, characterizing hallucination as a function of the latent misalignment (i.e.\ non-orthogonality) between semantic embedding directions and offering a geometric interpretation of how grounding errors propagate into nonsensical attributions.
    \item \textbf{A Universal XAI Technique:} OSP is compatible with multiple architectures, including CLIP, SigLIP, and diffusion-based models, and can be integrated with different saliency methods as a plug-and-play module.
    \item \textbf{Extensive Validation:} We evaluate OSP across 3 foundational models and 5 state-of-the-art XAI methods, demonstrating consistent improvements.
    \item \textbf{Practical Use Cases and User Study:} We demonstrate improved interpretability reliability, backed by empirical results and a user study.
\end{enumerate}

\section{Related Work}

\subsection{Vision-Language Encoder Attribution Methods}
Attribution methods and saliency maps aim to demystify the ``black box'' nature of Vision-Language Models (VLMs) by grounding textual concepts in visual regions. Discriminative approaches adapt gradient-based strategies to multimodal architectures. For instance, GradCAM \cite{selvaraju2017grad} has been successfully applied to CLIP to identify visual features that maximize similarity with specific text queries. For Vision Transformers (ViTs), methods like LeGrad \cite{legrad2025} provide feature formation sensitivity by tracing gradients to internal patch embeddings, while CheferCAM \cite{chefer2021} directly propagates relevance scores using attention maps and their gradients. Generative models, such as Stable Diffusion, rely on methods like DAAM \cite{tang-etal-2023-daam}, which aggregates cross-attention maps to spatially localize the influence of prompt words. Recent works such as TextSpan \cite{textspan2024} evaluate concept attribution by decomposing image representations into text-based constituents. However, unlike these existing explainability approaches that accept raw entangled query embeddings and refine spatial attribution post-hoc or within bottlenecks, our method performs a structured geometric intervention directly in the continuous representation space prior to attribution to proactively resolve spatial leakage caused by shared semantic similarities.

\subsection{Hallucination in Vision-Language Models}
Despite their remarkable zero-shot capabilities, VLMs and Large Vision-Language Models (LVLMs) frequently suffer from different types of hallucinations.
The severity of this issue has motivated the creation of targeted evaluation benchmarks. Several works tried to address object hallucination in the Visual Question Answering (VQA) setting: the CHAIR metric \cite{rohrbach2018object} quantifies object hallucinations in image captioning, the POPE benchmark \cite{li2023evaluating} evaluates object hallucination in LVLMs through yes-no questions, and THRONE \cite{kaul2024throne} benchmarks hallucination in free-form generations. To mitigate concept hallucination within Concept Bottleneck Models (CBMs), Kazmierczak et al. \cite{kazmierczak2025enhancing} introduce CHILI, a technique that disentangles image embeddings to localize concept pixels. However, other benchmarks focus on VQA, and CHILI focuses on concept prediction within CBMs and does not address the hallucination of attribution maps. Our method specifically targets \textit{object attribution hallucination}. While these diagnostic tools highlight the prevalence of hallucination as an empirical phenomenon to be benchmarked or mitigated via targeted disentanglement, our work uniquely establishes a formal, mathematical link connecting the root cause of these errors backward into the linear geometry of the multimodal representation space itself,  while it can be easily integrated with existing attribution methods.

\subsection{Representation Geometry of Vision-Language Encoders}

Emerging research characterizes the high-dimensional latent spaces of networks like CLIP not as uniform structures, but as complex geometric manifolds. The Linear Representation Hypothesis \cite{park2023linear} suggests that high-level concepts and hierarchical structures are encoded as linear directions. Furthermore, the Platonic Representation Hypothesis \cite{huh2024platonic} posits that diverse architectures naturally converge toward a shared geometric model of reality. Within contrastive models like CLIP, geometric phenomena such as the ``modality gap'' \cite{liang2022mind} reveal that image and text embeddings reside on separated, offset hyperspheres, severely restricting the angular portion of the space utilized by the model. To disentangle these compressed spaces, Sparse Autoencoders (SAEs) \cite{fel2025archetypal} have recently been employed to extract monosemantic, interpretable directions from dense activations. While prior works focus on analyzing structures like the modality gap or sparse features, OSP uniquely utilizes these geometric insights as a mitigation tool, strictly dynamically orthogonalizing directions at inference time to correct flawed model behavior and concept bleeding.

\section{The Geometry of Hallucination}
\subsection{Background on Saliency Map Methods for Multimodal Models}
\textit{Notations.} We consider a DNN denoted by $f_{\vomega}(\cdot)$, where $\vomega$ represents the model parameters. In this paper, we focus on multimodal models that take two inputs: an image $\vx^{\text{img}}$, and a text prompt $\vx^{\text{txt}}$.
In diffusion-based 
models, the image input 
initially consist of noise that is progressively refined during generation.

\textit{Feature Extractors and Activations.}
We define the saliency maps based on two distinct categories of model activations: \textbf{(1) Forward Activations:} Let $\mathbf{A}^{\text{MHA}}_{l,h}(\vx^{\text{img}})$ denote the activation map corresponding to the $h$-th attention head in the $l$-th layer of the Transformer. Similarly, let $\mathbf{A}^{\text{MLP}}_{l}(\vx)$ represent the output activations of the Multi-Layer Perceptron (MLP) sub-layer at layer $l$. We denote the complete set of forward activations as $ \{ \mathbf{A}^{\text{MHA}}_{l,h}(\vx^{\text{img}}), \mathbf{A}^{\text{MLP}}_{l}(\vx^{\text{img}}) \}_{l,h=1}^{L,H}$.
\textbf{(2) Backward Activations:} To incorporate local sensitivity and importance weighting, we utilize gradient-based signals derived from the backward pass. Let $S_c = \langle \va^{\text{txt}}_c, \va^{\text{img}} \rangle$ be the cosine similarity score for class or concept $c$. Following the formulation in LeGrad \cite{legrad2025} and CheferCAM \cite{chefer2021}, we define the backward activation set as $ \{ \nabla_{\mathbf{A}^{\text{MHA}}_{l,h}} S_c \}_{l,h=1}^{L,H}$, where $\nabla_{\mathbf{A}^{\text{MHA}}_{l,h}} S_c$ represents the gradient of the class similarity score with respect to the $h$-th attention head in the $l$-th layer.

\textit{Text and Image Embeddings.}
Let us consider a set of textual classes or prompts:
\(
\{ \vx^{\text{txt}}_c \}_c.
\)
Each text input is mapped to an embedding:
\(
\{ \va^{\text{txt}}_c \}_c, \quad \va^{\text{txt}}_c \in \mathbb{R}^h,
\)
where $h$ is the embedding dimension. Similarly, we denote by $\va^{\text{img}}_j(\vx^{\text{img}})$ the image embedding extracted at layer $j$. Depending on the architecture, for example, for CLIP or SigLIP, this may correspond to the class token, and for diffusion models, this may correspond to an averaged image token or latent representation.
For simplicity, we write $\va^{\text{img}}_j$. 

\textit{General Formulation of Saliency Maps.}
Most saliency methods aim to produce a single spatial importance map $\mathbf{V}_c(\mathbf{x})$ for a given class or prompt $c$. We observe that many methods can be written as a weighted combination of activation maps, where the weights depend on the text embedding or on its interaction with the image embedding. Formally, the saliency map can be expressed as:

\small
\begin{equation}
\begin{aligned}\label{eq:method_cases}
\mathbf{V}_c(\vx^{\text{img}})  = \\
    ~~ &\begin{cases}
        \mathbf{w}^c \, \mathbf{A}^{\text{MLP}}_{L}(\vx^{\text{img}})\text{ with } \mathbf{w}^c = \va^{\text{txt}}_c  , & \text{(\textbf{CAM})}, \\
        \mathbf{w}^c \, \mathbf{A}^{\text{MLP}}_{L}(\vx^{\text{img}})\text{ with } \mathbf{w}^c = \frac{\partial}{\partial \va^{\text{txt}}_c}\langle \va^{\text{txt}}_c, \va^{\text{img}}_d \rangle, & \text{(\textbf{GradCAM}~\cite{selvaraju2017grad})}, \\
        \sum_{l,h} \mathbf{w}_j \, \nabla_{\mathbf{A}^{\text{MHA}}_{l,h}} S_c  \text{ with } \mathbf{w}_j = 1/(HL) , & \text{(\textbf{LeGrad}~\cite{legrad2025})}, \\
        \left[ \prod_{l=1}^L \left( \mathbf{I} + \frac{1}{H}\sum_{h} \bigl(\nabla_{\mathbf{A}^{\text{MHA}}_{l,h}} S_c \odot \mathbf{A}^{\text{MHA}}_{l,h}(\vx^{\text{img}})\bigr)^{+} \right) \right]_{\text{CLS}} , & \text{(\textbf{CheferCAM}~\cite{chefer2021})}, \\
        \sum_{h} \mathbf{w}_j \bigl(\nabla_{\mathbf{A}^{\text{MHA}}_{L,h}} S_c \odot \mathbf{A}^{\text{MHA}}_{L,h}(\vx^{\text{img}})\bigr)^{+}, \text{ with } \mathbf{w}_j = 1/H , &\ \text{(\textbf{AttentionCAM}~\cite{chefer2020})}, \\
        \frac{1}{T} \sum_{t=1}^{T} \frac{1}{|\mathcal{H}|} \sum_{h \in \mathcal{H}} \mathbf{A}^{\text{cross}}_{(t,h)}(\vx^{\text{img}}, c), & \text{(\textbf{DAAM}~\cite{tang-etal-2023-daam})},
    \end{cases}
\end{aligned}
\end{equation}
\normalsize
where:  $(\cdot)^{+}$ denotes the ReLU activation function, $\odot$ denotes the element-wise Hadamard product, $\mathbf{A}^{\text{cross}}_{(t,h)}(\vx, c)$ refers to the cross-attention map for token $c$ at diffusion timestep $t$ and head $h$. This formulation shows that most multimodal saliency methods follow a similar structure where, before the aggregation, the multimodal representation can be viewed as living in a space of dimension $d \times H\times L$ (where $d$ is the embedding dimension), and then in a space of 1 dimension. Therefore, saliency extraction can be interpreted as a dimensionality reduction operation that maps this high-dimensional multimodal representation to a single scalar importance value per spatial location. This unified view will allow us to analyze how interactions between text and image embeddings can introduce semantic leakage and lead to hallucinated explanations.

\subsection{From Linear Attribution to Semantic Leakage}

\textit{Linear Dependence on Text Embeddings.} Most saliency methods for vision-language models can be written under the general form
\begin{equation}
\mathbf{V}_c(\vx)
=
f\big(\bar{\va}^{\text{img}}, \va^{\text{txt}}_c\big),
\end{equation}
where the operator $f(\cdot,\cdot)$ characterizes the interaction mechanism between the visual features and the text query, mapping the joint multimodal representation into the spatial saliency domain as described in Eq.~\eqref{eq:method_cases}.  In many practical instantiations (e.g., similarity-based scores or their gradients), this interaction behaves linearly with respect to the text embedding, or can be locally approximated as such (see Appendix~\ref{sec:linearity_appendix}). We rely on this assumption for the rest of this section.

\textit{Geometric Structure of the Embedding Space.} Consider two classes $A$ and $B$ with text embeddings
$\va^{\text{txt}}_A, \va^{\text{txt}}_B \in \mathbb{R}^d$.
In contrastive models, embeddings are normalized:
\begin{equation}
\|\va^{\text{txt}}_A\|_2
=
\|\va^{\text{txt}}_B\|_2
=
1.
\end{equation}

Let $\theta_{AB}$ denote the angle between them:
\begin{equation}
\langle
\va^{\text{txt}}_A,
\va^{\text{txt}}_B
\rangle
=
\cos(\theta_{AB}).
\end{equation}
Since semantically related concepts share directions in embedding space,
$\cos(\theta_{AB})$ is typically non-zero.
To empirically evaluate this, we computed the pairwise cosine similarities between the text embeddings of all 1,000 ImageNet classes. Across all 499,500 unique pairs, the similarity is never zero: the minimum similarity is 0.0967, with a mean of 0.5925, a median of 0.6082, and a maximum of 1.0000. This confirms that orthogonal concept embeddings are practically non-existent in standard VLM representation spaces. We denote by
$\va^{\text{txt}}_{B\perp A}$
the component of $\va^{\text{txt}}_B$
orthogonal to $\va^{\text{txt}}_A$.

\paragraph{Definition (Hallucinated Attribution).}

Let image $\vx$ belong to class $A$.
Let $\mathbf{V}_c(\vx)$ denote the saliency map extracted for class $c$.
We say that hallucination occurs if
\begin{equation}
\label{eq:hallucination_def}
mean_{\vx} \mathbf{V}_B(\vx)
\;\geq\;
mean_{\vx} \mathbf{V}_A(\vx),
\end{equation}
even though the image contains only class $A$.

\begin{lemma}[Linear Semantic Leakage]
Let image $\vx$ contain an object of class $A$. Let us denote $\va^{\text{img}}$ its image embedding.
If
\begin{equation}
\langle
\va^{\text{txt}}_A,
\va^{\text{txt}}_B
\rangle \neq 0,
\end{equation}
then the saliency map $\mathbf{V}_B(\vx)$
contains a component proportional to
$\mathbf{V}_A(\vx)$.
\end{lemma}

\begin{proof}

Because $\|\va^{\text{txt}}_A\|_2 = 1$,
we decompose $\va^{\text{txt}}_B$ via orthogonal projection:

\begin{equation}
\va^{\text{txt}}_B
=
\langle
\va^{\text{txt}}_B,
\va^{\text{txt}}_A
\rangle
\va^{\text{txt}}_A
+
\va^{\text{txt}}_{B\perp A}.
\end{equation}

Using the cosine identity:

\begin{equation}
\va^{\text{txt}}_B
=
\cos(\theta_{AB}) \,
\va^{\text{txt}}_A
+
\va^{\text{txt}}_{B\perp A}.
\end{equation}

All attribution methods considered in Eq.~\eqref{eq:method_cases}
are linear in the text embedding and can be written as

\begin{equation}
\mathbf{V}_c(\vx)
=
f \left( \bar{\va}^{\text{img}} , \va^{\text{txt}}_c  \right).
\end{equation}
where $\bar{\va}^{\text{img}}$ denotes the specific image representation (e.g., intermediate activations or normalized embeddings) utilized by the respective algorithm. The operator $f(\cdot, \cdot)$ characterizes the interaction mechanism between the visual features and the text query. We assume that $f(\cdot, \cdot)$ is linear with the text input, hence we have:

\begin{equation}
\mathbf{V}_B(\vx)
=
f \left(\bar{\va}^{\text{img}} ,
\left(
\cos(\theta_{AB}) \va^{\text{txt}}_A
+
\va^{\text{txt}}_{B\perp A}
\right)  \right).
\end{equation}

By linearity:

\begin{equation}
\mathbf{V}_B(\vx)
=
\cos(\theta_{AB})
\mathbf{V}_A(\vx)
+
f\left( \bar{\va}^{\text{img}} ,
\va^{\text{txt}}_{B\perp A}\right).
\end{equation}

The first term is proportional to the true saliency map
$\mathbf{V}_A(\vx)$.
This term is the \emph{ghost signal}.
\end{proof}

\textit{Implications.}
This result shows that hallucination is directly controlled by
$\cos(\theta_{AB})$.  We empirically validate it in Appendix~\ref{sec:correlation_appendix}: "Correlation between Cosine Similarity and Hallucination".
If the angle between embeddings is small,
the ghost signal can dominate the residual term. Therefore, reducing hallucination amounts to reducing
the projection of a query embedding onto unwanted semantic directions. This phenomenon demonstrates
the semantic contamination of class/concept $A$'s saliency map by the representation of class/concept $B$. Such interference suggests that the attribution mechanism fails to isolate class/concept-specific features when embeddings are non-orthogonal. We provide a comprehensive extension of this analysis to multi-class scenarios in Appendix~\ref{sec:multiclass_appendix}.

\section{Methodology}
\label{sec:method}

\subsection{Sparse Dictionary Decomposition of Text Embeddings}

We formulate semantic purification as a sparse decomposition problem in the joint embedding space.
Let $\va^{\text{txt}}_{\text{target}} \in \mathbb{R}^d$ denote the unit-normalized embedding of a target concept obtained from a pretrained vision-language model.
Let 
\(
\mathbf{D}
=
[\va^{\text{txt}}_{D_1}, \dots, \va^{\text{txt}}_{D_k}]
\in \mathbb{R}^{d \times k}
\)
be a 
dictionary of semantics embeddings, each normalized to unit norm.

In sparse representation theory~\cite{elad2010sparse}, a signal $\vx_s \in \mathbb{R}^d$ can be approximated as
\begin{equation}
\vx_s \approx \mathbf{D}\bm{\alpha},
\end{equation}
where $\bm{\alpha} \in \mathbb{R}^k$ is sparse.
The residual
\begin{equation}
\mathbf{r} = \vx_s - \mathbf{D}\bm{\alpha}
\end{equation}
captures the component of the signal that cannot be explained by the selected atoms.

In our setting, we treat $\va^{\text{txt}}_{\text{target}}$ as the signal and the distractor embeddings as dictionary atoms.
Our objective is not reconstruction accuracy per se, but rather \emph{orthogonal purification}: we aim to remove the semantic components shared with distractors and retain only the unique semantic direction of the target.
\subsection{Orthogonal Matching Pursuit for Semantic Purification}

To compute this decomposition we use Orthogonal Matching Pursuit (OMP)~\cite{pati1993orthogonal}, a greedy algorithm that iteratively selects atoms most correlated with the current residual.

Starting from 
\(
\mathbf{r}^{(0)} = \va^{\text{txt}}_{\text{target}},
\)
OMP performs the following steps:

\begin{enumerate}
    \item Select the atom most correlated with the current residual:
    \begin{equation}
    j^\star = \arg\max_{j \notin \Lambda}
    \left|
    \inner{\mathbf{r}^{(t-1)}}{\va^{\text{txt}}_{D_j}}
    \right|.
    \end{equation}
    
    \item Update the active set $\Lambda$ of selected atoms ($\mathbf{D}_{\Lambda} =[\va^{\text{txt}}_{D_j}]_{j \in \Lambda^{(t)}}$).
    
    \item Recompute the residual via orthogonal projection:
    \begin{equation}
    \mathbf{r}^{(t)}
    =
    \va^{\text{txt}}_{\text{target}}
    -
    \mathbf{D}_\Lambda
    (\mathbf{D}_\Lambda^\top \mathbf{D}_\Lambda)^{-1}
    \mathbf{D}_\Lambda^\top
    \va^{\text{txt}}_{\text{target}}.
    \end{equation}
\end{enumerate}

A fundamental property of OMP is:

\begin{equation}
\inner{\mathbf{r}^{(T)}}{\va^{\text{txt}}_{D_j}} = 0,
\qquad \forall j \in \Lambda,
\end{equation}
meaning the final residual is strictly orthogonal to all selected distractors.

We then normalize:
\begin{equation}
\mathbf{r}
=
\frac{\mathbf{r}^{(T)}}{\|\mathbf{r}^{(T)}\|_2}.
\end{equation}

\subsection{Orthogonal Semantic Projection (OSP)}

We define \textbf{Orthogonal Semantic Projection (OSP)} as the use of the OMP residual as the purified query embedding.

\paragraph{Definition.}
Given a target embedding $\va^{\text{txt}}_{\text{target}}$ and a 
dictionary of semantics $\mathbf{D}$, OSP outputs
\begin{equation}
\mathbf{r}
=
\text{OMP-residual}(\va^{\text{txt}}_{\text{target}}, \mathbf{D}),
\end{equation}
which satisfies
\begin{equation}
\inner{\mathbf{r}}{\va^{\text{txt}}_{D_i}} = 0,
\quad \forall i \in \Lambda.
\end{equation}

Intuitively, OSP decomposes the target embedding into
\begin{equation}
\va^{\text{txt}}_{\text{target}}
=
\underbrace{\mathbf{D}\bm{\alpha}}_{\text{shared semantics}}
+
\underbrace{\mathbf{r}}_{\text{unique semantics}},
\end{equation}
and retains only the residual component.

Because multimodal saliency methods are linear in the text embedding,
replacing $\va^{\text{txt}}_{\text{target}}$ by $\mathbf{r}$ eliminates ghost signals originating from shared semantic directions.

\subsection{Application to Attribution Methods}

\paragraph{Discriminative Methods.}
For CAM, GradCAM, LeGrad, CheferCAM, and AttentionCAM,
the saliency map can be written as
$\mathbf{V}_c(\vx)
= f \left( \bar{\va}^{\text{img}} , \va^{\text{txt}}_c  \right).$
Replacing $\va^{\text{txt}}_c$ with $\mathbf{r}$ yields
\begin{equation}
\mathbf{V}^{\text{OSP}}_c(\vx)
= f \left( \bar{\va}^{\text{img}} , \vr \right)
\end{equation}

If a distractor concept $D_i$ explains a region of the image,
then
\(
\inner{\mathbf{r}}{\va^{\text{txt}}_{D_i}} = 0
\)
implies that the corresponding contribution vanishes.
Thus, ghost saliency components are removed by construction.

\paragraph{Diffusion-Based Attribution (DAAM).}

In diffusion models, attribution arises from cross-attention maps.
Let $\mathbf{K}^{(l,h)}_{\text{target}}$ denote the key vector associated with the target token in layer $l$, head $h$.
We apply OSP in key space:

\begin{equation}
\mathbf{r}^{(l,h)}
=
\text{OMP-residual}
\left(
\mathbf{K}^{(l,h)}_{\text{target}},
\{\mathbf{K}^{(l,h)}_{D_i}\}
\right).
\end{equation}

The purified keys are substituted before attention softmax.
Since attention scores depend linearly on keys before normalization,
orthogonality ensures that distractor-aligned attention weights are suppressed.

\subsection{Dictionary of Semantics Construction}
\label{sec:dictionary}

For each target concept, we extract the \textit{dictionary of semantics}: a structured collection of semantically related distractor concepts whose text embeddings form the atoms of the 
dictionary of semantics $\mathbf{D}$.

To build the 
dictionary of semantics $\mathbf{D}$, we explored four accessible dictionary creation strategies: (1) using ImageNet classes, (2) incorporating ImageNet classes with WordNet semantic relations, and generating customized dictionaries via accessible Large Language Models, specifically (3) Gemini 3 Flash and (4) GPT-OSS 120B. Detailed definitions for these strategies are provided in Appendix~\ref{sec:dict_appendix}. For all approaches, the number of dictionary atoms is a fixed hyperparameter kept identical for both positive and negative cases. Furthermore, we filter out any atoms that exhibit excessively high cosine similarity to the target concept to avoid deleting the core semantic information. The hyperparameters for different dictionaries vary in a small space. More can be found in Appendix~\ref{sec:extended_quantitative_results}.

As shown in Table~\ref{tab:avg_per_dictionary}, all dictionaries consistently perform well and improve the AUROC. The dictionary generated with Gemini 3 Flash is selected as our primary approach. We carefully designed the LLM prompt to generate dictionary atoms structured into three specific semantic groups:
\begin{itemize}
    \item \textbf{Group 1: Visual Confusers} (15 concepts) -- objects that share strong visual features with the target concept.
    \item \textbf{Group 2: Co-occurring Context} (15 concepts) -- objects or background elements frequently found in the same scenes as the target.
    \item \textbf{Group 3: Semantic Hierarchy} (10 concepts) -- structurally related concepts comprising 5 hypernyms and 5 hyponyms.
\end{itemize}

Our results are fully reproducible: the generated dictionary is open source and publicly available, and the exact prompt used to recreate the dictionary is provided in Appendix~\ref{sec:prompt_appendix}.

\section{Experiments}

\subsection{Experimental Setups}
Overall, we test our framework across 3 different models (CLIP~\cite{radford2021learning}, SigLIP~\cite{zhai2023sigmoid}, Stable Diffusion 2~\cite{rombach2022high}) and 5 different methods (GradCAM~\cite{selvaraju2017grad}, LeGrad~\cite{legrad2025}, CheferCAM~\cite{chefer2021}, AttentionCAM~\cite{chefer2020}, DAAM~\cite{tang-etal-2023-daam}). We evaluate the empirical performance of these methods on the ImageNet-Segmentation \cite{guillaumin2014imagenet}, a curated benchmark containing 4,276 images from the ImageNet validation set equipped with precise pixel-level annotations. We also conduct extra experiments on MS COCO, and the results are available in Appendix~\ref{sec:extended_quantitative_results}.
To further validate our approach, we conducted a human study using a dataset of 1.5k heatmaps. For this evaluation, we use the five methods, evaluated both with and without their corresponding OSP heatmaps across three different concepts/classes. We use 50
randomly sampled images from the validation sets of Pascal VOC 2012 \cite{everingham2015pascal} and MS COCO \cite{lin2014microsoft}, making it 5 $\times$ 2 $\times$ 3 $\times$ 50 = 1,500 heatmaps. We specifically selected images that contain at least 2 unique objects to effectively test multi-concept disambiguation in clustered scenes. More information about this experiment is provided in section ~\ref{sec:user_study} and Appendix ~\ref{sec:user_study_appendix}. 

\subsection{Quantitative Results}

We quantitatively evaluate OSP on the ImageNet-Segmentation benchmark using four standard metrics: mean Intersection-over-Union (mIoU), mean Average Precision (mAP), pixel accuracy, and the Area Under the ROC Curve (AUROC). Detailed results are presented in Table~\ref{tab:detailed_results}, where the dictionary of semantics was extracted using Gemini. 
For each method-model pair, we generate heatmaps under two distinct conditions: \textbf{Positive prompts}: The queried concept is present in the image. \textbf{Negative prompts}: The queried concept is absent.
A faithful attribution method should yield high scores for positive prompts and low scores for negative prompts. OSP is designed to widen this gap by enhancing positive scores while suppressing negative ones. 
In particular, we utilize AUROC to validate our theoretical claims regarding hallucination. As formalized in Eq.~\ref{eq:hallucination_def}, explanation hallucination occurs when the mean saliency of an absent distractor class $B$ is comparable to or exceeds that of a present class $A$. Since AUROC evaluates the model's ability to correctly rank these attributions, it serves as a direct measure of our method's success in mitigating hallucination.

Our results indicate a significant improvement in AUROC, suggesting that OSP substantially enhances hallucination detection. Notably, this improvement does not come at the cost of localization performance, as the mIoU for positive prompts remains robust. Furthermore, the reduction in scores for the \textit{negative prompts} demonstrates that inaccurate heatmaps are more effectively suppressed. Any improvement observed in the negative prompts, if present, is smaller than that of the positive prompts. In this set of experiments, increasing the gap between the positive and negative prompts was selected as the objective criterion. This trend 
is consistent across all dictionaries as shown in Table~\ref{tab:avg_per_dictionary} (see Appendix~\ref{sec:extended_quantitative_results} for further details). Importantly, OSP is not dependent on LLMs: dictionaries constructed from ImageNet class names or WordNet semantic relations already yield consistent improvements across all metrics. Furthermore, OSP is a computationally efficient operation, with a runtime of less than one second across all evaluated architectures and saliency methods. Detailed analysis of the computational overhead is provided in Appendix~\ref{sec:runtime_appendix}.

\begin{table}[t]
\centering
\caption{\textbf{Detailed mIoU, mAP, and AUROC} before and after OSP on ImageNet-Segmentation using the Gemini 3 Flash dictionary. $\Delta$ columns show the change introduced by OSP. For positive prompts (\textcolor{green!60!black}{$\uparrow$} desired), OSP should increase scores; for negative prompts (\textcolor{red!70!black}{$\downarrow$} desired), OSP should decrease them. Favorable changes are highlighted in \textbf{bold}.}
\label{tab:detailed_results}
\setlength{\tabcolsep}{3pt}
\scriptsize
\resizebox{\textwidth}{!}{%
\begin{tabular}{ll@{\hskip 4pt}c@{\hskip 6pt}rr@{\hskip 3pt}r@{\hskip 8pt}rr@{\hskip 3pt}r@{\hskip 8pt}rr@{\hskip 3pt}r}
\toprule
\textbf{Method} & \textbf{Model} & \textbf{Prompt} & \multicolumn{3}{c}{\textbf{mIoU}} & \multicolumn{3}{c}{\textbf{mAP}} & \multicolumn{3}{c} {\textbf{AUROC}}\\
\cmidrule(lr){4-6} \cmidrule(lr){7-9} \cmidrule(lr){10-12}
 & & & Base & +OSP & $\Delta$ & Base & +OSP & $\Delta$ & Base & +OSP & $\Delta$ \\
\midrule
\multirow{4}{*}{LeGrad}
 & \multirow{2}{*}{CLIP}   & Pos \textcolor{green!60!black}{$\uparrow$} & 58.66 & 61.33 & \textbf{+2.67} & 82.49 & 85.74 & \textbf{+3.25}& \multirow{2}{*}{ 79.62 } & \multirow{2}{*}{ 80.64 } & \multirow{2}{*}{ \textbf{+1.02 }} \\
 &                         & Neg \textcolor{red!70!black}{$\downarrow$} & 40.88 & 40.74 & \textbf{-0.14} & 67.99 & 70.86 & +2.87 & & & \\
\cmidrule(l){2-12}
 & \multirow{2}{*}{SigLIP} & Pos \textcolor{green!60!black}{$\uparrow$} & 49.51 & 51.18 & \textbf{+1.67} & 78.32 & 78.78 & \textbf{+0.46}& \multirow{2}{*}{ 74.42 } & \multirow{2}{*}{ 74.73 } & \multirow{2}{*}{ \textbf{+0.31 }} \\
 &                         & Neg \textcolor{red!70!black}{$\downarrow$} & 37.27 & 34.68 & \textbf{-2.59} & 63.63 & 62.84 & \textbf{-0.79}& & & \\
\midrule
\multirow{4}{*}{CheferCAM}
 & \multirow{2}{*}{CLIP}   & Pos \textcolor{green!60!black}{$\uparrow$} & 48.71 & 51.32 & \textbf{+2.61} & 80.36 & 82.60 & \textbf{+2.24}& \multirow{2}{*}{ 77.63 } & \multirow{2}{*}{ 80.14 } & \multirow{2}{*}{ \textbf{+2.51 }} \\
 &                         & Neg \textcolor{red!70!black}{$\downarrow$} & 44.88 & 42.99 & \textbf{-1.89} & 78.74 & 80.12 & +1.38 & & & \\
\cmidrule(l){2-12}
 & \multirow{2}{*}{SigLIP} & Pos \textcolor{green!60!black}{$\uparrow$} & 37.66 & 39.60 & \textbf{+1.94} & 73.49 & 75.95 & \textbf{+2.46}& \multirow{2}{*}{ 55.47 } & \multirow{2}{*}{ 62.64 } & \multirow{2}{*}{ \textbf{+7.17 }} \\
 &                         & Neg \textcolor{red!70!black}{$\downarrow$} & 36.65 & 38.45 & +1.80 & 72.38 & 74.67 & +2.29 & & & \\
\midrule
\multirow{4}{*}{AttentionCAM}
 & \multirow{2}{*}{CLIP}   & Pos \textcolor{green!60!black}{$\uparrow$} & 40.14 & 47.96 & \textbf{+7.82} & 70.34 & 76.62 & \textbf{+6.28}& \multirow{2}{*}{ 52.68 } & \multirow{2}{*}{ 67.73 } & \multirow{2}{*}{ \textbf{+15.05 }} \\
 &                         & Neg \textcolor{red!70!black}{$\downarrow$} & 33.34 & 36.63 & +3.29 & 65.74 & 67.55 & +1.81 & & & \\
\cmidrule(l){2-12}
 & \multirow{2}{*}{SigLIP} & Pos \textcolor{green!60!black}{$\uparrow$} & 50.01 & 52.12 & \textbf{+2.11} & 80.20 & 80.81 & \textbf{+0.61}& \multirow{2}{*}{ 80.49 } & \multirow{2}{*}{ 83.45 } & \multirow{2}{*}{ \textbf{+2.96 }} \\
 &                         & Neg \textcolor{red!70!black}{$\downarrow$} & 40.00 & 38.51 & \textbf{-1.49} & 72.70 & 70.16 & \textbf{-2.54}& & & \\
\midrule
\multirow{4}{*}{GradCAM}
 & \multirow{2}{*}{CLIP}   & Pos \textcolor{green!60!black}{$\uparrow$} & 44.68 & 51.43 & \textbf{+6.75} & 74.94 & 79.80 & \textbf{+4.86}& \multirow{2}{*}{ 65.39 } & \multirow{2}{*}{ 71.82 } & \multirow{2}{*}{ \textbf{+6.43 }} \\
 &                         & Neg \textcolor{red!70!black}{$\downarrow$} & 33.83 & 34.23 & +0.40 & 67.34 & 68.45 & +1.11 & & & \\
\cmidrule(l){2-12}
 & \multirow{2}{*}{SigLIP} & Pos \textcolor{green!60!black}{$\uparrow$} & 38.69 & 43.26 & \textbf{+4.57} & 71.43 & 73.61 & \textbf{+2.18}& \multirow{2}{*}{ 67.07 } & \multirow{2}{*}{ 73.01 } & \multirow{2}{*}{ \textbf{+5.94 }} \\
 &                         & Neg \textcolor{red!70!black}{$\downarrow$} & 34.78 & 37.21 & +2.43 & 68.53 & 69.52 & +0.99 & & & \\
\midrule
\multirow{2}{*}{DAAM}
 & \multirow{2}{*}{SD 2}   & Pos \textcolor{green!60!black}{$\uparrow$} & 65.71 & 66.34 & \textbf{+0.63} & 88.55 & 89.76 & \textbf{+1.21}& \multirow{2}{*}{ 83.07 } & \multirow{2}{*}{ 85.48 } & \multirow{2}{*}{ \textbf{+2.41 }} \\
 &                         & Neg \textcolor{red!70!black}{$\downarrow$} & 59.49 & 59.03 & \textbf{-0.46} & 86.39 & 87.33 & +0.94 & & & \\
\bottomrule
\end{tabular}
}
\end{table}

\begin{table}[t]
\centering
\caption{\textbf{Average performance per dictionary strategy} on ImageNet-Segmentation. Values are averaged over all 9 method--model pairs (4 discriminative methods $\times$ 2 models + DAAM). $\Delta$ columns show the mean change introduced by OSP. For positive prompts (\textcolor{green!60!black}{$\uparrow$} desired), OSP should increase scores; for negative prompts (\textcolor{red!70!black}{$\downarrow$} desired), OSP should decrease them. Favorable changes are highlighted in \textbf{bold}.}
\label{tab:avg_per_dictionary}
\setlength{\tabcolsep}{3pt}
\scriptsize
\resizebox{\textwidth}{!}{%
\begin{tabular}{l@{\hskip 4pt}c@{\hskip 6pt}rr@{\hskip 3pt}r@{\hskip 8pt}rr@{\hskip 3pt}r@{\hskip 8pt}rr@{\hskip 3pt}r}
\toprule
\textbf{Dictionary} & \textbf{Prompt} & \multicolumn{3}{c}{\textbf{mIoU}} & \multicolumn{3}{c}{\textbf{mAP}} & \multicolumn{3}{c}{\textbf{AUROC}}\\
\cmidrule(lr){3-5} \cmidrule(lr){6-8} \cmidrule(lr){9-11}
 & & Base & +OSP & $\Delta$ & Base & +OSP & $\Delta$ & Base & +OSP & $\Delta$ \\
\midrule
\multirow{2}{*}{ImageNet Classes}
 & Pos \textcolor{green!60!black}{$\uparrow$} & 48.20 & 51.96 & \textbf{+3.76} & 77.79 & 80.32 & \textbf{+2.53} & \multirow{2}{*}{ 70.64 } & \multirow{2}{*}{ 74.08 } & \multirow{2}{*}{ \textbf{+3.44 }} \\
 & Neg \textcolor{red!70!black}{$\downarrow$} & 42.09 & 41.02 & \textbf{-1.07} & 73.55 & 72.68 & \textbf{-0.87} & & & \\
\midrule
\multirow{2}{*}{ImageNet + WordNet}
 & Pos \textcolor{green!60!black}{$\uparrow$} & 47.98 & 51.40 & \textbf{+3.42} & 77.79 & 80.32 & \textbf{+2.53} & \multirow{2}{*}{ 70.43 } & \multirow{2}{*}{ 73.90 } & \multirow{2}{*}{ \textbf{+3.47 }} \\
 & Neg \textcolor{red!70!black}{$\downarrow$} & 40.12 & 41.20 & +1.08 & 71.49 & 72.77 & +1.28 & & & \\
\midrule
\multirow{2}{*}{Gemini 3 Flash}
 & Pos \textcolor{green!60!black}{$\uparrow$} & 48.20 & 51.61 & \textbf{+3.41} & 77.79 & 80.41 & \textbf{+2.62} & \multirow{2}{*}{ 70.65 } & \multirow{2}{*}{ 75.52 } & \multirow{2}{*}{ \textbf{+4.87 }} \\
 & Neg \textcolor{red!70!black}{$\downarrow$} & 40.12 & 40.27 & +0.15 & 71.49 & 72.39 & +0.90 & & & \\
\midrule
\multirow{2}{*}{GPT-OSS 120B}
 & Pos \textcolor{green!60!black}{$\uparrow$} & 48.20 & 51.34 & \textbf{+3.14} & 77.79 & 80.29 & \textbf{+2.50} & \multirow{2}{*}{ 70.65 } & \multirow{2}{*}{ 76.26 } & \multirow{2}{*}{ \textbf{+5.61 }} \\
 & Neg \textcolor{red!70!black}{$\downarrow$} & 40.12 & 39.90 & \textbf{-0.22} & 71.49 & 72.02 & +0.53 & & & \\
\bottomrule
\end{tabular}
}
\end{table}

\subsection{Visual Results}

Figure~\ref{fig:visual_all1}, and figures in Appendix \ref{sec:extra_visuals} present qualitative comparisons across all five attribution methods (AttentionCAM, CheferCAM, GradCAM, LeGrad, DAAM) with and without OSP, spanning both discriminative (CLIP/SigLIP) and generative (Diffusion) architectures. Each figure shows heatmaps for the target objects present in the image as well as a non-existent distractor class, illustrating how standard methods produce hallucinated attributions for the absent concept while OSP effectively suppresses these ghost signals. While standard methods often hallucinate absent concepts, OSP suppresses these signals and improves \textbf{positive prompt} localization. Thus, OSP both mitigates hallucination and refines true attributions. The following section evaluates these gains through a user study

\begin{figure}[t]
\centering
\includegraphics[width=\textwidth]{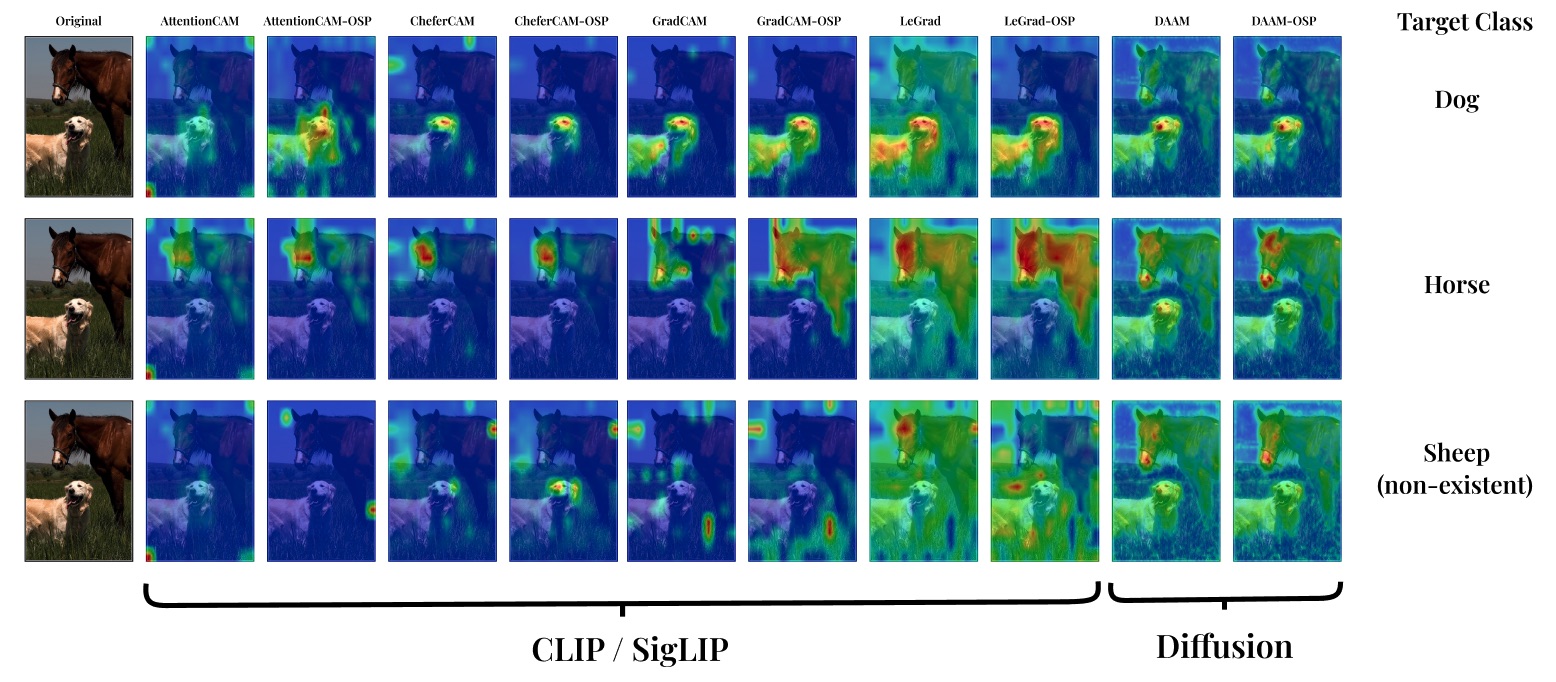}
\vspace{-2mm}
\caption{Visual comparison across all methods on a \textbf{horse--dog} scene. Rows correspond to queries for ``Dog'' (target), ``Horse'' (target), and ``Sheep'' (non-existent distractor). OSP consistently eliminates hallucinated attributions for the absent sheep across both discriminative and generative methods.}
\label{fig:visual_all1}
\end{figure}

\section{Practical Use Cases and User Study}

\subsection{Use Cases}
A key implication of the Hallucination Theorem is that the severity of hallucination is predictive solely from the geometry of the text embedding space, without requiring access to image data or model inference. This mathematical property enables several practical applications for deploying Vision-Language Models in real-world interpretation tasks.

\paragraph{Fine-Grained Concept Disambiguation.}
A frequent failure in multimodal interpretability is the inability to distinguish between closely related semantic concepts or sub-categories. This is often linked to the grounding issues inherent in these models \cite{kazmierczak2025enhancing, liu2024investigating}. This problem also affects Concept Bottleneck Models (CBMs) \cite{srivastava2024vlg, kazmierczak2025enhancing, oikarinen2023label, havasi2022addressing}, where the concept extractor typically operates as a "black box." As shown in \cite{kazmierczak2025enhancing, debole2025if} and illustrated in Fig.~\ref{fig:usecase_bird}, standard attribution methods struggle to provide accurate heatmaps for these specific concepts. However, OSP rectifies these heatmaps, enabling robust, fine-grained concept disambiguation. We evaluate several attribution methods across various bird-related semantic concepts; additional accuracy results are provided in Appendix~\ref{sec:extra_use_cases}. Ultimately, OSP can explain the concepts extracted by CBMs, making the previously opaque components of these models more transparent.

\begin{figure}[htbp]
\centering
\includegraphics[width=\textwidth]{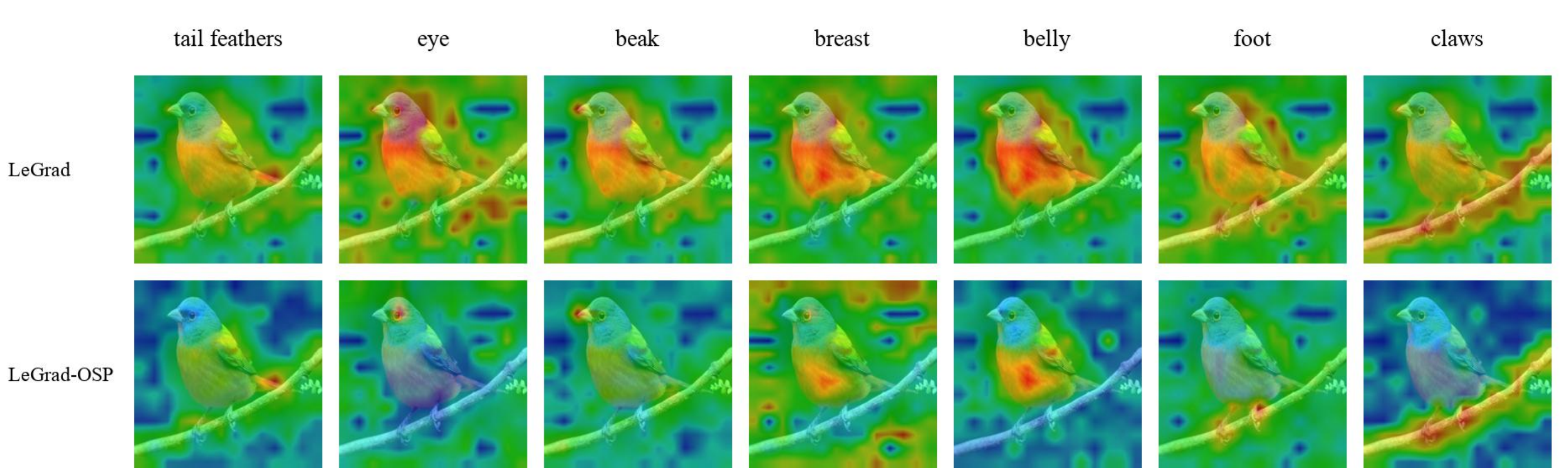}
\vspace{-2mm}
\caption{Practical use case of OSP for fine-grained concept disambiguation. Standard methods suffer from semantic leakage and highlight the general object broadly, whereas OSP successfully disentangles the concepts to highlight specific features.}
\label{fig:usecase_bird}
\end{figure}
\subsection{User Study}
While our quantitative metrics validate OSP on ground-truth segmentation masks, a critical question remains: does OSP make attribution methods more \textit{reliable} for human interpretation? To evaluate this, we conducted a user study with 200 participants, inspired by the human-aligned evaluation framework proposed by Kazmierczak et al.~\cite{kazmierczak2024benchmarking}. 
Participants were shown heatmaps generated by standard attribution methods, with and without OSP, and were asked: \textit{``Based on the heatmap, which class is the model focusing on?''}. 

\label{sec:user_study} 
\begin{wrapfigure}{r}{0.45\textwidth}
\centering
\vspace{-20pt}
\includegraphics[width=\linewidth]{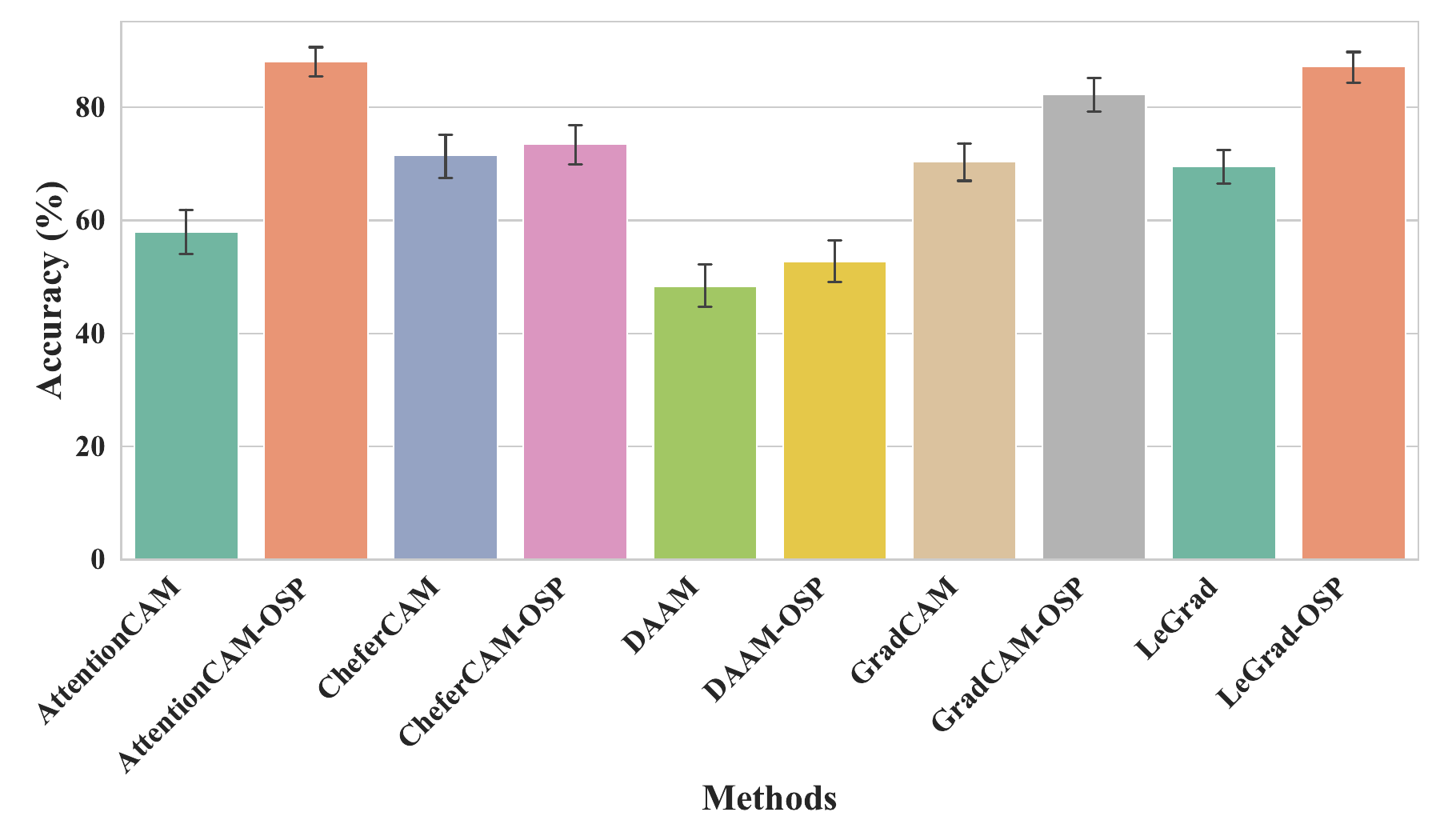}
\caption{User study results. Our method consistently achieves better accuracy, confidence, and trust scores.}
\vspace{-80pt}
\label{fig:method_comparison}
\end{wrapfigure} 
They then rated their confidence in their answer and, after the correct class was revealed, their trust in the explanation.

\textbf{Results.} The results are shown in Fig.~\ref{fig:method_comparison} and Appendix~\ref{sec:user_study_appendix}. OSP increases human identification accuracy, confidence, and trust scores and architectures. This demonstrates that our geometric intervention not only provides a theoretically sound grounding but also makes attribution maps more reliable for end users, translating into measurable improvements in human interpretability. Detailed demographic distributions, statistical analysis including ANOVA results, deeper performance results including detailed confidence and trust scores, and subgroup analyses can be found in Appendix~\ref{sec:user_study_appendix}. Extended results further confirm that OSP's improvements hold across subgroups.

\vspace{-5mm}
\section{Conclusion}
We presented a geometric analysis of attribution hallucination in Vision-Language Models, showing that it arises from Linear Semantic Leakage: the unavoidable non-orthogonality of text embeddings in representation spaces. Building on this analysis, we introduced Orthogonal Semantic Projection, a modular geometric intervention that leverages Orthogonal Matching Pursuit to project query embeddings onto directions orthogonal to distractors. We proved that this orthogonalization mathematically eliminates ghost saliency signals by construction.

Extensive experiments across three foundational models, five attribution methods, four dictionary construction strategies, and two datasets demonstrate that OSP consistently improves AUROC while preserving or enhancing localization fidelity. A comprehensive user study further confirms that OSP yields gains in interpretability.

Looking ahead, we believe that the geometric perspective introduced here opens promising directions for extending OSP to larger-scale generative architectures and broader trustworthy AI applications.

\clearpage
\bibliographystyle{splncs04}
\bibliography{main}

\input{appendix}

\end{document}

%% file: appendix.tex
\clearpage
\setcounter{page}{1}
\pagenumbering{arabic} 
\appendix

\renewcommand{\thetable}{A.\arabic{table}}
\renewcommand{\thefigure}{A.\arabic{figure}}

\begin{center}
  \Large \textbf{Disentangling Hallucinations: Orthogonal Semantic Projection for Robust Interpretability\\ --Supplementary Material--}
\end{center}
\vspace{1cm}

\section{Theoretical Insights}
\subsection{Linearity Principle of Multimodal Saliency Methods}
\label{sec:linearity_appendix}

We now show that all methods defined in Eq.~\eqref{eq:method_cases}
satisfy a linearity principle with respect to the text embedding.

\paragraph{Assumption.}
Let the text embedding for class $c$ be decomposed as:
\begin{equation}
\va^{\text{txt}}_{c}
=
\alpha \va^{\text{txt}}_{c_1}
+
\va^{\text{txt}}_{c_2},
\quad
\alpha \in \mathbb{R}.
\end{equation}

We prove that for each method:
\begin{equation}
\mathbf{V}_c(\vx^{\text{img}})
=
\alpha
\mathbf{V}_{c_1}(\vx^{\text{img}})
+
\mathbf{V}_{c_2}(\vx^{\text{img}}).
\end{equation}

\paragraph{1. CAM.}

From Eq.~\eqref{eq:method_cases}:
\begin{equation}
\mathbf{V}_c(\vx^{\text{img}})
=
\va^{\text{txt}}_c
\mathbf{A}^{\text{MLP}}_{L}(\vx^{\text{img}}).
\end{equation}

Substituting the decomposition:
\begin{align}
\mathbf{V}_c
&=
(\alpha \va^{\text{txt}}_{c_1}
+
\va^{\text{txt}}_{c_2})
\mathbf{A}^{\text{MLP}}_{L}
\\
&=
\alpha
\va^{\text{txt}}_{c_1}
\mathbf{A}^{\text{MLP}}_{L}
+
\va^{\text{txt}}_{c_2}
\mathbf{A}^{\text{MLP}}_{L}
\\
&=
\alpha
\mathbf{V}_{c_1}
+
\mathbf{V}_{c_2}.
\end{align}

Thus CAM is linear.

\paragraph{2. GradCAM.}

GradCAM depends on:
\begin{equation}
\mathbf{w}^c
=
\frac{\partial}{\partial \va^{\text{txt}}_c}
\langle
\va^{\text{txt}}_c,
\va^{\text{img}}
\rangle.
\end{equation}

Since:
\begin{equation}
S_c
=
\langle
\va^{\text{txt}}_c,
\va^{\text{img}}
\rangle
=
\alpha
\langle
\va^{\text{txt}}_{c_1},
\va^{\text{img}}
\rangle
+
\langle
\va^{\text{txt}}_{c_2},
\va^{\text{img}}
\rangle,
\end{equation}

we obtain:
\begin{equation}
\mathbf{w}^c
=
\alpha \mathbf{w}^{c_1}
+
\mathbf{w}^{c_2}.
\end{equation}

Since the saliency map is linear in the weights,
GradCAM is linear.

\paragraph{3. LeGrad.}

LeGrad uses:
\begin{equation}
\mathbf{V}_c
=
\sum_{l,h}
\frac{1}{HL}
\nabla_{\mathbf{A}^{\text{MHA}}_{l,h}} S_c.
\end{equation}

Because:
\begin{equation}
S_c
=
\alpha S_{c_1}
+
S_{c_2},
\end{equation}

and gradients are linear:
\begin{equation}
\nabla S_c
=
\alpha \nabla S_{c_1}
+
\nabla S_{c_2},
\end{equation}

we obtain:
\begin{equation}
\mathbf{V}_c
=
\alpha
\mathbf{V}_{c_1}
+
\mathbf{V}_{c_2}.
\end{equation}

Thus LeGrad is linear.

\paragraph{4. CheferCAM.}

CheferCAM computes relevance by accumulating attention gradients across layers via matrix multiplication:
\begin{equation}
\mathbf{V}_c = \left[ \prod_{l=1}^L \left( \mathbf{I} + \frac{1}{H}\sum_{h} \bigl(\nabla_{\mathbf{A}^{\text{MHA}}_{l,h}} S_c \odot \mathbf{A}^{\text{MHA}}_{l,h}\bigr)^{+} \right) \right]_{\text{CLS}}.
\end{equation}

Before applying ReLU, the update matrix at each layer $l$ is linear in $\nabla S_c$.
Since:
\begin{equation}
\nabla S_c
=
\alpha \nabla S_{c_1}
+
\nabla S_{c_2},
\end{equation}

the pre-ReLU update matrix $\tilde{\mathbf{M}}^l_c = \frac{1}{H}\sum_{h} \nabla_{\mathbf{A}^{\text{MHA}}_{l,h}} S_c \odot \mathbf{A}^{\text{MHA}}_{l,h}$ satisfies:
\begin{equation}
\tilde{\mathbf{M}}^l_c
=
\alpha
\tilde{\mathbf{M}}^l_{c_1}
+
\tilde{\mathbf{M}}^l_{c_2}.
\end{equation}

Because CheferCAM multiplies these matrices across layers, the final attribution map is not strictly linear with respect to the text embedding. However, the ghost signal is linearly injected into the update matrix at every single layer before the ReLU activation, which subsequently accumulates through the forward matrix product.

\paragraph{5. AttentionCAM.}

AttentionCAM restricts the relevance aggregation to the last layer $L$ without any matrix multiplications across layers.
Since it computes the relevance from $\bigl(\nabla_{\mathbf{A}^{\text{MHA}}_{L,h}} S_c \odot \mathbf{A}^{\text{MHA}}_{L,h}\bigr)^{+}$, before the ReLU it is purely a linear combination of $\nabla S_c$ and activations.
Thus, it strictly satisfies the linearity property up to the final ReLU operation.

\paragraph{6. DAAM.}

DAAM averages cross-attention maps:
\begin{equation}
\mathbf{V}_c
=
\frac{1}{T}
\sum_{t=1}^{T}
\frac{1}{|\mathcal{H}|}
\sum_{h \in \mathcal{H}}
\mathbf{A}^{\text{cross}}_{(t,h)}(\vx^{\text{img}}, c).
\end{equation}

Cross-attention maps depend linearly on the query embedding
$\va^{\text{txt}}_c$.
Therefore:
\begin{equation}
\mathbf{A}^{\text{cross}}(\cdot, c)
=
\alpha
\mathbf{A}^{\text{cross}}(\cdot, c_1)
+
\mathbf{A}^{\text{cross}}(\cdot, c_2),
\end{equation}

which implies:
\begin{equation}
\mathbf{V}_c
=
\alpha
\mathbf{V}_{c_1}
+
\mathbf{V}_{c_2}.
\end{equation}

\paragraph{Conclusion.}

All considered multimodal saliency methods are linear
with respect to the text embedding
(up to the final ReLU in CheferCAM/AttentionCAM).

Therefore, for any decomposition
\(
\va^{\text{txt}}_c
=
\alpha \va^{\text{txt}}_{c_1}
+
\va^{\text{txt}}_{c_2},
\)
the resulting saliency map decomposes as:
\begin{equation}
\mathbf{V}_c
=
\alpha
\mathbf{V}_{c_1}
+
\mathbf{V}_{c_2}.
\end{equation}

This linearity is the fundamental mechanism
behind semantic leakage and hallucinated explanations.

Technically, "strict" linearity does not hold for all methods; for instance, CheferCAM and AttentionCAM use ReLU operations, while DAAM applies a softmax function over the cross-attention keys. Nonetheless, a first-order Taylor expansion of a non-linear attribution function $g(\bar{\va}^{\text{img}}, \va^{\text{txt}}_c)$ around a base text embedding $\va_0$ yields:
\begin{equation}
g(\bar{\va}^{\text{img}}, \va^{\text{txt}}_c) = g(\bar{\va}^{\text{img}}, \va_0) + \mathbf{J}_{\va_0} \left( \va^{\text{txt}}_c - \va_0 \right) + \mathcal{O}\left(\|\va^{\text{txt}}_c - \va_0\|^2\right),
\end{equation}
where $\mathbf{J}_{\va_0}$ is the Jacobian of the attribution map with respect to the text embedding, and the higher-order error term is bounded by $\frac{1}{2} \|\mathbf{H}\| \|\va^{\text{txt}}_c - \va_0\|^2$ with $\mathbf{H}$ being the Hessian. Since OSP acts as a geometric rotation of the text embedding on the unit hypersphere rather than a large translation, the perturbation $\|\va^{\text{txt}}_c - \va_0\|$ remains small, ensuring the system operates well within the locally linear regime. Furthermore, for methods utilizing ReLUs (such as CheferCAM), the ghost signal is injected linearly before the activation functions at each layer, propagating through the network layers. Empirically, the consistent improvements in AUROC across non-linear methods (e.g., CheferCAM, AttentionCAM, and DAAM) confirm that the linear approximation effectively captures and mitigates the primary sources of explanation hallucination.

\subsection{Extension of Linear Semantic Leakage to Multi-Class Scenarios}
\label{sec:multiclass_appendix}

In the main text, Theorem~1 considers a simplified scenario where the image contains an object of a single class $A$. We can naturally extend this analysis to a multi-class setting where the image $\vx$ contains objects belonging to multiple classes, denoted as $A_1, A_2, \dots, A_m$. 

Suppose the text embeddings for these present classes are given by $\va^{\text{txt}}_{A_i}$ for $i=1,\dots,m$. We consider a distractor class $B$ that does not exist in the image. Although the object $B$ is absent, its text embedding $\va^{\text{txt}}_B$ may have non-zero cosine similarities with the present concepts $\va^{\text{txt}}_{A_i}$.

We can decompose the embedding of $B$ by projecting it onto the subspace spanned by the embeddings of the present classes $\{ \va^{\text{txt}}_{A_i} \}_{i=1}^m$:
\begin{equation}
\va^{\text{txt}}_B
=
\sum_{i=1}^m \beta_i \va^{\text{txt}}_{A_i}
+
\va^{\text{txt}}_{B\perp \{A_i\}},
\end{equation}
where $\beta_i$ are projection coefficients that depend on the pairwise similarities among the present classes and their similarities with $B$, and $\va^{\text{txt}}_{B\perp \{A_i\}}$ is the residual component orthogonal to the subspace spanned by all $\va^{\text{txt}}_{A_i}$.

As established in Appendix~\ref{sec:linearity_appendix}, the attribution function $f \left( \bar{\va}^{\text{img}}, \va^{\text{txt}}_c \right)$ is linear in the text embedding $\va^{\text{txt}}_c$. Substituting the decomposition into the attribution function yields:
\begin{align}
\mathbf{V}_B(\vx)
&=
f \left( \bar{\va}^{\text{img}}, \sum_{i=1}^m \beta_i \va^{\text{txt}}_{A_i} + \va^{\text{txt}}_{B\perp \{A_i\}} \right) \\
&=
\sum_{i=1}^m \beta_i f \left( \bar{\va}^{\text{img}}, \va^{\text{txt}}_{A_i} \right)
+
f \left( \bar{\va}^{\text{img}}, \va^{\text{txt}}_{B\perp \{A_i\}} \right) \\
&=
\sum_{i=1}^m \beta_i \mathbf{V}_{A_i}(\vx)
+
f \left( \bar{\va}^{\text{img}}, \va^{\text{txt}}_{B\perp \{A_i\}} \right).
\end{align}

This result demonstrates that in images containing multiple distinct objects, the hallucinated saliency map for an absent concept $B$ will be a linear combination of the true saliency maps of all related concepts present in the image. The ghost signal is thus distributed across multiple regions, proportionally weighted by the semantic similarity ($\beta_i$) of $B$ to each valid concept $A_i$. The OSP intervention proposed in our methodology efficiently reduces this complex, multi-concept semantic leakage by mathematically orthogonalizing the text embedding of $B$ with respect to a comprehensive dictionary of semantics.

\section{Correlation between Cosine Similarity and Hallucination}
\label{sec:correlation_appendix}

Our theoretical analysis suggests that the severity of hallucination is directly controlled by the cosine similarity between the queried target and a potential distractor embedding. To validate this empirically, we analyzed the correlation between the text embedding cosine similarity and the degree of hallucinated attribution (measured via area under the attribution heatmap on negative prompts) for different methods.

As shown in Table~\ref{tab:correlation} and Figure~\ref{fig:correlation}, there is a highly significant positive correlation between the cosine similarity of the concepts and the amount of hallucination across all evaluated methods. This confirms our theoretical finding that semantic leakage driven by non-orthogonal embeddings is a primary cause of explanation hallucination.

\begin{table}[htbp]
\centering
\caption{\textbf{Correlation between similarity and hallucination.}}
\label{tab:correlation}
\resizebox{0.7\textwidth}{!}{
\begin{tabular}{lcccc}
\toprule
\textbf{Method} & \textbf{Pearson $r$} & \textbf{Pearson $p$-value} & \textbf{Spearman $r$} & \textbf{Spearman $p$-value} \\
\midrule
LeGrad & 0.2231 & $9.68 \times 10^{-13}$ & 0.2148 & $6.73 \times 10^{-12}$ \\
GradCAM & 0.3515 & $1.89 \times 10^{-30}$ & 0.3407 & $1.36 \times 10^{-28}$ \\
CheferCAM & 0.1791 & $1.17 \times 10^{-8}$ & 0.1603 & $3.45 \times 10^{-7}$ \\
AttentionCAM & 0.2295 & $2.01 \times 10^{-13}$ & 0.1760 & $2.11 \times 10^{-8}$ \\
DAAM & 0.2913 & $5.28 \times 10^{-21}$ & 0.2923 & $3.78 \times 10^{-21}$ \\
\bottomrule
\end{tabular}
}
\end{table}

\begin{figure}[htbp]
\centering
\includegraphics[width=0.5\linewidth]{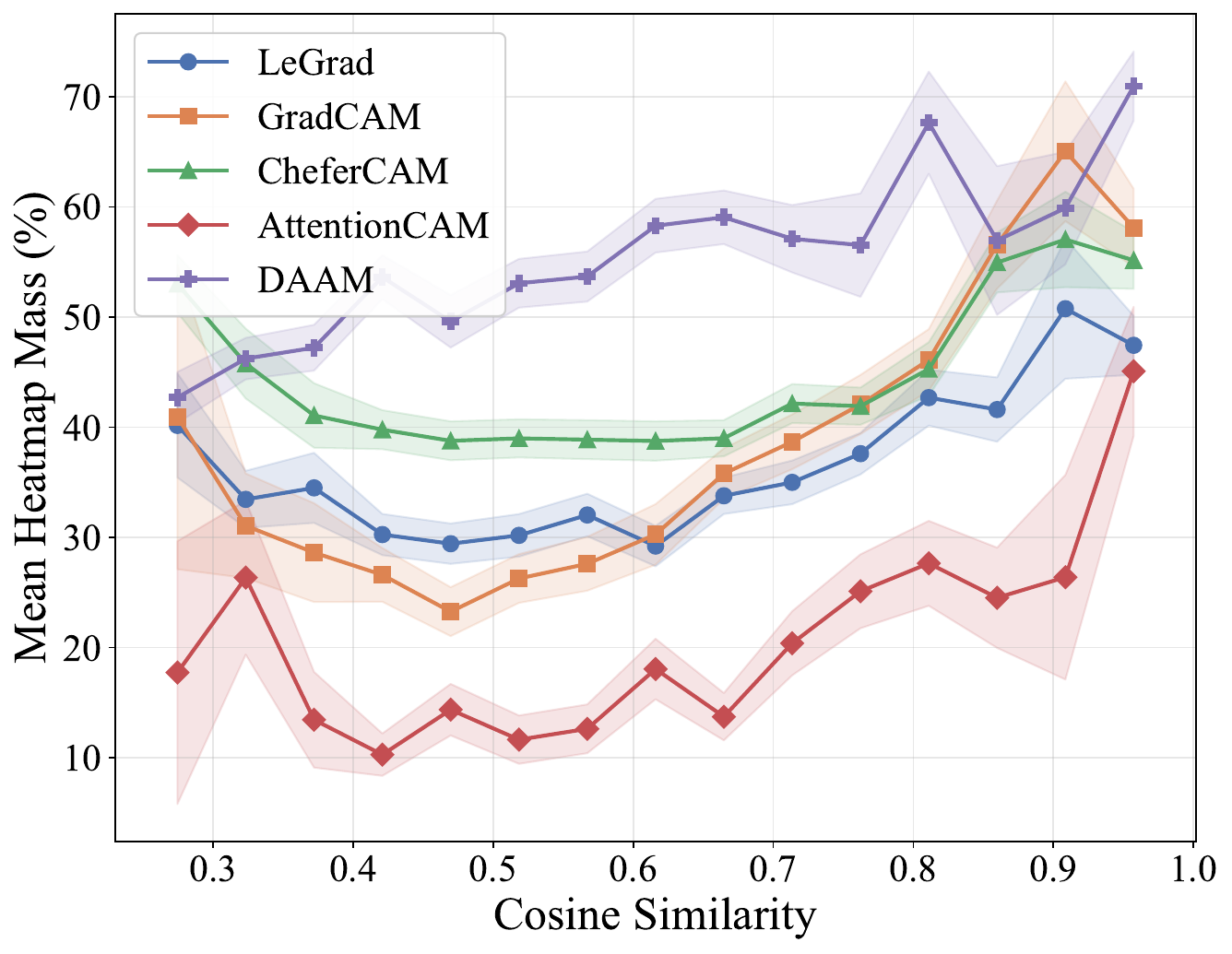}
\caption{Cosine similarity vs Hallucination.}
\label{fig:correlation}
\end{figure}

\section{Dictionary of Semantics Construction Strategies}
\label{sec:dict_appendix}

The \textbf{Dictionary of Semantics} $\mathbf{D}$ is defined as a collection of human-understandable concepts (words or short phrases) that convey semantic information related to the target class. These "atoms" or semantics represent categories that share visual or conceptual features with the main class. The primary objective of the dictionary is to facilitate the \textbf{denoising} of the embedding of the class/concept prompt. By identifying and isolating overlapping features through our OSP algorithm, we can filter out "semantic noise" and isolate the unique attributes of the target. Because OSP relies on these semantics to define the projection space, the choice of the dictionary is essential. In this section, we present various strategies to select the optimal semantics.

As discussed in the main text, we explored four different strategies for building the dictionary of semantics $\mathbf{D}$. The full quantitative results obtained with each of these four strategies are detailed in Section~\ref{sec:extended_quantitative_results}.

\begin{enumerate}
    \item \textbf{ImageNet Classes:} We used all 445 classes from the ImageNet Segmentation dataset as dictionary atoms, excluding the target class itself (resulting in 444 concepts).
    \item \textbf{ImageNet Classes + WordNet:} We augmented the ImageNet classes with semantic relations from WordNet. WordNet provides hypernyms, co-hyponyms (siblings), hyponyms, and synonyms. In this strategy, we incorporated all of these relations except for the synonyms, which we intentionally excluded to prevent fully deleting the target concept's semantic information.
    \item \textbf{Gemini 3 Flash Generation:} We prompted the Gemini 3 Flash Large Language Model (LLM) to generate a customized dictionary of semantics for each target class.
    \item \textbf{GPT-OSS 120B Generation:} We similarly prompted the GPT-OSS 120B model to generate a custom dictionary of semantics.
\end{enumerate}

Regarding the WordNet strategy, we observed that including synonyms negatively affects performance. While most synonyms are filtered out by the cosine similarity threshold, the ones that are not filtered decrease performance by $-1.33$ mIoU on average.

To demonstrate that our method does not strictly rely on LLMs, we highlight two additional settings:
\begin{itemize}
    \item \textbf{User Study:} We utilized minimal dictionaries containing only two existing objects and one non-existing object.
    \item \textbf{PartImageNet++:} The dictionary size was restricted to a maximum of 4 atoms/semantics per category, directly reflecting the dataset's natural part-based structure.
\end{itemize}
These cases demonstrate that highly effective localization can be achieved with concise, manually-defined dictionaries, proving that large-scale model generation is a convenience rather than a requirement.

\section{Extended Quantitative Results}
\label{sec:extended_quantitative_results}

We present the complete set of hyperparameters and quantitative results for the four different dictionaries of semantics creation strategies evaluated in this work. For each strategy, we report the hyperparameters used, the Detailed mIoU and mAP, and the Full quantitative results (including pixel Accuracy and Area Under the ROC Curve).

\begin{quote} \textbf{AUROC Computation Details} \ The AUROC is computed using a paired-image approach that evaluates both localization accuracy and hallucination suppression simultaneously. For an image with ground truth (GT) mask $M \in \{0,1\}^{H \times W}$, two heatmaps are utilized: $H_{pos}$ (correct prompt) and $H_{neg}$ (wrong/absent prompt).

\begin{enumerate} \item \textbf{Pairing Strategy:} The measurement is framed as a binary classification problem over a concatenated domain. The ground truth vector is $Y = [\text{vec}(M), \vec{0}]$ and the prediction score vector is $P = [\text{vec}(H_{pos}), \text{vec}(H_{neg})]$.
\item \textbf{Absent Classes:} For each image, exactly $N=1$ absent class is sampled to construct the negative scenario. In the MS COCO script, this is chosen from other objects present in the image metadata; in the ImageNet script, it is sampled randomly. \end{enumerate} \end{quote} The reported mean AUROC is the arithmetic average over all valid pairs.

\paragraph{Implementation Details for DAAM.}
DAAM utilizes the generative prior of Stable Diffusion in a discriminative manner on real images.
To adapt the generative DAAM to real images, the pipeline performs a partial reconstruction flow:
\begin{enumerate}
    \item The real image $I$ is encoded into the latent space $z = \mathcal{E}(I)$ using the VAE.
    \item A single forward pass is executed at a specific diffusion timestep by adding minimal noise $\epsilon$ to the latents: $z_t = \text{SC}(z, \epsilon, t)$.
    \item Cross-attention maps are then extracted from the UNet during this single denoising step to produce the attribution heatmaps.
\end{enumerate}

\paragraph{Key-Space Orthogonalization.}
The intervention for anti-hallucination is implemented by overriding the attention processors. For a target concept $c$ and distractor concepts $\{d_i\}$, the key vector $k_c$ for the target token is orthogonalized against the subspace spanned by distractor keys $k_{d_i}$:
\begin{equation}
k'_c = k_c - \beta_{\text{omp}} \sum_{i} \text{proj}_{k_{d_i}}(k_c)
\end{equation}
Here, $\beta_{\text{omp}}$ is the \textbf{OMP substitution weight}, which controls the strength of the orthogonalization.

\subsection{Strategy 1 for the dictionary of semantics: ImageNet Classes}

In this strategy, we utilized the 445 ImageNet classes as the vocabulary for our dictionary of semantics. We provide the selected hyperparameters in Table~\ref{tab:all_hyperparameters}, and the full quantitative results in Table~\ref{tab:full_results_imagenet}.

\begin{table}[h]
\centering
\caption{\textbf{Full quantitative results on ImageNet-Segmentation using the dictionary from Strategy 1 (ImageNet Classes).} All four metrics are shown before (Base) and after (+OSP) applying Orthogonal Semantic Projection when using only ImageNet classes, along with deltas ($\Delta$). The Gap Improvement is computed as $\Delta_{\text{pos}} - \Delta_{\text{neg}}$, where each $\Delta$ sums changes in mIoU, Accuracy, mAP, and AUROC. Favorable changes are highlighted in \textbf{bold}.}
\label{tab:full_results_imagenet}
\resizebox{\textwidth}{!}{%
\setlength{\tabcolsep}{2.5pt}
\scriptsize
\begin{tabular}{ll@{\hskip 3pt}c@{\hskip 4pt}rr@{\hskip 2pt}r@{\hskip 6pt}rr@{\hskip 2pt}r@{\hskip 6pt}rr@{\hskip 2pt}r@{\hskip 6pt}rr@{\hskip 2pt}r@{\hskip 6pt}r}
\toprule
\textbf{Method} & \textbf{Model} & \textbf{Pr.} & \multicolumn{3}{c}{\textbf{mIoU}} & \multicolumn{3}{c}{\textbf{Accuracy}} & \multicolumn{3}{c}{\textbf{mAP}} & \multicolumn{3}{c}{\textbf{AUROC}} & \textbf{Gap} \\
\cmidrule(lr){4-6} \cmidrule(lr){7-9} \cmidrule(lr){10-12} \cmidrule(lr){13-15}
 & & & B & +O & $\Delta$ & B & +O & $\Delta$ & B & +O & $\Delta$ & B & +O & $\Delta$ & \textbf{Impr.} \\
\midrule
\multirow{4}{*}{LeGrad}
 & \multirow{2}{*}{CLIP}   & + & 58.66 & 55.30 & -3.36 & 77.52 & 73.32 & -4.20 & 82.49 & 82.25 & -0.24 & \multirow{2}{*}{79.60} & \multirow{2}{*}{77.47} & \multirow{2}{*}{-2.13} & \multirow{2}{*}{\textbf{1.67}} \\
 &                         & -- & 40.88 & 37.46 & \textbf{-3.42} & 64.32 & 56.63 & \textbf{-7.69} & 67.99 & 67.50 & \textbf{-0.49} &  &  &  &  \\
\cmidrule(l){2-16}
 & \multirow{2}{*}{SigLIP} & + & 49.51 & 49.83 & \textbf{+0.32} & 73.28 & 69.55 & -3.73 & 78.32 & 78.30 & -0.02 & \multirow{2}{*}{74.42} & \multirow{2}{*}{70.82} & \multirow{2}{*}{-3.60} & \multirow{2}{*}{\textbf{10.76}} \\
 &                         & -- & 37.27 & 32.82 & \textbf{-4.45} & 63.47 & 51.85 & \textbf{-11.62} & 63.63 & 61.91 & \textbf{-1.72} &  &  &  &  \\
\midrule
\multirow{4}{*}{CheferCAM}
 & \multirow{2}{*}{CLIP}   & + & 48.71 & 55.12 & \textbf{+6.41} & 69.32 & 75.94 & \textbf{+6.62} & 80.36 & 83.39 & \textbf{+3.03} & \multirow{2}{*}{77.63} & \multirow{2}{*}{79.19} & \multirow{2}{*}{\textbf{+1.56}} & \multirow{2}{*}{\textbf{28.53}} \\
 &                         & -- & 51.90 & 48.39 & \textbf{-3.51} & 76.30 & 71.37 & \textbf{-4.93} & 83.68 & 81.21 & \textbf{-2.47} &  &  &  &  \\
\cmidrule(l){2-16}
 & \multirow{2}{*}{SigLIP} & + & 37.66 & 39.12 & \textbf{+1.46} & 60.53 & 63.33 & \textbf{+2.80} & 73.49 & 75.50 & \textbf{+2.01} & \multirow{2}{*}{55.47} & \multirow{2}{*}{60.55} & \multirow{2}{*}{\textbf{+5.08}} & \multirow{2}{*}{\textbf{17.50}} \\
 &                         & -- & 40.53 & 38.59 & \textbf{-1.94} & 63.44 & 61.36 & \textbf{-2.08} & 76.64 & 74.51 & \textbf{-2.13} &  &  &  &  \\
\midrule
\multirow{4}{*}{Att.CAM}
 & \multirow{2}{*}{CLIP}   & + & 40.14 & 55.08 & \textbf{+14.94} & 68.67 & 76.98 & \textbf{+8.31} & 70.34 & 81.88 & \textbf{+11.54} & \multirow{2}{*}{52.68} & \multirow{2}{*}{75.20} & \multirow{2}{*}{\textbf{+22.52}} & \multirow{2}{*}{\textbf{84.12}} \\
 &                         & -- & 54.74 & 44.04 & \textbf{-10.70} & 76.86 & 69.65 & \textbf{-7.21} & 83.54 & 74.64 & \textbf{-8.90} &  &  &  &  \\
\cmidrule(l){2-16}
 & \multirow{2}{*}{SigLIP} & + & 50.01 & 51.87 & \textbf{+1.86} & 70.33 & 71.96 & \textbf{+1.63} & 80.20 & 80.44 & \textbf{+0.24} & \multirow{2}{*}{80.49} & \multirow{2}{*}{79.07} & \multirow{2}{*}{-1.42} & \multirow{2}{*}{\textbf{7.63}} \\
 &                         & -- & 40.00 & 39.26 & \textbf{-0.74} & 66.83 & 64.70 & \textbf{-2.13} & 72.70 & 70.25 & \textbf{-2.45} &  &  &  &  \\
\midrule
\multirow{4}{*}{GradCAM}
 & \multirow{2}{*}{CLIP}   & + & 44.68 & 50.32 & \textbf{+5.64} & 69.48 & 70.86 & \textbf{+1.38} & 74.94 & 79.32 & \textbf{+4.38} & \multirow{2}{*}{65.39} & \multirow{2}{*}{70.77} & \multirow{2}{*}{\textbf{+5.38}} & \multirow{2}{*}{\textbf{24.38}} \\
 &                         & -- & 33.83 & 32.72 & \textbf{-1.11} & 58.86 & 53.12 & \textbf{-5.74} & 68.84 & 68.09 & \textbf{-0.75} &  &  &  &  \\
\cmidrule(l){2-16}
 & \multirow{2}{*}{SigLIP} & + & 38.69 & 41.13 & \textbf{+2.44} & 57.78 & 69.10 & \textbf{+11.32} & 71.43 & 72.00 & \textbf{+0.57} & \multirow{2}{*}{67.07} & \multirow{2}{*}{69.39} & \multirow{2}{*}{\textbf{+2.32}} & \multirow{2}{*}{\textbf{33.68}} \\
 &                         & -- & 39.60 & 37.19 & \textbf{-2.41} & 83.36 & 69.04 & \textbf{-14.32} & 68.53 & 68.23 & \textbf{-0.30} &  &  &  &  \\
\midrule
\multirow{2}{*}{DAAM}
 & \multirow{2}{*}{SD 2}   & + & 65.71 & 69.89 & \textbf{+4.18} & 81.33 & 83.94 & \textbf{+2.61} & 88.55 & 89.85 & \textbf{+1.30} & \multirow{2}{*}{83.07} & \multirow{2}{*}{84.28} & \multirow{2}{*}{\textbf{+1.21}} & \multirow{2}{*}{\textbf{10.15}} \\
 &                         & -- & 59.49 & 58.70 & \textbf{-0.79} & 76.60 & 75.13 & \textbf{-1.47} & 86.39 & 87.80 & +1.41 &  &  &  &  \\
\bottomrule
\end{tabular}}
\end{table}

\subsection{Strategy 2 for the dictionary of semantics:: ImageNet Classes + WordNet (Without Synonyms)}
Here, the dictionary of semantics was constructed by augmenting the ImageNet classes with semantic relationships from WordNet, deliberately excluding synonyms to preserve the target concept. The tuned hyperparameters are listed in Table~\ref{tab:all_hyperparameters}, and the complete performance metrics are summarized in Table~\ref{tab:full_results_wordnet}.

\begin{table}[h]
\centering
\caption{\textbf{Full quantitative results on ImageNet-Segmentation using the dictionary from Strategy 2 (ImageNet Classes + WordNet).} All four metrics are shown before (Base) and after (+OSP) applying Orthogonal Semantic Projection when using ImageNet classes and WordNet (except for synonyms), along with deltas ($\Delta$). The Gap Improvement is computed as $\Delta_{\text{pos}} - \Delta_{\text{neg}}$, where each $\Delta$ sums changes in mIoU, Accuracy, mAP, and AUROC. Favorable changes are highlighted in \textbf{bold}.}
\label{tab:full_results_wordnet}
\resizebox{\textwidth}{!}{%
\setlength{\tabcolsep}{2.5pt}
\scriptsize
\begin{tabular}{ll@{\hskip 3pt}c@{\hskip 4pt}rr@{\hskip 2pt}r@{\hskip 6pt}rr@{\hskip 2pt}r@{\hskip 6pt}rr@{\hskip 2pt}r@{\hskip 6pt}rr@{\hskip 2pt}r@{\hskip 6pt}r}
\toprule
\textbf{Method} & \textbf{Model} & \textbf{Pr.} & \multicolumn{3}{c}{\textbf{mIoU}} & \multicolumn{3}{c}{\textbf{Accuracy}} & \multicolumn{3}{c}{\textbf{mAP}} & \multicolumn{3}{c}{\textbf{AUROC}} & \textbf{Gap} \\
\cmidrule(lr){4-6} \cmidrule(lr){7-9} \cmidrule(lr){10-12} \cmidrule(lr){13-15}
 & & & B & +O & $\Delta$ & B & +O & $\Delta$ & B & +O & $\Delta$ & B & +O & $\Delta$ & \textbf{Impr.} \\
\midrule
\multirow{4}{*}{LeGrad}
 & \multirow{2}{*}{CLIP}   & + & 58.66 & 55.61 & -3.05 & 77.52 & 73.88 & -3.64 & 82.49 & 82.20 & -0.29 & \multirow{2}{*}{77.62} & \multirow{2}{*}{75.28} & \multirow{2}{*}{-2.34} & \multirow{2}{*}{\textbf{0.37}} \\
 &                         & -- & 40.88 & 38.13 & \textbf{-2.75} & 64.32 & 57.71 & \textbf{-6.61} & 67.99 & 67.66 & \textbf{-0.33} &  &  &  &  \\
\cmidrule(l){2-16}
 & \multirow{2}{*}{SigLIP} & + & 49.51 & 49.51 & 0.00 & 73.28 & 69.21 & -4.07 & 78.32 & 78.14 & -0.18 & \multirow{2}{*}{74.42} & \multirow{2}{*}{68.98} & \multirow{2}{*}{-5.44} & \multirow{2}{*}{\textbf{5.50}} \\
 &                         & -- & 37.27 & 33.77 & \textbf{-3.50} & 63.47 & 52.28 & \textbf{-11.19} & 63.63 & 63.13 & \textbf{-0.50} &  &  &  &  \\
\midrule
\multirow{4}{*}{CheferCAM}
 & \multirow{2}{*}{CLIP}   & + & 48.71 & 56.10 & \textbf{+7.39} & 69.32 & 75.58 & \textbf{+6.26} & 80.36 & 83.42 & \textbf{+3.06} & \multirow{2}{*}{77.63} & \multirow{2}{*}{79.45} & \multirow{2}{*}{\textbf{+1.82}} & \multirow{2}{*}{\textbf{8.23}} \\
 &                         & -- & 44.88 & 48.90 & +4.02 & 66.44 & 70.30 & +3.86 & 78.74 & 81.16 & +2.42 &  &  &  &  \\
\cmidrule(l){2-16}
 & \multirow{2}{*}{SigLIP} & + & 37.66 & 39.15 & \textbf{+1.49} & 60.53 & 63.39 & \textbf{+2.86} & 73.49 & 75.45 & \textbf{+1.96} & \multirow{2}{*}{55.47} & \multirow{2}{*}{60.73} & \multirow{2}{*}{\textbf{+5.26}} & \multirow{2}{*}{\textbf{6.04}} \\
 &                         & -- & 36.65 & 38.27 & +1.62 & 59.28 & 61.16 & +1.88 & 72.38 & 74.41 & +2.03 &  &  &  &  \\
\midrule
\multirow{4}{*}{Att.CAM}
 & \multirow{2}{*}{CLIP}   & + & 40.14 & 55.32 & \textbf{+15.18} & 68.67 & 77.21 & \textbf{+8.54} & 70.34 & 82.08 & \textbf{+11.74} & \multirow{2}{*}{52.68} & \multirow{2}{*}{75.64} & \multirow{2}{*}{\textbf{+22.96}} & \multirow{2}{*}{\textbf{31.16}} \\
 &                         & -- & 33.34 & 44.28 & +10.94 & 62.44 & 69.96 & +7.52 & 65.74 & 74.54 & +8.80 &  &  &  &  \\
\cmidrule(l){2-16}
 & \multirow{2}{*}{SigLIP} & + & 50.01 & 51.08 & \textbf{+1.07} & 70.33 & 70.29 & -0.04 & 80.20 & 80.41 & \textbf{+0.21} & \multirow{2}{*}{80.49} & \multirow{2}{*}{79.01} & \multirow{2}{*}{-1.48} & \multirow{2}{*}{\textbf{3.26}} \\
 &                         & -- & 40.00 & 38.59 & \textbf{-1.41} & 62.57 & 62.78 & +0.21 & 72.70 & 70.40 & \textbf{-2.30} &  &  &  &  \\
\midrule
\multirow{4}{*}{GradCAM}
 & \multirow{2}{*}{CLIP}   & + & 44.68 & 51.29 & \textbf{+6.61} & 69.48 & 71.32 & \textbf{+1.84} & 74.94 & 79.89 & \textbf{+4.95} & \multirow{2}{*}{65.39} & \multirow{2}{*}{71.53} & \multirow{2}{*}{\textbf{+6.14}} & \multirow{2}{*}{\textbf{22.15}} \\
 &                         & -- & 33.83 & 34.10 & +0.27 & 58.86 & 54.23 & \textbf{-4.63} & 67.34 & 69.09 & +1.75 &  &  &  &  \\
\cmidrule(l){2-16}
 & \multirow{2}{*}{SigLIP} & + & 38.69 & 41.20 & \textbf{+2.51} & 57.78 & 69.13 & \textbf{+11.35} & 71.43 & 72.06 & \textbf{+0.63} & \multirow{2}{*}{67.07} & \multirow{2}{*}{69.75} & \multirow{2}{*}{\textbf{+2.68}} & \multirow{2}{*}{\textbf{0.95}} \\
 &                         & -- & 34.78 & 37.14 & +2.36 & 54.72 & 69.02 & +14.30 & 68.53 & 68.09 & \textbf{-0.44} &  &  &  &  \\
\midrule
\multirow{2}{*}{DAAM}
 & \multirow{2}{*}{SD 2}   & + & 63.71 & 63.36 & -0.35 & 80.94 & 79.83 & -1.11 & 88.55 & 89.17 & \textbf{+0.62} & \multirow{2}{*}{83.07} & \multirow{2}{*}{84.69} & \multirow{2}{*}{\textbf{+1.62}} & \multirow{2}{*}{\textbf{4.39}} \\
 &                         & -- & 59.49 & 57.62 & \textbf{-1.87} & 76.60 & 74.82 & \textbf{-1.78} & 86.39 & 86.43 & +0.04 &  &  &  &  \\
\bottomrule
\end{tabular}}
\end{table}

\subsection{Strategy 3 for the dictionary of semantics:: Gemini 3 Flash Generation}
For this strategy, we prompted the Gemini 3 Flash Large Language Model to generate a customized dictionary of semantics for each target class using the structured prompt detailed in Figure~\ref{fig:llm_prompt}. The resulting hyperparameters used for this dictionary of semantics are reported in Table~\ref{tab:all_hyperparameters}, and the full quantitative results demonstrating its performance are provided in Table~\ref{tab:full_results_gemini}.

\begin{table}[h]
\centering
\caption{\textbf{Full quantitative results on ImageNet-Segmentation using the dictionary from Strategy 3 (Gemini 3 Flash).} All four metrics are shown before (Base) and after (+OSP) applying Orthogonal Semantic Projection when using a dictionary of semantics created with Gemini 3 Flash, along with deltas ($\Delta$). The Gap Improvement is computed as $\Delta_{\text{pos}} - \Delta_{\text{neg}}$, where each $\Delta$ sums changes in mIoU, Accuracy, mAP, and AUROC. Favorable changes are highlighted in \textbf{bold}.}
\label{tab:full_results_gemini}
\resizebox{\textwidth}{!}{%
\setlength{\tabcolsep}{2.5pt}
\scriptsize
\begin{tabular}{ll@{\hskip 3pt}c@{\hskip 4pt}rr@{\hskip 2pt}r@{\hskip 6pt}rr@{\hskip 2pt}r@{\hskip 6pt}rr@{\hskip 2pt}r@{\hskip 6pt}rr@{\hskip 2pt}r@{\hskip 6pt}r}
\toprule
\textbf{Method} & \textbf{Model} & \textbf{Pr.} & \multicolumn{3}{c}{\textbf{mIoU}} & \multicolumn{3}{c}{\textbf{Accuracy}} & \multicolumn{3}{c}{\textbf{mAP}} & \multicolumn{3}{c}{\textbf{AUROC}} & \textbf{Gap} \\
\cmidrule(lr){4-6} \cmidrule(lr){7-9} \cmidrule(lr){10-12} \cmidrule(lr){13-15}
 & & & B & +O & $\Delta$ & B & +O & $\Delta$ & B & +O & $\Delta$ & B & +O & $\Delta$ & \textbf{Impr.} \\
\midrule
\multirow{4}{*}{LeGrad}
 & \multirow{2}{*}{CLIP}   & + & 58.66 & 61.33 & \textbf{+2.67} & 77.52 & 78.44 & \textbf{+0.92} & 82.49 & 85.74 & \textbf{+3.25} & \multirow{2}{*}{79.62} & \multirow{2}{*}{80.64} & \multirow{2}{*}{\textbf{+1.02}} & \multirow{2}{*}{\textbf{7.81}} \\
 &                         & -- & 40.88 & 40.74 & \textbf{-0.14} & 64.32 & 61.64 & \textbf{-2.68} & 67.99 & 70.86 & +2.87 &  &  &  &  \\
\cmidrule(l){2-16}
 & \multirow{2}{*}{SigLIP} & + & 49.51 & 51.18 & \textbf{+1.67} & 73.28 & 71.60 & -1.68 & 78.32 & 78.78 & \textbf{+0.46} & \multirow{2}{*}{74.42} & \multirow{2}{*}{74.73} & \multirow{2}{*}{\textbf{+0.31}} & \multirow{2}{*}{\textbf{10.76}} \\
 &                         & -- & 37.27 & 34.68 & \textbf{-2.59} & 63.47 & 56.85 & \textbf{-6.62} & 63.63 & 62.84 & \textbf{-0.79} &  &  &  &  \\
\midrule
\multirow{4}{*}{CheferCAM}
 & \multirow{2}{*}{CLIP}   & + & 48.71 & 51.32 & \textbf{+2.61} & 69.32 & 74.14 & \textbf{+4.82} & 80.36 & 82.60 & \textbf{+2.24} & \multirow{2}{*}{77.63} & \multirow{2}{*}{80.14} & \multirow{2}{*}{\textbf{+2.51}} & \multirow{2}{*}{\textbf{9.69}} \\
 &                         & -- & 44.88 & 42.99 & \textbf{-1.89} & 66.44 & 69.44 & +3.00 & 78.74 & 80.12 & +1.38 &  &  &  &  \\
\cmidrule(l){2-16}
 & \multirow{2}{*}{SigLIP} & + & 37.66 & 39.60 & \textbf{+1.94} & 60.53 & 63.55 & \textbf{+3.02} & 73.49 & 75.95 & \textbf{+2.46} & \multirow{2}{*}{55.47} & \multirow{2}{*}{62.64} & \multirow{2}{*}{\textbf{+7.17}} & \multirow{2}{*}{\textbf{7.90}} \\
 &                         & -- & 36.65 & 38.45 & +1.80 & 59.28 & 61.88 & +2.60 & 72.38 & 74.67 & +2.29 &  &  &  &  \\
\midrule
\multirow{4}{*}{Att.CAM}
 & \multirow{2}{*}{CLIP}   & + & 40.14 & 47.96 & \textbf{+7.82} & 68.67 & 73.28 & \textbf{+4.61} & 70.34 & 76.62 & \textbf{+6.28} & \multirow{2}{*}{52.68} & \multirow{2}{*}{67.73} & \multirow{2}{*}{\textbf{+15.05}} & \multirow{2}{*}{\textbf{24.68}} \\
 &                         & -- & 33.34 & 36.63 & +3.29 & 62.44 & 66.42 & +3.98 & 65.74 & 67.55 & +1.81 &  &  &  &  \\
\cmidrule(l){2-16}
 & \multirow{2}{*}{SigLIP} & + & 50.01 & 52.12 & \textbf{+2.11} & 70.33 & 72.58 & \textbf{+2.25} & 80.20 & 80.81 & \textbf{+0.61} & \multirow{2}{*}{80.49} & \multirow{2}{*}{83.45} & \multirow{2}{*}{\textbf{+2.96}} & \multirow{2}{*}{\textbf{6.27}} \\
 &                         & -- & 40.00 & 38.51 & \textbf{-1.49} & 62.57 & 68.26 & +5.69 & 72.70 & 70.16 & \textbf{-2.54} &  &  &  &  \\
\midrule
\multirow{4}{*}{GradCAM}
 & \multirow{2}{*}{CLIP}   & + & 44.68 & 51.43 & \textbf{+6.75} & 69.48 & 72.61 & \textbf{+3.13} & 74.94 & 79.80 & \textbf{+4.86} & \multirow{2}{*}{65.39} & \multirow{2}{*}{71.82} & \multirow{2}{*}{\textbf{+6.43}} & \multirow{2}{*}{\textbf{22.45}} \\
 &                         & -- & 33.83 & 34.23 & +0.40 & 58.86 & 56.07 & \textbf{-2.79} & 67.34 & 68.45 & +1.11 &  &  &  &  \\
\cmidrule(l){2-16}
 & \multirow{2}{*}{SigLIP} & + & 38.69 & 43.26 & \textbf{+4.57} & 57.78 & 69.55 & \textbf{+11.77} & 71.43 & 73.61 & \textbf{+2.18} & \multirow{2}{*}{67.07} & \multirow{2}{*}{73.01} & \multirow{2}{*}{\textbf{+5.94}} & \multirow{2}{*}{\textbf{6.80}} \\
 &                         & -- & 34.78 & 37.21 & +2.43 & 54.72 & 68.96 & +14.24 & 68.53 & 69.52 & +0.99 &  &  &  &  \\
\midrule
\multirow{2}{*}{DAAM}
 & \multirow{2}{*}{SD 2}   & + & 65.71 & 66.34 & \textbf{+0.63} & 81.33 & 83.25 & \textbf{+1.92} & 88.55 & 89.76 & \textbf{+1.21} & \multirow{2}{*}{83.07} & \multirow{2}{*}{85.48} & \multirow{2}{*}{\textbf{+2.41}} & \multirow{2}{*}{\textbf{6.35}} \\
 &                         & -- & 59.49 & 59.03 & \textbf{-0.46} & 77.10 & 76.44 & \textbf{-0.66} & 86.39 & 87.33 & +0.94 &  &  &  &  \\
\bottomrule
\end{tabular}}
\end{table}

\subsection{Strategy 4 for the dictionary of semantics:: GPT-OSS 120B Generation}
Similarly, this strategy leverages the GPT-OSS 120B model to generate a custom dictionary of semantics following the same prompt structure shown in Figure~\ref{fig:llm_prompt}. The hyperparameters selected for this approach are detailed in Table~\ref{tab:all_hyperparameters}, and the extensive quantitative results in Table~\ref{tab:full_results_gpt}.

\begin{table}[h]
\centering
\caption{\textbf{Full quantitative results on ImageNet-Segmentation using the dictionary from Strategy 4 (GPT-OSS 120B).} All four metrics are shown before (Base) and after (+OSP) applying Orthogonal Semantic Projection when using a dictionary of semantics created with GPT-OSS 120B, along with deltas ($\Delta$). The Gap Improvement is computed as $\Delta_{\text{pos}} - \Delta_{\text{neg}}$, where each $\Delta$ sums changes in mIoU, Accuracy, mAP, and AUROC. Favorable changes are highlighted in \textbf{bold}.}
\label{tab:full_results_gpt}
\resizebox{\textwidth}{!}{%
\setlength{\tabcolsep}{2.5pt}
\scriptsize
\begin{tabular}{ll@{\hskip 3pt}c@{\hskip 4pt}rr@{\hskip 2pt}r@{\hskip 6pt}rr@{\hskip 2pt}r@{\hskip 6pt}rr@{\hskip 2pt}r@{\hskip 6pt}rr@{\hskip 2pt}r@{\hskip 6pt}r}
\toprule
\textbf{Method} & \textbf{Model} & \textbf{Pr.} & \multicolumn{3}{c}{\textbf{mIoU}} & \multicolumn{3}{c}{\textbf{Accuracy}} & \multicolumn{3}{c}{\textbf{mAP}} & \multicolumn{3}{c}{\textbf{AUROC}} & \textbf{Gap} \\
\cmidrule(lr){4-6} \cmidrule(lr){7-9} \cmidrule(lr){10-12} \cmidrule(lr){13-15}
 & & & B & +O & $\Delta$ & B & +O & $\Delta$ & B & +O & $\Delta$ & B & +O & $\Delta$ & \textbf{Impr.} \\
\midrule
\multirow{4}{*}{LeGrad}
 & \multirow{2}{*}{CLIP}   & + & 58.66 & 58.01 & -0.65 & 77.52 & 75.09 & -2.43 & 82.49 & 85.39 & \textbf{+2.90} & \multirow{2}{*}{79.62} & \multirow{2}{*}{80.06} & \multirow{2}{*}{\textbf{+0.44}} & \multirow{2}{*}{\textbf{10.04}} \\
 &                         & -- & 40.88 & 37.37 & \textbf{-3.51} & 64.32 & 55.82 & \textbf{-8.50} & 67.99 & 70.22 & +2.23 &  &  &  &  \\
\cmidrule(l){2-16}
 & \multirow{2}{*}{SigLIP} & + & 49.51 & 51.68 & \textbf{+2.17} & 73.28 & 72.76 & -0.52 & 78.32 & 79.32 & \textbf{+1.00} & \multirow{2}{*}{74.42} & \multirow{2}{*}{74.67} & \multirow{2}{*}{\textbf{+0.25}} & \multirow{2}{*}{\textbf{15.22}} \\
 &                         & -- & 37.27 & 33.66 & \textbf{-3.61} & 63.47 & 57.31 & \textbf{-6.16} & 63.63 & 61.08 & \textbf{-2.55} &  &  &  &  \\
\midrule
\multirow{4}{*}{CheferCAM}
 & \multirow{2}{*}{CLIP}   & + & 48.71 & 51.98 & \textbf{+3.27} & 69.32 & 73.94 & \textbf{+4.62} & 80.36 & 82.59 & \textbf{+2.23} & \multirow{2}{*}{77.63} & \multirow{2}{*}{81.37} & \multirow{2}{*}{\textbf{+3.74}} & \multirow{2}{*}{\textbf{11.11}} \\
 &                         & -- & 44.88 & 43.29 & \textbf{-1.59} & 66.44 & 69.35 & +2.91 & 78.74 & 80.17 & +1.43 &  &  &  &  \\
\cmidrule(l){2-16}
 & \multirow{2}{*}{SigLIP} & + & 37.66 & 38.37 & \textbf{+0.71} & 60.53 & 65.87 & \textbf{+5.34} & 73.49 & 75.67 & \textbf{+2.18} & \multirow{2}{*}{55.47} & \multirow{2}{*}{60.36} & \multirow{2}{*}{\textbf{+4.89}} & \multirow{2}{*}{\textbf{3.54}} \\
 &                         & -- & 36.65 & 38.83 & +2.18 & 59.28 & 64.34 & +5.06 & 72.38 & 74.72 & +2.34 &  &  &  &  \\
\midrule
\multirow{4}{*}{Att.CAM}
 & \multirow{2}{*}{CLIP}   & + & 40.14 & 44.64 & \textbf{+4.50} & 68.67 & 71.37 & \textbf{+2.70} & 70.34 & 74.64 & \textbf{+4.30} & \multirow{2}{*}{52.68} & \multirow{2}{*}{63.99} & \multirow{2}{*}{\textbf{+11.31}} & \multirow{2}{*}{\textbf{15.19}} \\
 &                         & -- & 33.34 & 36.12 & +2.78 & 62.44 & 66.36 & +3.92 & 65.74 & 67.19 & +1.45 &  &  &  &  \\
\cmidrule(l){2-16}
 & \multirow{2}{*}{SigLIP} & + & 50.01 & 52.89 & \textbf{+2.88} & 70.33 & 73.69 & \textbf{+3.36} & 80.20 & 81.72 & \textbf{+1.52} & \multirow{2}{*}{80.49} & \multirow{2}{*}{85.81} & \multirow{2}{*}{\textbf{+5.32}} & \multirow{2}{*}{\textbf{11.75}} \\
 &                         & -- & 40.00 & 37.74 & \textbf{-2.26} & 62.57 & 68.52 & +5.95 & 72.70 & 70.34 & \textbf{-2.36} &  &  &  &  \\
\midrule
\multirow{4}{*}{GradCAM}
 & \multirow{2}{*}{CLIP}   & + & 44.68 & 53.78 & \textbf{+9.10} & 69.48 & 73.89 & \textbf{+4.41} & 74.94 & 81.26 & \textbf{+6.32} & \multirow{2}{*}{65.39} & \multirow{2}{*}{76.55} & \multirow{2}{*}{\textbf{+11.16}} & \multirow{2}{*}{\textbf{26.74}} \\
 &                         & -- & 33.83 & 35.83 & +2.00 & 58.86 & 59.43 & +0.57 & 67.34 & 69.02 & +1.68 &  &  &  &  \\
\cmidrule(l){2-16}
 & \multirow{2}{*}{SigLIP} & + & 38.69 & 44.37 & \textbf{+5.68} & 57.78 & 67.38 & \textbf{+9.60} & 71.43 & 73.32 & \textbf{+1.89} & \multirow{2}{*}{67.07} & \multirow{2}{*}{75.46} & \multirow{2}{*}{\textbf{+8.39}} & \multirow{2}{*}{\textbf{7.74}} \\
 &                         & -- & 34.78 & 38.02 & +3.24 & 54.72 & 68.28 & +13.56 & 68.53 & 69.55 & +1.02 &  &  &  &  \\
\midrule
\multirow{2}{*}{DAAM}
 & \multirow{2}{*}{SD 2}   & + & 65.71 & 66.39 & \textbf{+0.68} & 81.33 & 82.37 & \textbf{+1.04} & 88.55 & 88.68 & \textbf{+0.13} & \multirow{2}{*}{83.07} & \multirow{2}{*}{88.12} & \multirow{2}{*}{\textbf{+5.05}} & \multirow{2}{*}{\textbf{8.94}} \\
 &                         & -- & 59.49 & 58.18 & \textbf{-1.31} & 76.60 & 76.39 & \textbf{-0.21} & 86.39 & 85.87 & \textbf{-0.52} &  &  &  &  \\
\bottomrule
\end{tabular}}
\end{table}

\subsection{Hyperparameter Selection and Optimization Objective}

To determine the optimal hyperparameters for each dictionary strategy, we conducted a grid search over a constrained parameter space. This search was performed on a validation subset consisting of 20 representative samples. We defined a composite objective function to balance improvements in localization with the preservation of discriminative power:
\begin{equation}
\label{eq:optimization_objective}
    \text{Objective} = (\sum \Delta_{\text{pos}}) - (\sum \Delta_{\text{neg}}) + \Delta \text{AUROC}
\end{equation}
where $\Delta_{\text{pos}}$ and $\Delta_{\text{neg}}$ represent the changes in localization metrics (mIoU, Accuracy, mAP) for positive and negative prompts respectively, and $\Delta \text{AUROC}$ is the change in the paired AUROC metric.

We observe that, although hallucination metrics (e.g., mIoU on negative prompts) may occasionally increase after the intervention, the gap between positive- and negative-prompt performance consistently widens, indicating a relative improvement in disentanglement. The optimization objective could be further modified to strictly enforce decreases in negative prompt activations.

One observation is that when the max cosine similarity parameter ($\tau_{\text{cos}}$) is increased beyond 0.9, fewer atoms are filtered from the dictionary, often leading to a much more aggressive drop in hallucinations. However, our primary goal remains to maximize AUROC while maintaining a robust localization fidelity.

\subsection{Dictionary-Size Ablation Study}
To evaluate the sensitivity of OSP to the number of atoms in the dictionary of semantics, we conducted an ablation study on the CLIP ViT-B/16 architecture using the LeGrad attribution method and the Gemini-generated dictionary. We varied the dictionary size $T$ (number of atoms) across the values $\{3, 5, 10, 20, 40, 60, 80\}$. 

As shown in Table~\ref{tab:dict_ablation}, OSP yields substantial improvements in AUROC compared to the baseline (79.62) even with very small dictionary sizes (e.g., +2.04 with just 3 atoms). The peak improvement is achieved at 5 atoms (+2.54), after which the performance remains high and stable, verifying that our method is highly robust to the exact size of the dictionary of semantics.

\begin{table}[h]
\centering
\caption{\textbf{Dictionary-size ablation} on CLIP ViT-B/16 using the LeGrad attribution method and the Gemini dictionary (baseline AUROC is 79.62).}
\label{tab:dict_ablation}
\vspace{-2pt}
\setlength{\tabcolsep}{6.5pt}
\scriptsize
\begin{tabular}{lccccccc}
\toprule
\textbf{\# Atoms} & 3 & 5 & 10 & 20 & 40 & 60 & 80 \\
\midrule
\textbf{AUROC} & 81.66 & \textbf{82.16} & 82.09 & 81.90 & 81.51 & 81.40 & 81.40 \\
\textbf{Delta ($\Delta$)} & +2.04 & \textbf{+2.54} & +2.47 & +2.28 & +1.89 & +1.78 & +1.78 \\
\bottomrule
\end{tabular}
\end{table}

\subsection{Perturbation-Based Faithfulness Evaluation}
Following the established evaluation methodology in explainable AI (Petsiuk et al.~\cite{petsiuk2018rise}, Hedstr{\"o}m et al.~\cite{hedstrom2023quantus}), we conduct perturbation-based evaluation to assess the faithfulness of the attributions. Specifically, we measure:
\begin{itemize}
    \item \textbf{Insertion AUC ($\uparrow$)}: Pixels are iteratively inserted into a blurred baseline image in the order of their attribution importance. A higher area under the curve (AUC) indicates that the most important pixels are highly informative for the model's prediction.
    \item \textbf{Deletion AUC ($\downarrow$)}: Pixels are iteratively removed (blurred) from the image in the order of their attribution importance. A lower deletion AUC indicates that the attribution method successfully identifies the most critical pixels.
\end{itemize}
Importantly, to isolate the effect of OSP on spatial attribution accuracy without modifying the classification target, we perform the classification scoring using the \textit{original} (un-projected) text embeddings, whereas OSP is only used to compute the spatial attribution mask.

As reported in Table~\ref{tab:perturbation}, OSP consistently preserves or improves faithfulness metrics across all evaluated attribution methods. The most significant gains are observed for AttentionCAM (improving the Insertion$-$Deletion gap by \textbf{+2.01}) and GradCAM (improving the gap by \textbf{+1.12}). These results demonstrate that OSP not only improves localization but also yields more faithful attributions that align with the features causally used by the underlying model.

\begin{table}[h]
\centering
\caption{\textbf{Perturbation-based faithfulness evaluation} on CLIP ViT-B/16 on the ImageNet-Segmentation dataset using a blur baseline (20 steps). Higher Insertion AUC and lower Deletion AUC indicate better faithfulness.}
\label{tab:perturbation}
\vspace{-2pt}
\setlength{\tabcolsep}{3pt}
\scriptsize
\begin{tabular}{l cc c cc c cc}
\toprule
& \multicolumn{2}{c}{\textbf{Ins.\ AUC} $\uparrow$} & & \multicolumn{2}{c}{\textbf{Del.\ AUC} $\downarrow$} & & \multicolumn{2}{c}{\textbf{Ins$-$Del} $\uparrow$} \\
\cmidrule(lr){2-3} \cmidrule(lr){5-6} \cmidrule(lr){8-9}
\textbf{Method} & Base & +OSP & & Base & +OSP & & Base & +OSP \\
\midrule
LeGrad       & 25.08 & \textbf{25.15} & & 18.48 & 18.56 & & 6.60 & 6.59 \\
CheferCAM    & 24.47 & \textbf{24.78} & & 19.25 & \textbf{19.10} & & 5.22 & \textbf{5.68} \\
GradCAM      & 24.47 & \textbf{24.69} & & 20.51 & \textbf{19.61} & & 3.96 & \textbf{5.08} \\
AttentionCAM & 23.03 & \textbf{23.63} & & 22.44 & \textbf{21.03} & & 0.59 & \textbf{2.60} \\
\bottomrule
\end{tabular}
\end{table}

\subsection{Generalization to Large Vision-Language Models (LVLMs)}
To demonstrate the generalizability of OSP to modern Large Vision-Language Models (LVLMs), we evaluate our method on the CLIP ViT-L/14 visual backbone of LLaVA-1.5~\cite{liu2024llava}. LLaVA-1.5 keeps the CLIP visual encoder frozen, so the encoder-level spatial attributions are structurally identical to standalone CLIP. 

Because OSP acts as a geometric intervention modifying solely the text query embedding, it can be integrated directly with the LLaVA visual encoder without modifying any of the autoregressive language model's weights. Table~\ref{tab:lvlm} reports the results using the LeGrad attribution method and the Gemini dictionary. OSP achieves consistent gains across all positive prompt metrics (+0.38 mIoU, +0.32 Accuracy, +0.64 mAP) while successfully suppressing attributions on negative prompts, raising the AUROC from 73.72 to \textbf{73.91}. These results demonstrate that OSP generalizes effectively to downstream multimodal systems that leverage contrastive vision-language representation spaces.

\begin{table}[h]
\centering
\caption{\textbf{OSP evaluation on the visual encoder of LLaVA-1.5} (CLIP ViT-L/14, LeGrad, Gemini dictionary).}
\label{tab:lvlm}
\vspace{-2pt}
\setlength{\tabcolsep}{5pt}
\scriptsize
\begin{tabular}{lcccc}
\toprule
& \textbf{mIoU} & \textbf{Accuracy} & \textbf{mAP} & \textbf{AUROC} \\
\midrule
Base (Pos) & 57.52 & 75.59 & 84.71 & \multirow{2}{*}{73.72 $\to$ \textbf{73.91}} \\
+OSP (Pos) & \textbf{57.90} & \textbf{75.91} & \textbf{85.35} & \\
\midrule
Base (Neg) & 34.09 & 51.12 & 71.25 & \\
+OSP (Neg) & \textbf{33.85} & \textbf{50.91} & \textbf{70.69} & \\
\bottomrule
\end{tabular}
\end{table}

\section{Extra Use Cases: Evaluation on Concept-Level Annotations on PartImageNet++}
\label{sec:extra_use_cases}

While ImageNet provides comprehensive class-level labels, it lacks fine-grained part annotations. To evaluate the robustness of \textbf{OSP} in more granular scenarios, we extended our evaluation using the PartImageNet++ dataset~\cite{li2024partimagenet}. This dataset scales up part-based models for robust recognition and provides detailed part-level or concept-level segmentations.

We compare the performance of baseline methods (LeGrad~\cite{legrad2025} and CHILI~\cite{kazmierczak2025enhancing}) against their OSP-enhanced versions. As shown in Table~\ref{tab:partimagenet_results}, OSP consistently improves accuracy across different attribution and disentanglement techniques. Specifically, for usual attribution methods like LeGrad, OSP increases the localization accuracy. 
Notably, OSP is complementary to such methods. CHILI~\cite{kazmierczak2025enhancing} is a disentanglement technique specifically designed for image-side embeddings. While CHILI disentangles visual concepts in the image feature space, our proposed OSP orthogonalizes the concepts within the text-side embeddings. Our results demonstrate that when the text-side OSP is integrated with the image-side disentanglement of CHILI, the combined technique yields significant performance enhancements over using CHILI in isolation. The visual results can be seen in Figures~\ref{fig:visual_jet} and \ref{fig:timber_wolf}.

\begin{table}[htbp]
\centering
\caption{\textbf{Quantitative Results on PartImageNet++.} We report mIoU, Pixel Accuracy, and mAP for LeGrad, CHILI, and their OSP-augmented versions. The hyperparameters are binarization threshold ($\tau_{\text{act}}$), number of OMP atoms, (= number of semantics), ($T$), and maximum cosine similarity threshold ($\tau_{\text{cos}}$).}
\label{tab:partimagenet_results}
\resizebox{0.8\textwidth}{!}{
\begin{tabular}{lcccccc}
\toprule
\textbf{Method} & \textbf{mIoU} & \textbf{Pixel Acc.} & \textbf{mAP} & $\tau_{\text{act}}$ & $T$ & $\tau_{\text{cos}}$ \\
\midrule
LeGrad & 48.30 & 76.55 & 80.30 & --- & --- & --- \\
LeGrad + OSP & \textbf{52.09} & \textbf{88.37} & \textbf{82.56} & 0.725 & 3 & 0.85 \\
\midrule
CHILI & 43.42 & 82.79 & 83.96 & --- & --- & --- \\
CHILI + OSP & \textbf{51.10} & \textbf{86.32} & \textbf{92.65} & 0.5 & 4 & 0.7 \\
\bottomrule
\end{tabular}
}
\end{table}

\section{Detailed User Study Setup and Results}
\label{sec:user_study_appendix}

\subsection{Study Setup}

\textbf{Stimuli.} We compiled a dataset of 50 randomly selected images from the validation sets of the Pascal VOC 2012 \cite{everingham2015pascal} and MS COCO \cite{lin2014microsoft} datasets. Each image was specifically chosen to contain at least 2 unique objects. 

\textbf{Participants.} We recruited 200 participants. There were no specific inclusion criteria. The recruited participants were from 21 countries across 5 continents.

\textbf{Study Procedure.} Prior to the main evaluation, we provided a detailed tutorial to familiarize participants with the interface and the task (see Figure~\ref{fig:user_study_ui}). Each participant is inspecting 31 images. For each image, users were presented with saliency maps generated by one standard attribution method with or without our proposed OSP. Participants did not know the text prompts used to generate the saliency maps. To ensure that people were not answering randomly, we added an attention check question which filtered out some users eventually. Among the 31 heatmaps shown, one heatmap with a great mIoU was an attention check question, and users answering that question wrongly were not counted. 98\% of the participants answered it correctly. For each image, users face the following steps:
\begin{enumerate}
    \item \textbf{Object Identification:} Users were asked, \textit{``Based on the heatmap, which class is the model focusing on?''} The available options were:
    \begin{itemize}
        \item Target 1 (existing in the image)
        \item Target 2 (existing in the image)
        \item None of them
    \end{itemize}
    \item \textbf{Confidence Rating:} Before confirming their selection, users reported their confidence level in their answer, on a scale from 1 to 5 (where 1 is not confident and 5 is highly confident).
    \item \textbf{Post-Disclosure Assessment:} After the user submitted their answer, the correct target class was revealed. We then asked, \textit{``How well does the heatmap work to understand the correct answer?''} to assess the perceived quality and fidelity of the explanation. Participants had to answer on a scale from 1 to 5. 
\end{enumerate}

\begin{figure}[htbp]
\centering
\includegraphics[width=0.45\textwidth]{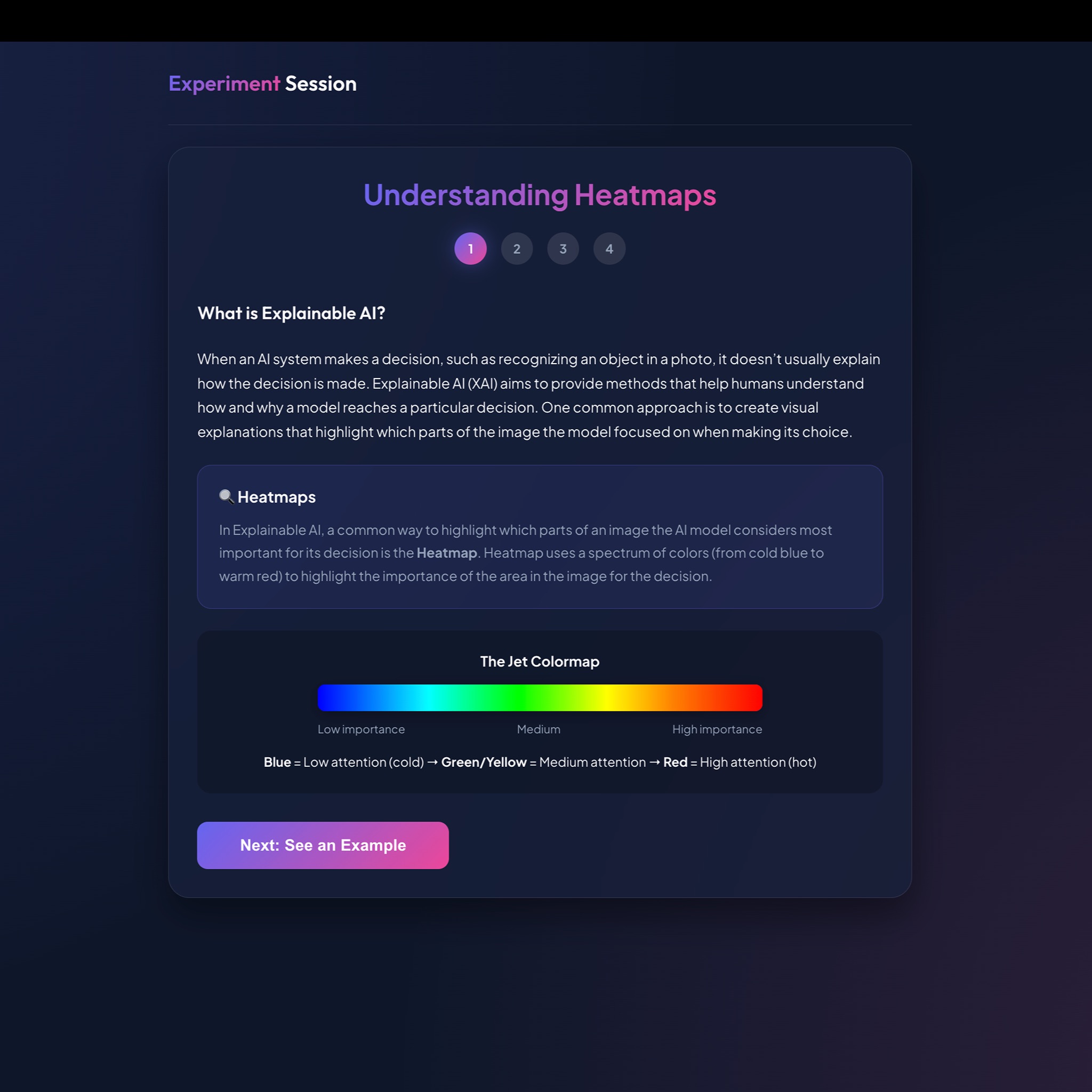}\hfill
\includegraphics[width=0.45\textwidth]{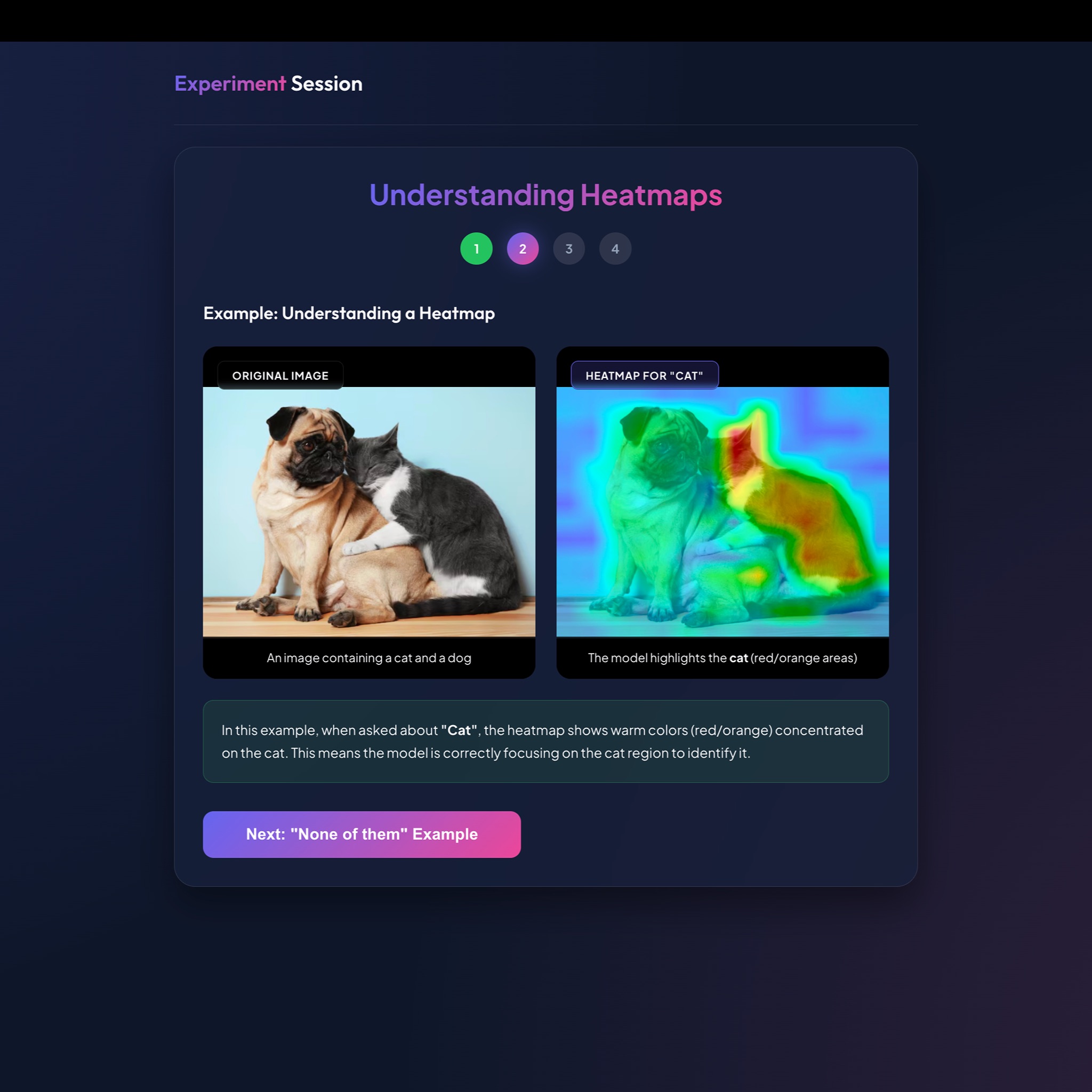}\hfill
\includegraphics[width=0.45\textwidth]{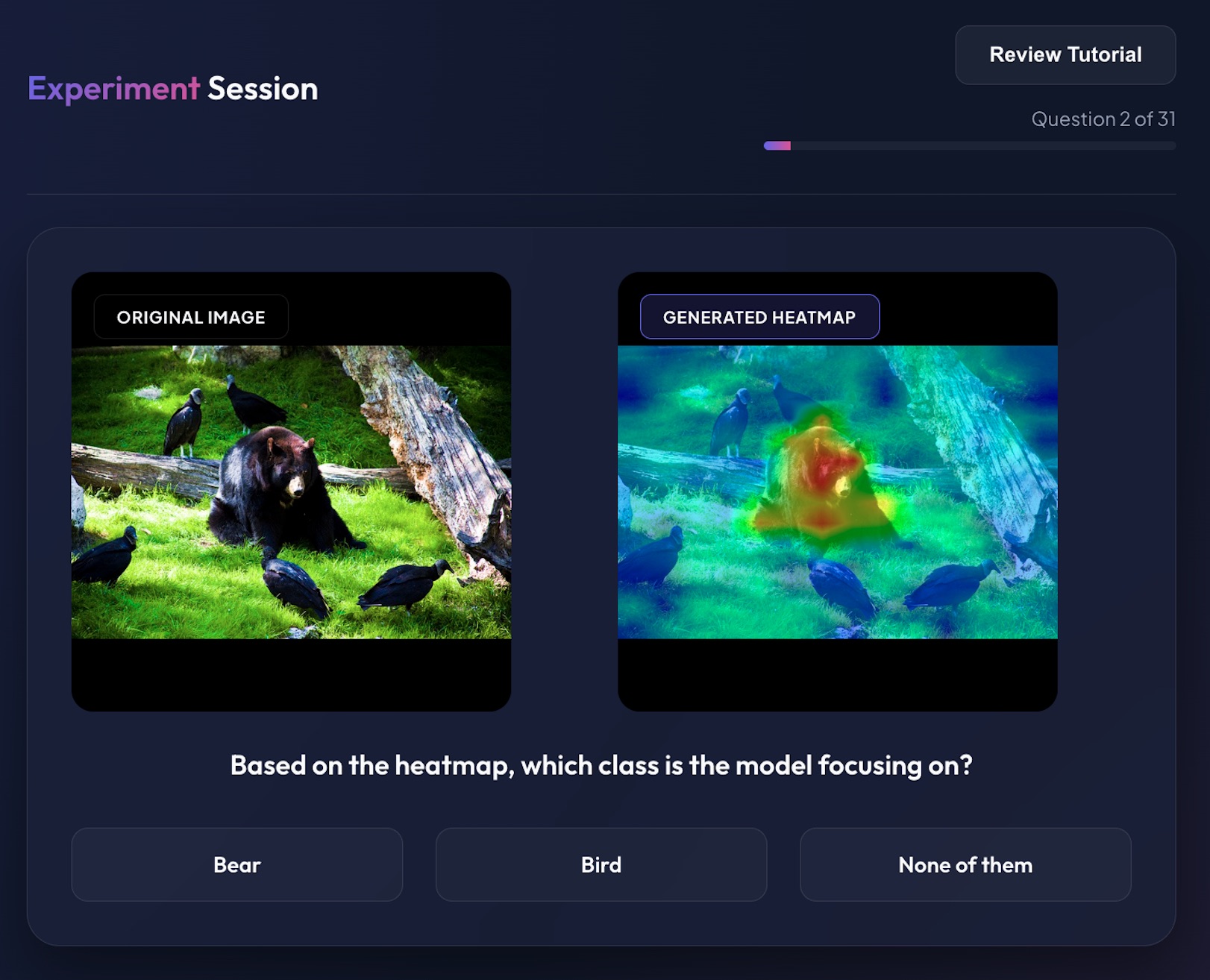}
\caption{User study interface. From left to right: the first two images show the tutorial provided to participants to explain the task, and the third image displays an example of the main evaluation interface.}
\label{fig:user_study_ui}
\end{figure}

Each of the 50 images was seen by 4 participants, allowing us to compute the mean accuracy (the number of correct guesses of the class among the four answers), the mean confidence, and the mean trust value. 

This methodology allows us to quantitatively measure how semantic hallucination in baseline methods misleads users into selecting incorrect objects or ``None of them,'' and how OSP improves both the accuracy of human identification and the users' confidence in the model's explanations.

\subsection{Participant Demographics}
We collected responses from 200 participants from 21 unique countries across 5 continents. Figure~\ref{fig:demographics_overall1} provides a comprehensive overview of the participant demographics, highlighting the diversity in age, gender, educational background, AI expertise, and geographic location. We also present statistics related to the survey itself, such as perceived difficulty, duration, and response consistency.

\begin{figure}[htbp]
\centering
\includegraphics[width=0.48\textwidth]{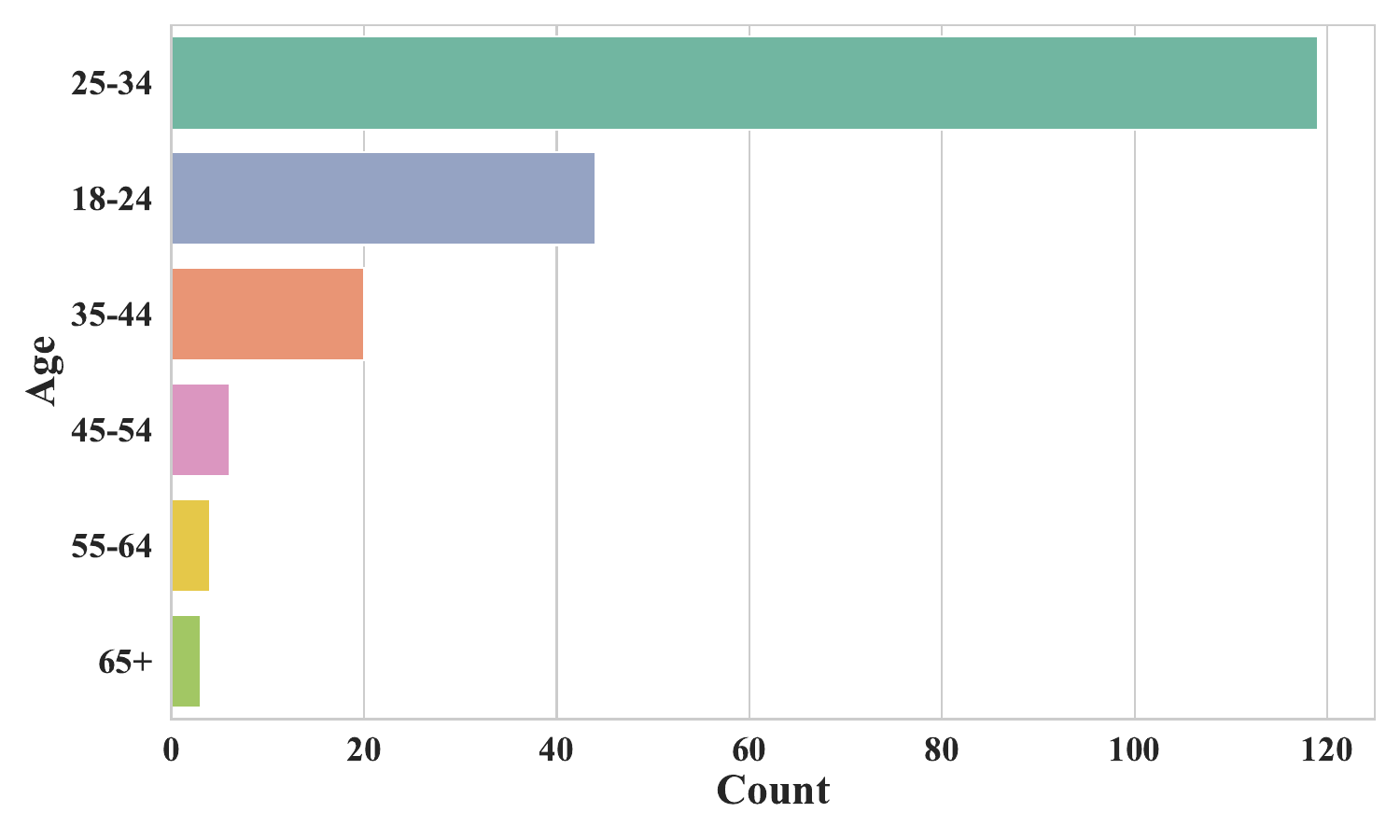}\hfill
\includegraphics[width=0.48\textwidth]{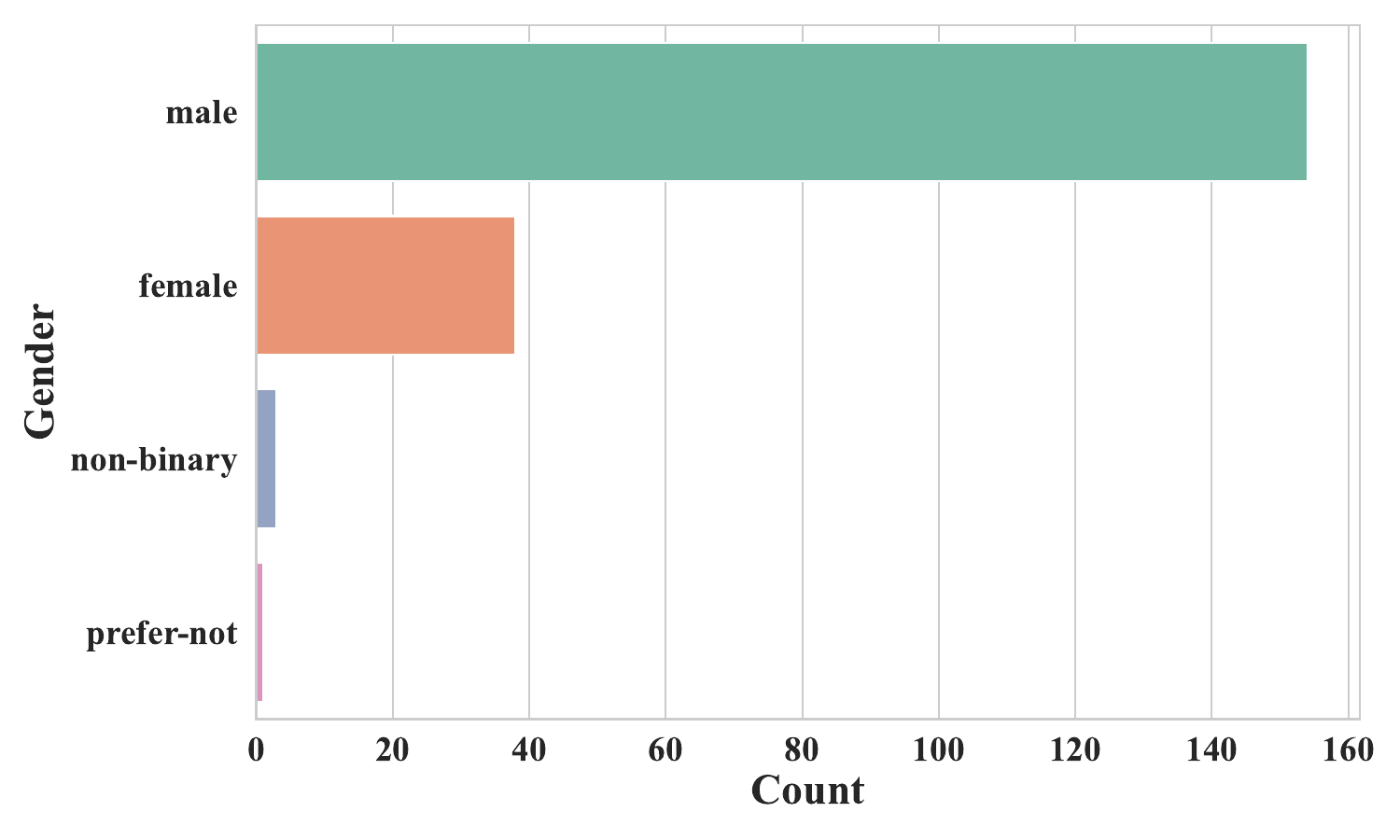}
\\
\vspace{0.3cm}
\includegraphics[width=0.48\textwidth]{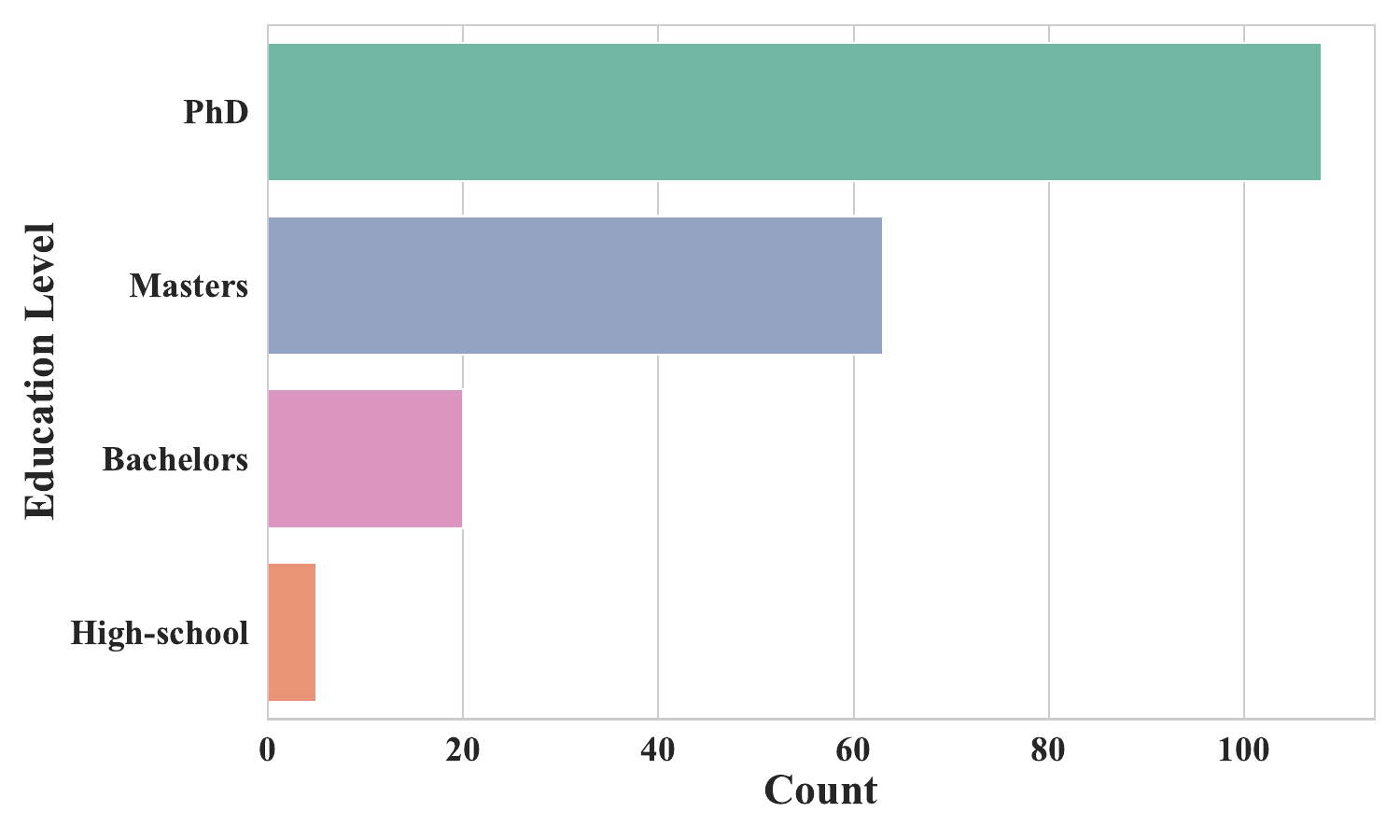}\hfill
\includegraphics[width=0.48\textwidth]{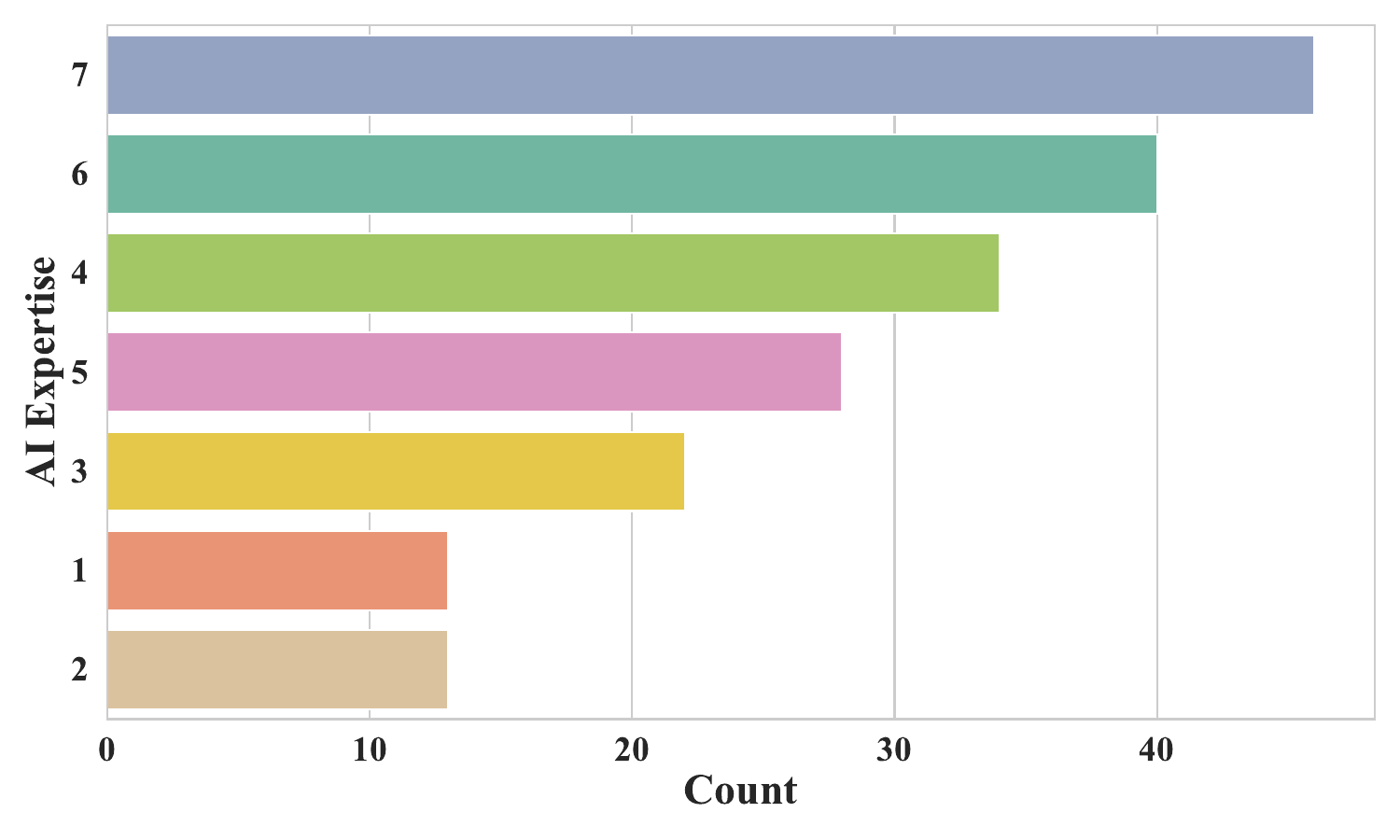}
\caption{Participant demographics including age, gender, education level, and AI expertise.}
\label{fig:demographics_overall1}
\end{figure}

\subsection{Detailed Metrics, Statistical Analysis, and Subgroup Analysis}
In addition to the overall method comparison presented in the main text, Figure~\ref{fig:method_detailed} shows detailed breakdowns of confidence and trust scores. 

\textbf{Statistical Analysis.} A two-way ANOVA reports a significant effect of the method ($F(4, 1950) = 87.1$, $p<1$e-3) and a significant effect of the use of OSP ($F(1, 1950) = 149.8$, $p<1$e-3) on accuracy. Similar results hold for confidence ($F(4, 1950) = 20.4$, $p<1$e-3 and $F(4, 1950) = 38.2$, $p<1$e-3), and trust scores ($F(4, 1950) = 74.8$, $p<1$e-3 and $F(4, 1950) = 160.6$, $p<1$e-3) compared to the baseline methods across all architectures.

Furthermore, we analyzed the results across different demographic subgroups, as shown in Figures~\ref{fig:subgroup_ai} and~\ref{fig:subgroup_edu}, indicating that OSP's improvements hold consistently regardless of the participants' AI expertise or educational background.

\begin{figure}[htbp]
\centering
\includegraphics[width=0.48\textwidth]{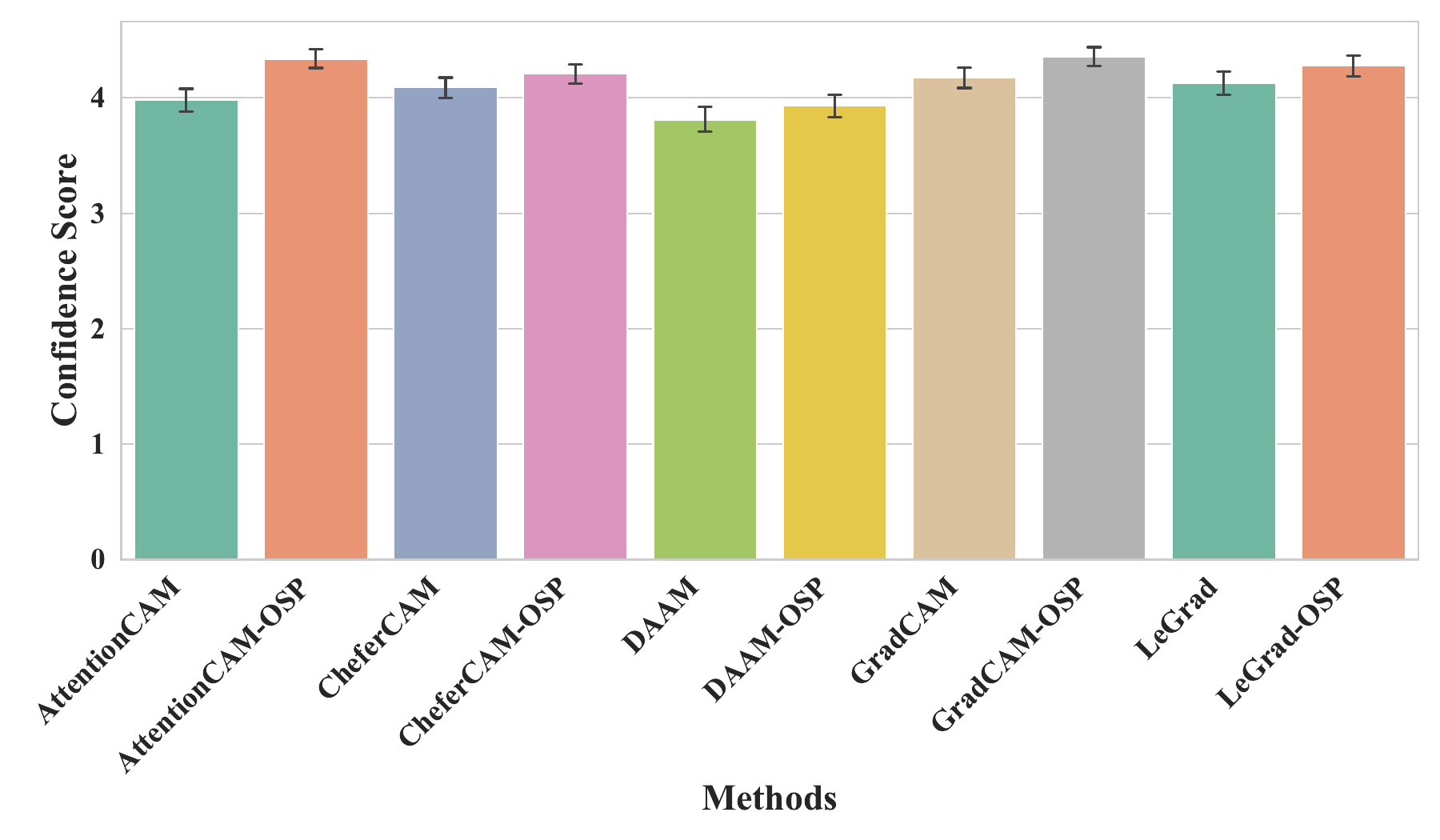}\hfill
\includegraphics[width=0.48\textwidth]{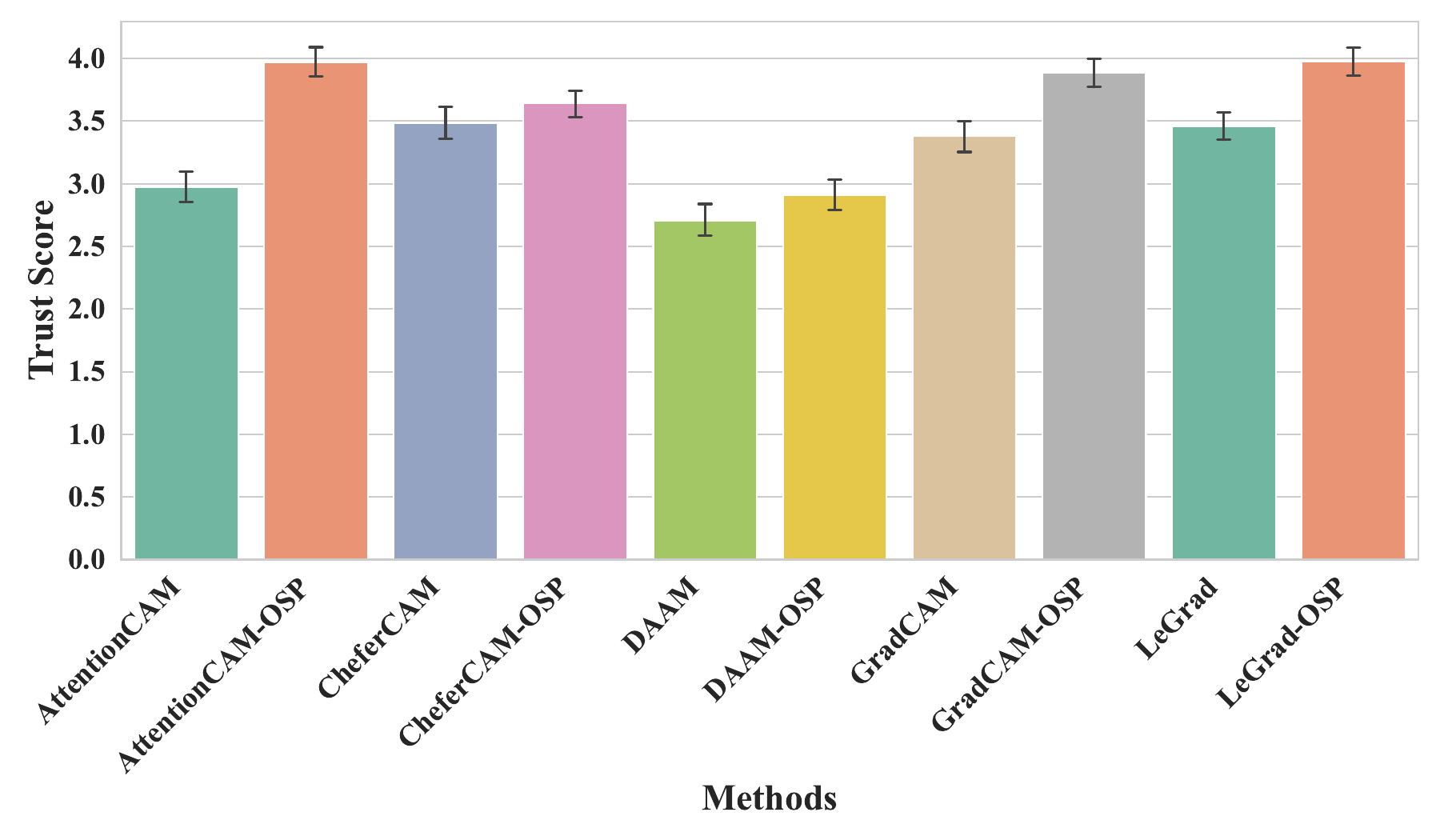}
\caption{Detailed comparison of confidence and trust scores across methods.}
\label{fig:method_detailed}
\end{figure}

\begin{figure}[htbp]
\centering
\includegraphics[width=0.45\textwidth]{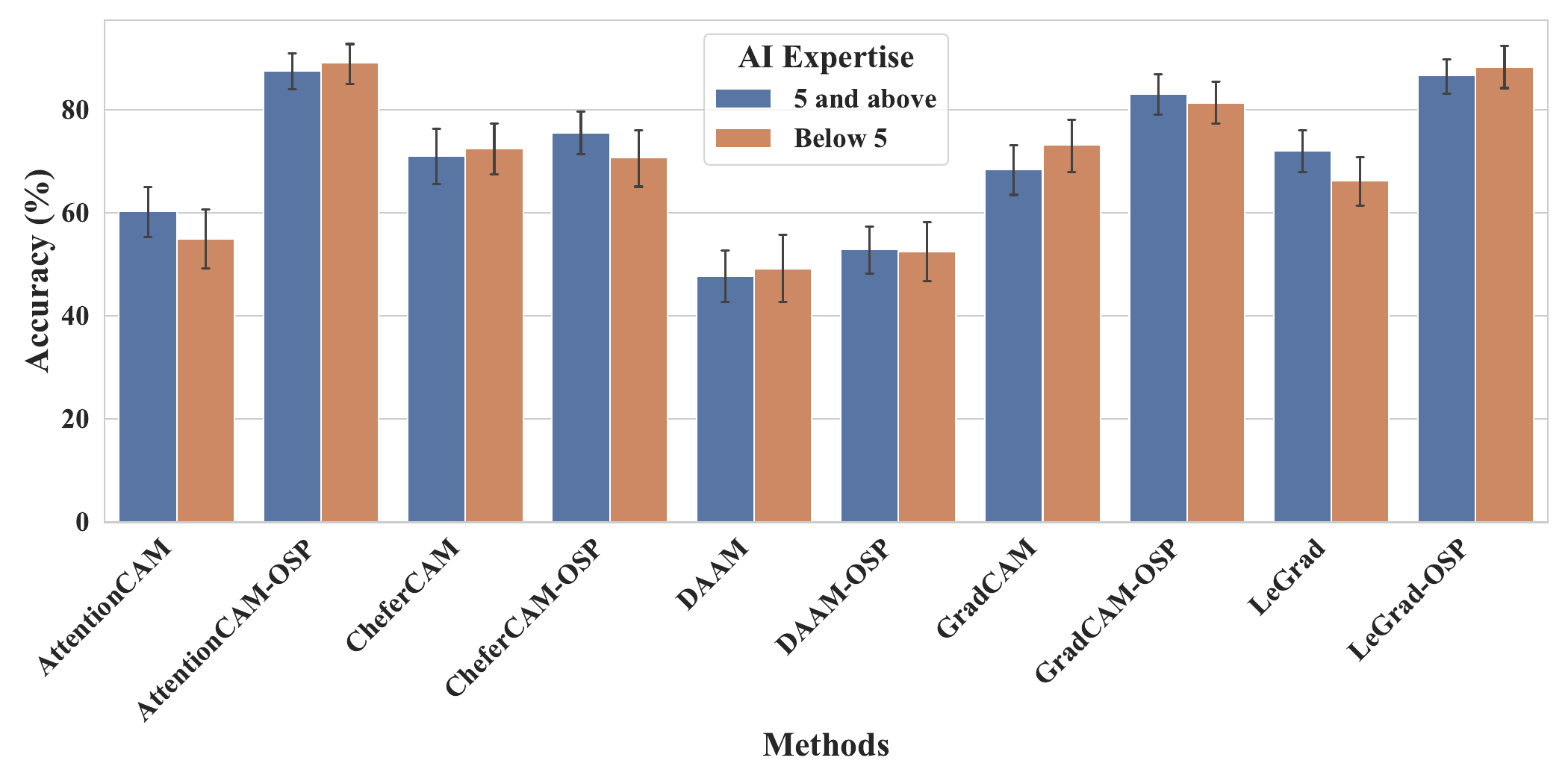}\hfill
\includegraphics[width=0.45\textwidth]{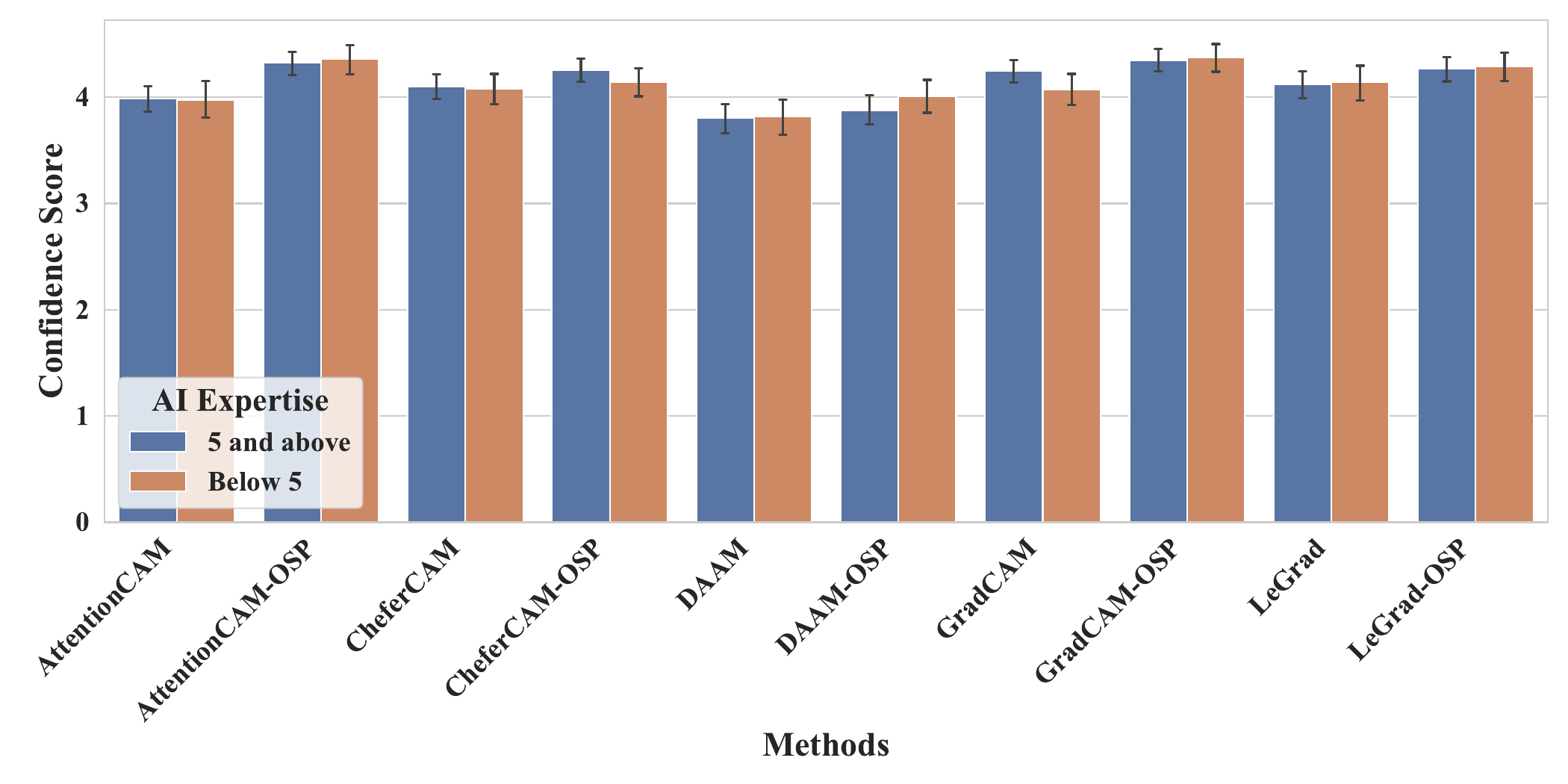}\hfill
\includegraphics[width=0.45\textwidth]{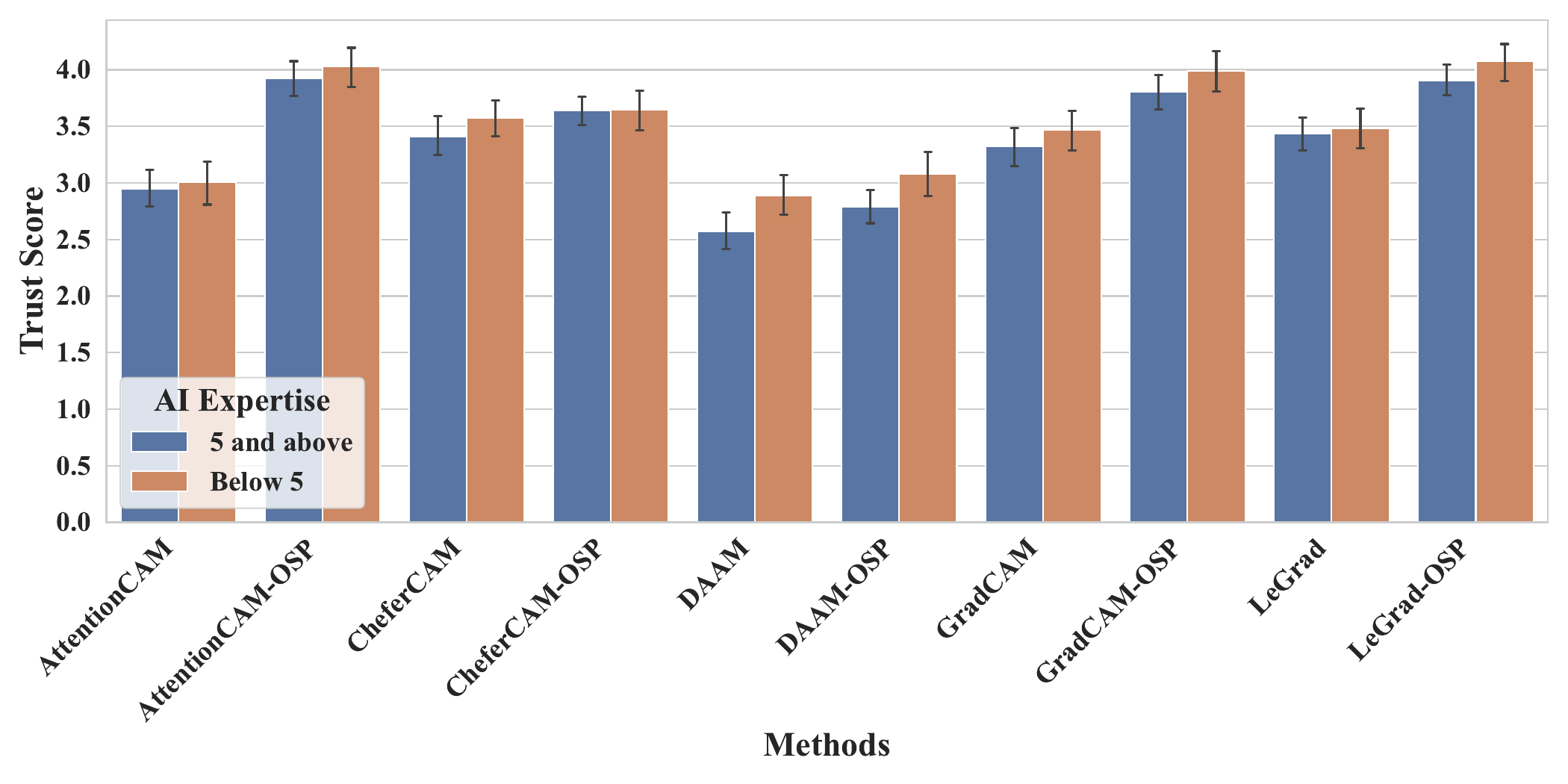}
\caption{Performance metrics categorized by participants' AI expertise.}
\label{fig:subgroup_ai}
\end{figure}

\begin{figure}[htbp]
\centering
\includegraphics[width=0.45\textwidth]{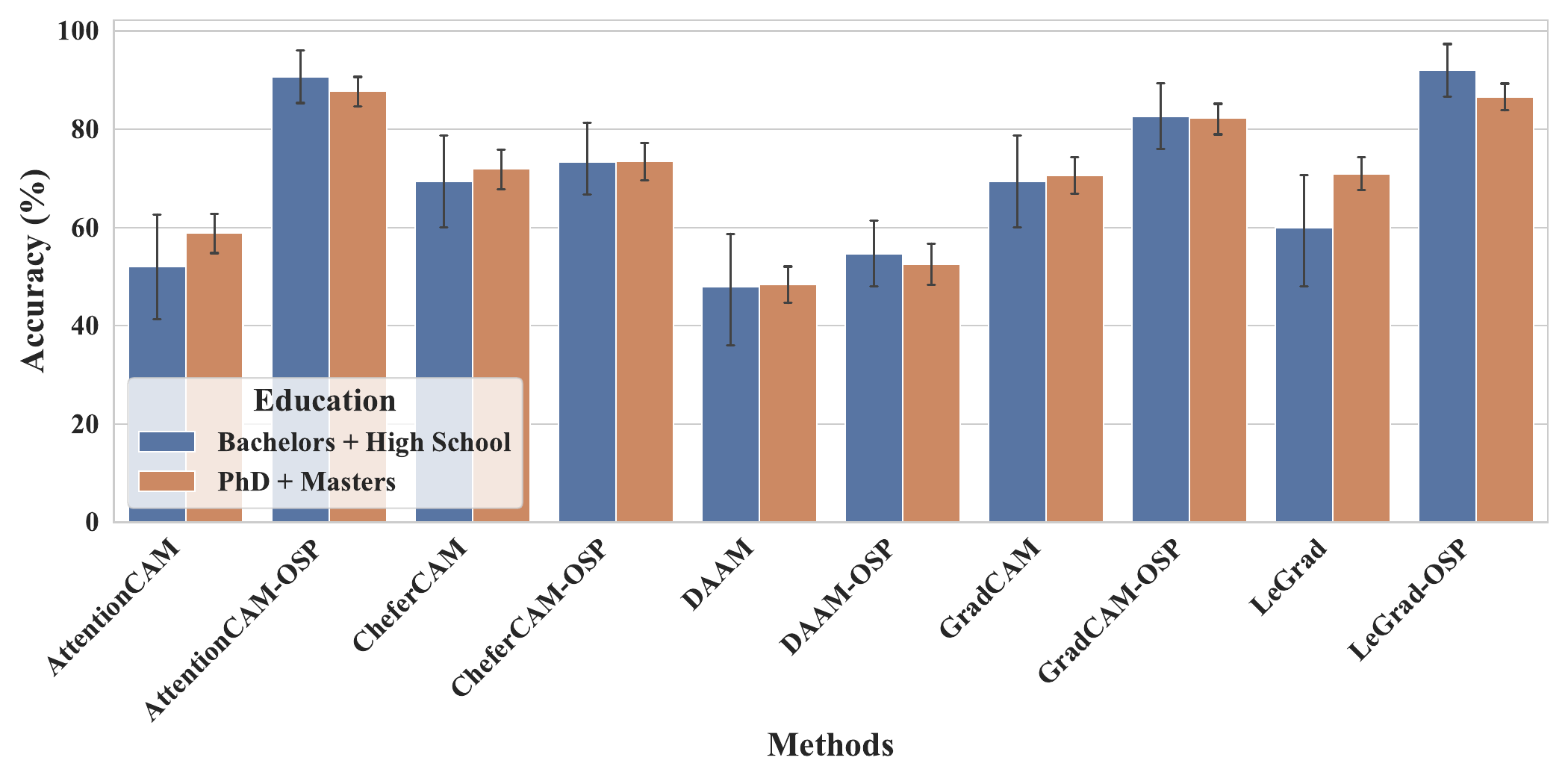}\hfill
\includegraphics[width=0.45\textwidth]{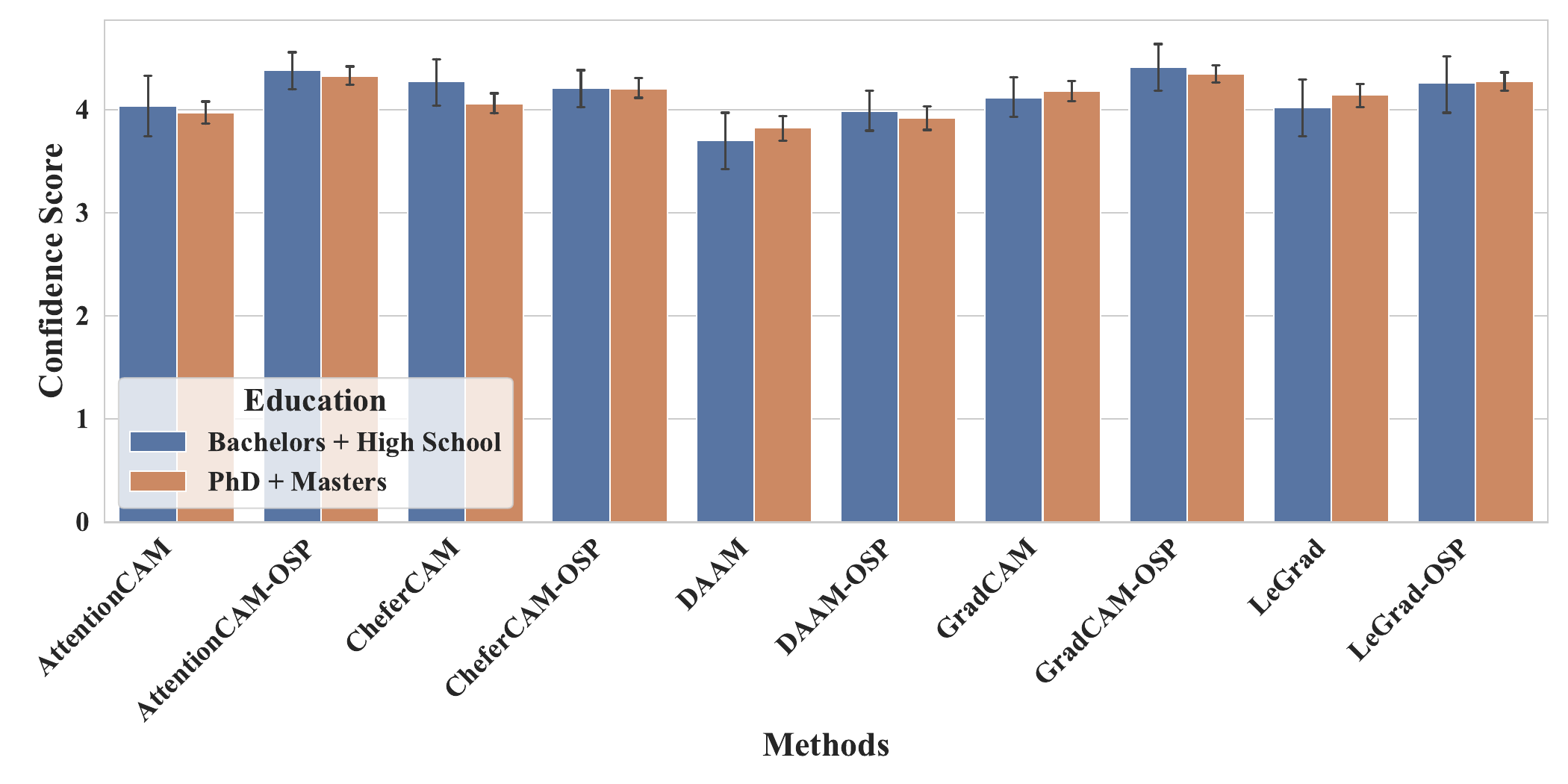}\hfill
\includegraphics[width=0.45\textwidth]{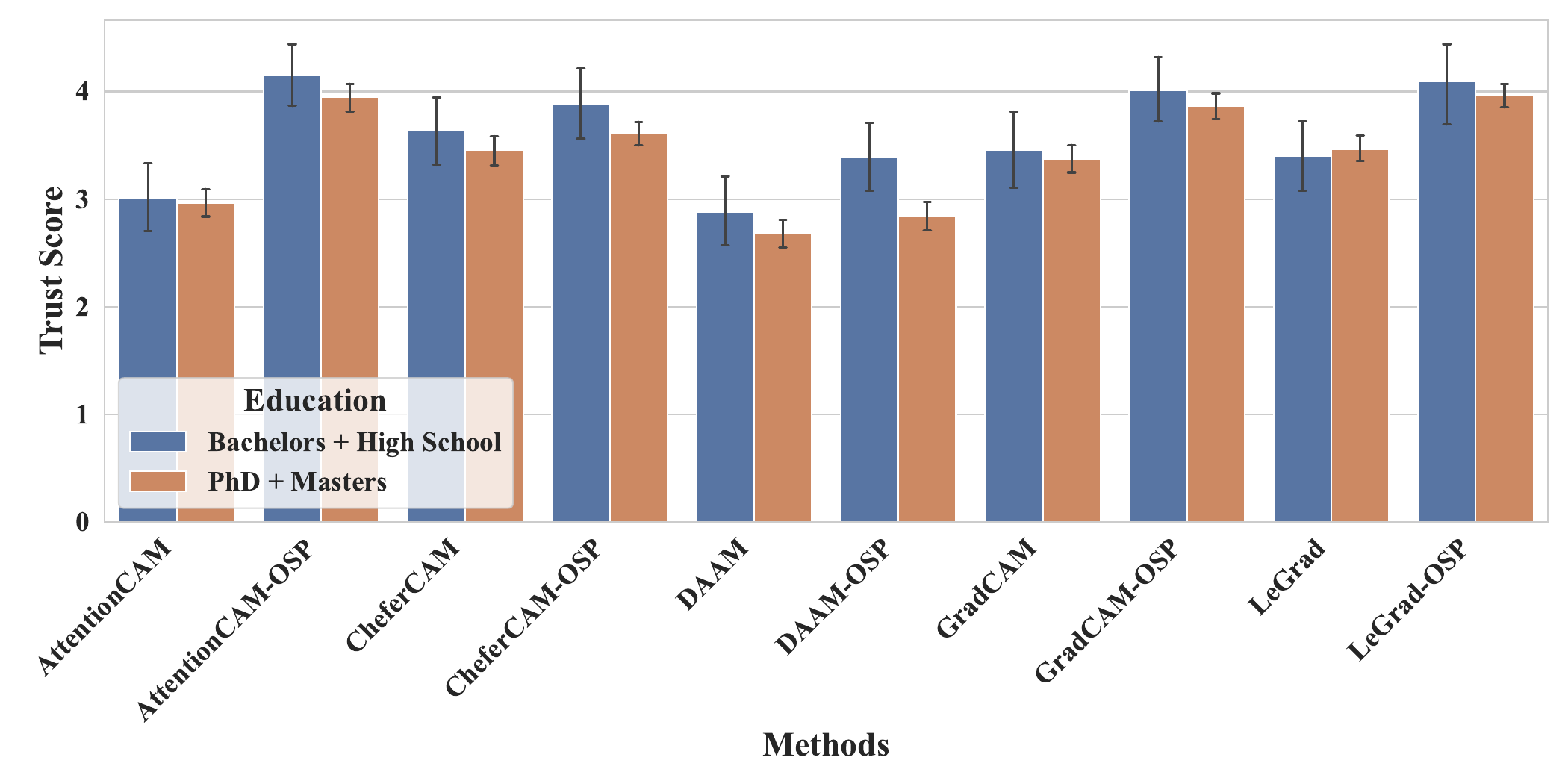}
\caption{Performance metrics categorized by participants' education levels.}
\label{fig:subgroup_edu}
\end{figure}

\section{Extra Visual Results}
\label{sec:extra_visuals}

In this section, we provide additional qualitative results to further demonstrate the effectiveness of our proposed Orthogonal Semantic Projection (OSP) across different scenarios. Figure~\ref{fig:visual_all3} showcases a visual comparison across all methods on a complex scene containing multiple objects, highlighting how OSP successfully suppresses spurious activations for both target and non-existent distractor classes. We also provide fine-grained part-level localization results on the PartImageNet++ dataset for the hare and timber wolf categories in Figures~\ref{fig:visual_jet} and \ref{fig:timber_wolf}, respectively, demonstrating the precision of OSP in granular segmentations.

\begin{figure}[htbp]
\centering
\includegraphics[width=1.08\textwidth]{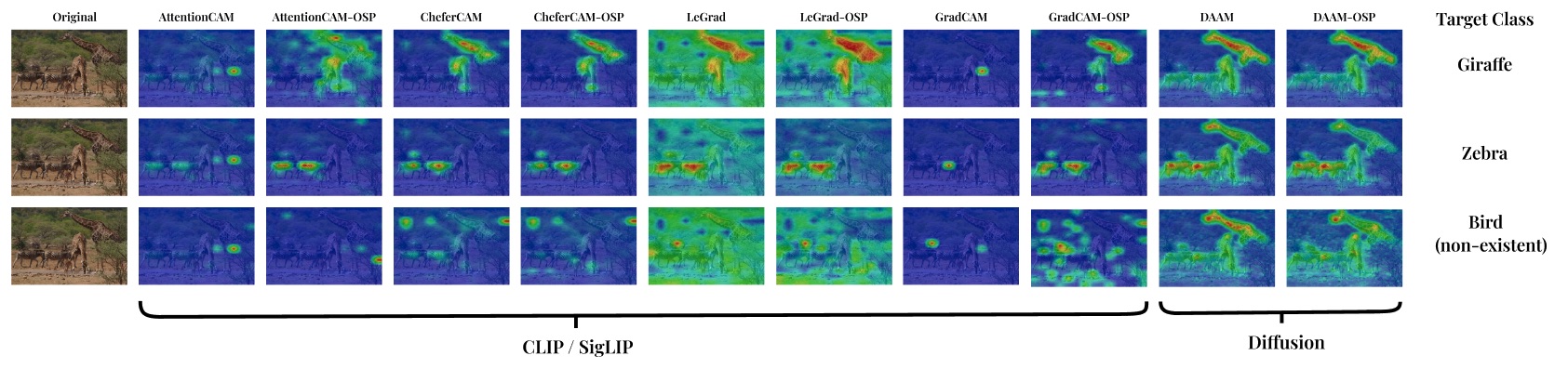}
\caption{Visual comparison across all methods on a \textbf{giraffe--zebra} scene. Rows correspond to queries for ``Giraffe'' (target), ``Zebra'' (target), and ``Bird'' (non-existent distractor). Even for semantically distant distractors, standard methods produce spurious activations that OSP removes.}
\label{fig:visual_all3}
\end{figure}

\begin{figure}[htbp]
\centering
\includegraphics[width=0.85\textwidth]{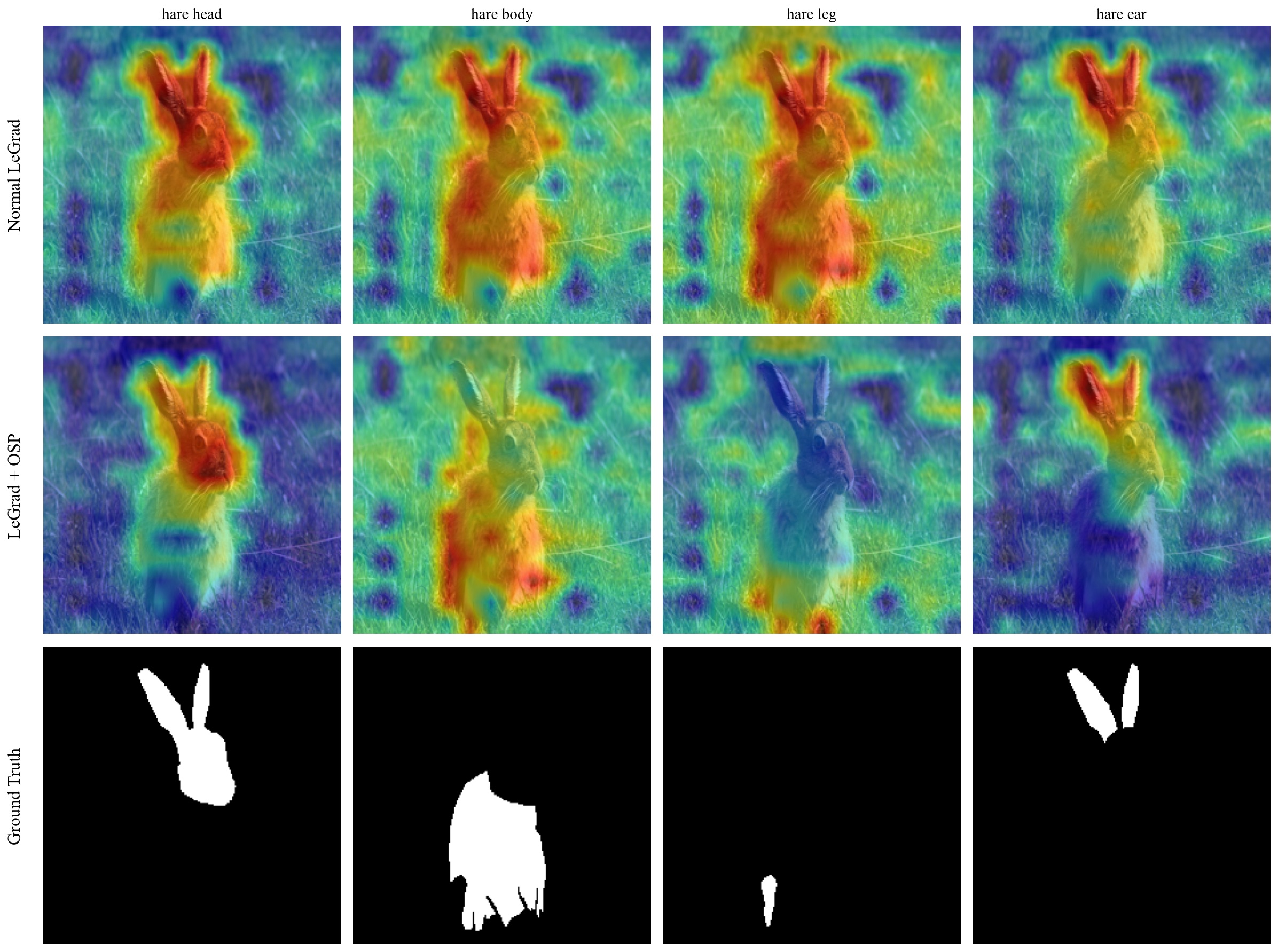}
\caption{Visual results on the \textbf{hare} category from PartImageNet++. The heatmap demonstrates precise part-level localization when using OSP, effectively distinguishing between different hare (rabbit) parts.}
\label{fig:visual_jet}
\end{figure}

\begin{figure}[htbp]
\centering
\includegraphics[width=0.85\textwidth]{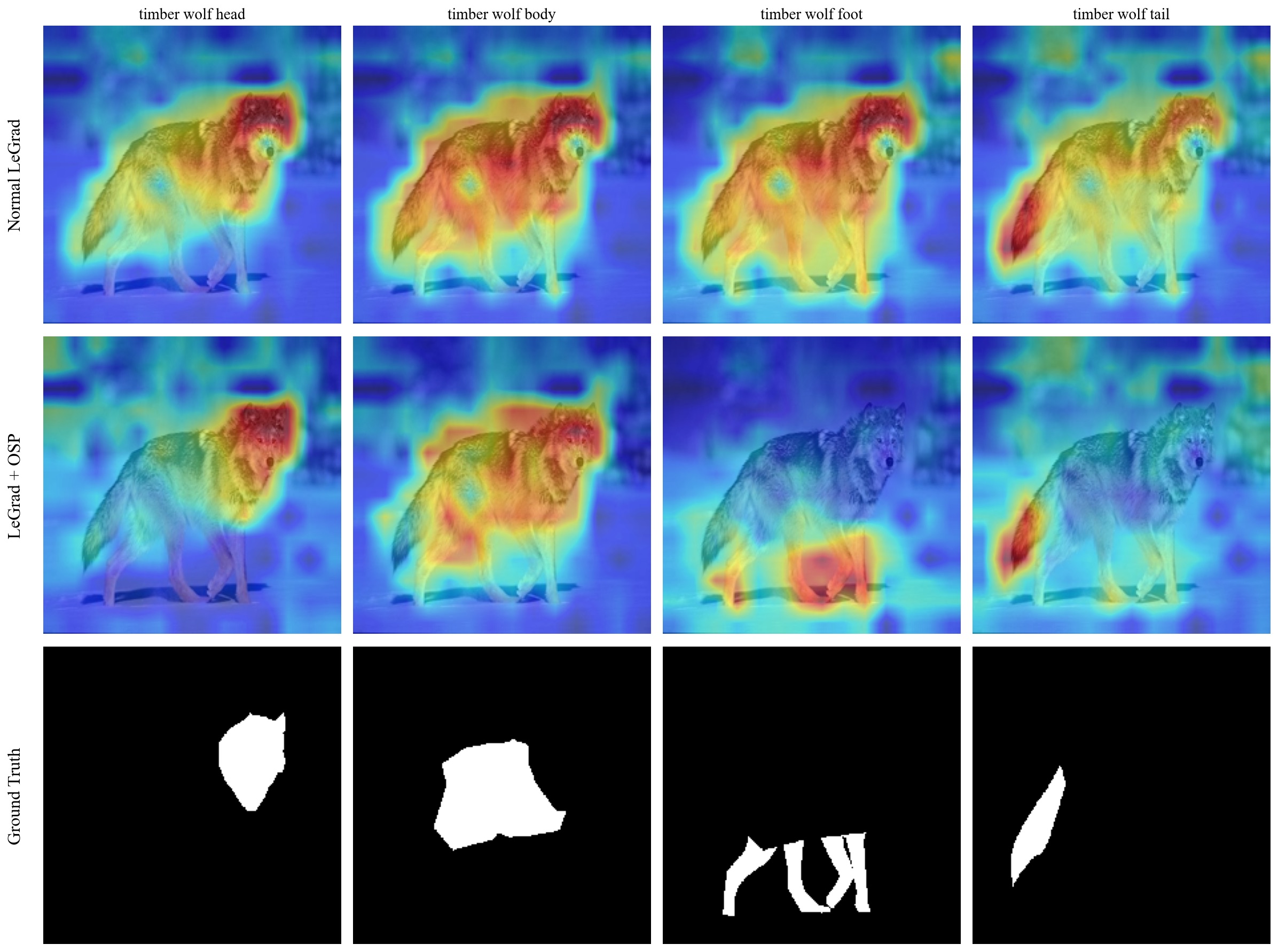}
\caption{Visual results on the \textbf{timber wolf} category from PartImageNet++. The heatmap demonstrates precise part-level localization when using OSP, effectively distinguishing between different timber wolf parts.}
\label{fig:timber_wolf}
\end{figure}

\section{LLM Prompt for Dictionary of Semantics Generation}
\label{sec:prompt_appendix}

We utilized the following prompt (see Figure~\ref{fig:llm_prompt}) to generate the dictionaries of semantics for each ImageNet class. The prompt was identical for both Gemini 3 Flash and GPT-OSS 120B generation strategies.

\begin{figure}[htbp]
\begin{tcolorbox}[title=LLM Prompt for Dictionary of Semantics Generation, fontupper=\footnotesize]
You are helping build a visual concept dictionary of semantics for zero-shot image segmentation. Use the 445 existing ImageNet classes in \texttt{"scripts/data/gtsegs\_ijcv.mat"}. 

For the ImageNet class \texttt{"\{class\_name\}"} (WordNet ID: \texttt{\{wnid\}}), generate exactly 40 concepts organized into 3 groups. Each concept should be 1-2 words.

\textbf{RULES:}
\begin{itemize}
    \item Do NOT include the class name itself
    \item Each concept must be distinct (no near-duplicates)
    \item Focus on what would be relevant in \textbf{PHOTOGRAPHS} (not abstract concepts)
    \item Prefer concrete, visually-grounded nouns over adjectives
\end{itemize}

\textbf{GROUP 1 - VISUAL CONFUSERS (15 concepts):}\\
Objects or textures that look visually similar to a \texttt{\{class\_name\}} and could be mistaken for it in photographs. Think about shape, color, texture, and size.

\textbf{GROUP 2 - CO-OCCURRING CONTEXT (15 concepts):}\\
Objects, scenes, or environments that frequently appear in the SAME photograph as a \texttt{\{class\_name\}} but are NOT the \texttt{\{class\_name\}} itself.

\textbf{GROUP 3 - SEMANTIC HIERARCHY (10 concepts):}\\
5 broader categories that \texttt{\{class\_name\}} belongs to (hypernyms/superclasses), and 5 more specific variants or subtypes of \texttt{\{class\_name\}} (hyponyms/subspecies).

Return as a JSON object:
\begin{verbatim}
{
  "visual_confusers": [...],
  "co_occurring_context": [...],  
  "semantic_hierarchy": [...]
}
\end{verbatim}
\end{tcolorbox}
\caption{The structured LLM prompt used to generate visual concept dictionaries for each ImageNet class.}
\label{fig:llm_prompt}
\end{figure}

\section{Runtime Analysis}
\label{sec:runtime_appendix}
The runtime of the Orthogonal Semantic Projection (OSP) was evaluated across various hardware configurations and dictionary sizes ($K$). Table~\ref{tab:omp_runtime} summarizes the average duration (ms) for OMP-based residual embedding creation, measured with an embedding dimension $d=512$ and 8 atoms. Across all architectures (CLIP, SigLIP, and Stable Diffusion 2), the overhead added by the OMP-based purification is consistently negligible compared to the total model inference and gradient computation time, ensuring that OSP remains a practical and efficient addition to the attribution pipeline.

\begin{table}[h]
\centering
\caption{\textbf{Runtime analysis of OMP-based residual embedding creation (ms) across different hardware and dictionary sizes (K).} Measured with $d=512$ and 8 atoms.}
\label{tab:omp_runtime}
\begin{tabular}{lcccc}
\toprule
$K$ & M2 Pro (CPU) & Apple MPS (GPU) & i9-10980XE (CPU)  &  NVIDIA RTX 3090 (GPU) \\ \midrule
100  & 0.77 & 2.13 & 2.40 & 2.63 \\
1000 & 0.87 & 2.44 & 2.19 & 2.82 \\
5000 & 1.38 & 4.13 & 2.33 & 4.07 \\ \bottomrule
\end{tabular}
\end{table}

\section{Evaluation on Multi-Object Datasets: MS COCO 2017}
\label{sec:coco_appendix}

We utilized four distinct dictionary creation strategies: (1) a dictionary consisting solely of the 80 MS COCO class names, (2) an augmented dictionary combining these classes with semantic relationships (hyponyms and hypernyms) from WordNet, excluding synonyms, (3) a dictionary generated by Gemini 3 Flash, and (4) a dictionary generated by prompting GPT-OSS 120B. The results the average performance per dictionary strategy is in Table~\ref{tab:avg_per_dictionary_coco} and the full results are in Tables~\ref{tab:full_results_coco}, \ref{tab:full_results_coco_wordnet}, \ref{tab:full_results_coco_gemini}, and \ref{tab:full_results_coco_gpt}, respectively. In all settings, OSP consistently improves the separation between present and absent concepts even in these complex multi-object scenes. The hyperparameters used for these experiments are detailed in Table~\ref{tab:all_hyperparameters_coco}.

\begin{table}[h!]
\centering
\caption{\textbf{Average performance per dictionary strategy} on MS COCO 2017. Values are averaged over all 9 method--model pairs (4 discriminative methods $\times$ 2 models + DAAM). $\Delta$ columns show the mean change introduced by OSP. For positive prompts (\textcolor{green!60!black}{$\uparrow$} desired), OSP should increase scores; for negative prompts (\textcolor{red!70!black}{$\downarrow$} desired), OSP should decrease them. Favorable changes are highlighted in \textbf{bold}.}
\label{tab:avg_per_dictionary_coco}
\setlength{\tabcolsep}{3pt}
\scriptsize
\resizebox{\textwidth}{!}{%
\begin{tabular}{l@{\hskip 4pt}c@{\hskip 6pt}rr@{\hskip 3pt}r@{\hskip 8pt}rr@{\hskip 3pt}r@{\hskip 8pt}rr@{\hskip 3pt}r}
\toprule
\textbf{Dictionary} & \textbf{Prompt} & \multicolumn{3}{c}{\textbf{mIoU}} & \multicolumn{3}{c}{\textbf{mAP}} & \multicolumn{3}{c}{\textbf{AUROC}}\\
\cmidrule(lr){3-5} \cmidrule(lr){6-8} \cmidrule(lr){9-11}
 & & Base & +OSP & $\Delta$ & Base & +OSP & $\Delta$ & Base & +OSP & $\Delta$ \\
\midrule
\multirow{2}{*}{Classes Only}
 & Pos \textcolor{green!60!black}{$\uparrow$} & 45.62 & 51.59 & \textbf{+5.98} & 88.94 & 90.31 & \textbf{+1.37} & \multirow{2}{*}{ 76.53 } & \multirow{2}{*}{ 82.32 } & \multirow{2}{*}{ \textbf{+5.79 }} \\
 & Neg \textcolor{red!70!black}{$\downarrow$} & 37.43 & 37.12 & \textbf{-0.31} & 82.84 & 82.29 & \textbf{-0.55} & & & \\
\midrule
\multirow{2}{*}{WordNet + Classes}
 & Pos \textcolor{green!60!black}{$\uparrow$} & 45.62 & 51.19 & \textbf{+5.58} & 88.94 & 89.81 & \textbf{+0.86} & \multirow{2}{*}{ 76.53 } & \multirow{2}{*}{ 81.17 } & \multirow{2}{*}{ \textbf{+4.64 }} \\
 & Neg \textcolor{red!70!black}{$\downarrow$} & 37.43 & 37.71 & +0.28 & 82.84 & 82.06 & \textbf{-0.79} & & & \\
\midrule
\multirow{2}{*}{Gemini 3 Flash}
 & Pos \textcolor{green!60!black}{$\uparrow$} & 45.62 & 51.69 & \textbf{+6.08} & 88.94 & 90.00 & \textbf{+1.06} & \multirow{2}{*}{ 76.53 } & \multirow{2}{*}{ 81.00 } & \multirow{2}{*}{ \textbf{+4.47 }} \\
 & Neg \textcolor{red!70!black}{$\downarrow$} & 37.43 & 38.81 & +1.38 & 82.84 & 83.18 & +0.34 & & & \\
\midrule
\multirow{2}{*}{GPT-OSS 120B}
 & Pos \textcolor{green!60!black}{$\uparrow$} & 45.62 & 51.39 & \textbf{+5.77} & 88.94 & 90.12 & \textbf{+1.17} & \multirow{2}{*}{ 76.53 } & \multirow{2}{*}{ 82.02 } & \multirow{2}{*}{ \textbf{+5.48 }} \\
 & Neg \textcolor{red!70!black}{$\downarrow$} & 37.43 & 37.88 & +0.45 & 82.84 & 82.56 & \textbf{-0.28} & & & \\
\bottomrule
\end{tabular}
}
\end{table}

\begin{table}[h]
\centering
\caption{\textbf{Full quantitative results on MS COCO 2017 using the dictionary from Strategy 1 (Classes Only).} All four metrics are shown before (Base) and after (+OSP) applying Orthogonal Semantic Projection, along with deltas ($\Delta$). The Gap Improvement is the cumulative change across metrics. Favorable changes are highlighted in \textbf{bold}.}
\label{tab:full_results_coco}
\resizebox{\textwidth}{!}{%
\setlength{\tabcolsep}{2.5pt}
\scriptsize
\begin{tabular}{ll@{\hskip 3pt}c@{\hskip 4pt}rr@{\hskip 2pt}r@{\hskip 6pt}rr@{\hskip 2pt}r@{\hskip 6pt}rr@{\hskip 2pt}r@{\hskip 6pt}rr@{\hskip 2pt}r@{\hskip 6pt}r}
\toprule
\textbf{Method} & \textbf{Model} & \textbf{Pr.} & \multicolumn{3}{c}{\textbf{mIoU}} & \multicolumn{3}{c}{\textbf{Accuracy}} & \multicolumn{3}{c}{\textbf{mAP}} & \multicolumn{3}{c}{\textbf{AUROC}} & \textbf{Gap} \\
\cmidrule(lr){4-6} \cmidrule(lr){7-9} \cmidrule(lr){10-12} \cmidrule(lr){13-15}
 & & & B & +O & $\Delta$ & B & +O & $\Delta$ & B & +O & $\Delta$ & B & +O & $\Delta$ & \textbf{Impr.} \\
\midrule
\multirow{4}{*}{LeGrad}
 & \multirow{2}{*}{CLIP}   & + & 59.40 & 57.67 & -1.73 & 86.68 & 83.06 & -3.62 & 90.78 & 90.38 & -0.40 & \multirow{2}{*}{87.45} & \multirow{2}{*}{85.67} & \multirow{2}{*}{-1.78} & \multirow{2}{*}{-0.53} \\
 &                         & -- & 39.04 & 37.08 & \textbf{-1.96} & 75.06 & 69.15 & \textbf{-5.91} & 77.00 & 77.87 & +0.87 &  &  &  &  \\
\cmidrule(l){2-16}
 & \multirow{2}{*}{SigLIP} & + & 47.30 & 51.07 & \textbf{+3.77} & 85.16 & 77.18 & -7.98 & 89.37 & 90.15 & \textbf{+0.78} & \multirow{2}{*}{77.81} & \multirow{2}{*}{80.64} & \multirow{2}{*}{\textbf{+2.83}} & \multirow{2}{*}{\textbf{25.46}} \\
 &                         & -- & 41.82 & 34.29 & \textbf{-7.53} & 82.57 & 63.74 & \textbf{-18.83} & 80.21 & 80.51 & +0.30 &  &  &  &  \\
\midrule
\multirow{4}{*}{GradCAM}
 & \multirow{2}{*}{CLIP}   & + & 46.98 & 54.06 & \textbf{+7.08} & 75.45 & 79.60 & \textbf{+4.15} & 88.69 & 91.49 & \textbf{+2.80} & \multirow{2}{*}{77.82} & \multirow{2}{*}{82.86} & \multirow{2}{*}{\textbf{+5.04}} & \multirow{2}{*}{\textbf{18.77}} \\
 &                         & -- & 35.75 & 35.89 & +0.14 & 66.50 & 65.54 & \textbf{-0.96} & 81.91 & 83.03 & +1.12 &  &  &  &  \\
\cmidrule(l){2-16}
 & \multirow{2}{*}{SigLIP} & + & 37.69 & 44.39 & \textbf{+6.70} & 61.73 & 70.50 & \textbf{+8.77} & 87.08 & 87.51 & \textbf{+0.43} & \multirow{2}{*}{69.47} & \multirow{2}{*}{73.60} & \multirow{2}{*}{\textbf{+4.13}} & \multirow{2}{*}{\textbf{11.52}} \\
 &                         & -- & 30.15 & 33.31 & +3.16 & 54.23 & 60.96 & +6.73 & 80.45 & 79.07 & \textbf{-1.38} &  &  &  &  \\
\midrule
\multirow{4}{*}{CheferCAM}
 & \multirow{2}{*}{CLIP}   & + & 45.36 & 53.77 & \textbf{+8.41} & 71.41 & 78.65 & \textbf{+7.24} & 90.52 & 91.91 & \textbf{+1.39} & \multirow{2}{*}{83.16} & \multirow{2}{*}{86.85} & \multirow{2}{*}{\textbf{+3.69}} & \multirow{2}{*}{\textbf{16.07}} \\
 &                         & -- & 40.61 & 42.15 & +1.54 & 67.56 & 70.26 & +2.70 & 88.40 & 88.82 & +0.42 &  &  &  &  \\
\cmidrule(l){2-16}
 & \multirow{2}{*}{SigLIP} & + & 39.92 & 45.54 & \textbf{+5.62} & 68.19 & 78.23 & \textbf{+10.04} & 87.33 & 88.88 & \textbf{+1.55} & \multirow{2}{*}{71.77} & \multirow{2}{*}{77.37} & \multirow{2}{*}{\textbf{+5.60}} & \multirow{2}{*}{\textbf{7.40}} \\
 &                         & -- & 37.86 & 42.83 & +4.97 & 66.39 & 75.99 & +9.60 & 85.56 & 86.40 & +0.84 &  &  &  &  \\
\midrule
\multirow{4}{*}{Att.CAM}
 & \multirow{2}{*}{CLIP}   & + & 44.65 & 58.37 & \textbf{+13.72} & 78.10 & 83.29 & \textbf{+5.19} & 85.13 & 91.00 & \textbf{+5.87} & \multirow{2}{*}{53.84} & \multirow{2}{*}{86.26} & \multirow{2}{*}{\textbf{+32.42}} & \multirow{2}{*}{\textbf{71.24}} \\
 &                         & -- & 40.24 & 37.10 & \textbf{-3.14} & 74.66 & 68.87 & \textbf{-5.79} & 83.49 & 78.38 & \textbf{-5.11} &  &  &  &  \\
\cmidrule(l){2-16}
 & \multirow{2}{*}{SigLIP} & + & 45.68 & 47.08 & \textbf{+1.40} & 71.28 & 71.53 & \textbf{+0.25} & 90.16 & 89.83 & -0.33 & \multirow{2}{*}{81.01} & \multirow{2}{*}{82.06} & \multirow{2}{*}{\textbf{+1.05}} & \multirow{2}{*}{\textbf{8.45}} \\
 &                         & -- & 33.50 & 31.78 & \textbf{-1.72} & 60.86 & 57.87 & \textbf{-2.99} & 81.45 & 80.08 & \textbf{-1.37} &  &  &  &  \\
\midrule
\multirow{2}{*}{DAAM}
 & \multirow{2}{*}{SD 2}   & + & 43.58 & 52.39 & \textbf{+8.81} & 66.72 & 77.48 & \textbf{+10.76} & 91.43 & 91.64 & \textbf{+0.21} & \multirow{2}{*}{86.47} & \multirow{2}{*}{85.56} & \multirow{2}{*}{-0.91} & \multirow{2}{*}{\textbf{11.65}} \\
 &                         & -- & 37.86 & 39.64 & +1.78 & 61.47 & 67.53 & +6.06 & 87.10 & 86.48 & \textbf{-0.62} &  &  &  &  \\
\bottomrule
\end{tabular}}
\end{table}

\begin{table}[h]
\centering
\caption{\textbf{Full quantitative results on MS COCO 2017 using the dictionary from Strategy 2 (WordNet + Classes).} All four metrics are shown before (Base) and after (+OSP) applying Orthogonal Semantic Projection, along with deltas ($\Delta$). The Gap Improvement is the cumulative change across metrics. Favorable changes are highlighted in \textbf{bold}.}
\label{tab:full_results_coco_wordnet}
\resizebox{\textwidth}{!}{%
\setlength{\tabcolsep}{2.5pt}
\scriptsize
\begin{tabular}{ll@{\hskip 3pt}c@{\hskip 4pt}rr@{\hskip 2pt}r@{\hskip 6pt}rr@{\hskip 2pt}r@{\hskip 6pt}rr@{\hskip 2pt}r@{\hskip 6pt}rr@{\hskip 2pt}r@{\hskip 6pt}r}
\toprule
\textbf{Method} & \textbf{Model} & \textbf{Pr.} & \multicolumn{3}{c}{\textbf{mIoU}} & \multicolumn{3}{c}{\textbf{Accuracy}} & \multicolumn{3}{c}{\textbf{mAP}} & \multicolumn{3}{c}{\textbf{AUROC}} & \textbf{Gap} \\
\cmidrule(lr){4-6} \cmidrule(lr){7-9} \cmidrule(lr){10-12} \cmidrule(lr){13-15}
 & & & B & +O & $\Delta$ & B & +O & $\Delta$ & B & +O & $\Delta$ & B & +O & $\Delta$ & \textbf{Impr.} \\
\midrule
\multirow{4}{*}{LeGrad}
 & \multirow{2}{*}{CLIP}   & + & 59.40 & 56.48 & -2.92 & 86.68 & 82.92 & -3.76 & 90.78 & 89.71 & -1.07 & \multirow{2}{*}{87.45} & \multirow{2}{*}{83.80} & \multirow{2}{*}{-3.65} & \multirow{2}{*}{-4.00} \\
 &                         & -- & 39.04 & 37.19 & \textbf{-1.85} & 75.06 & 69.42 & \textbf{-5.64} & 77.00 & 77.09 & +0.09 &  &  &  &  \\
\cmidrule(l){2-16}
 & \multirow{2}{*}{SigLIP} & + & 47.30 & 52.83 & \textbf{+5.53} & 85.16 & 82.39 & -2.77 & 89.37 & 88.84 & -0.53 & \multirow{2}{*}{77.81} & \multirow{2}{*}{78.33} & \multirow{2}{*}{\textbf{+0.52}} & \multirow{2}{*}{\textbf{17.15}} \\
 &                         & -- & 41.82 & 38.28 & \textbf{-3.54} & 82.57 & 72.94 & \textbf{-9.63} & 80.21 & 78.98 & \textbf{-1.23} &  &  &  &  \\
\midrule
\multirow{4}{*}{GradCAM}
 & \multirow{2}{*}{CLIP}   & + & 46.98 & 52.08 & \textbf{+5.10} & 75.45 & 77.86 & \textbf{+2.41} & 88.69 & 91.37 & \textbf{+2.68} & \multirow{2}{*}{77.82} & \multirow{2}{*}{83.33} & \multirow{2}{*}{\textbf{+5.51}} & \multirow{2}{*}{\textbf{17.29}} \\
 &                         & -- & 35.75 & 35.31 & \textbf{-0.44} & 66.50 & 64.76 & \textbf{-1.74} & 81.91 & 82.50 & +0.59 &  &  &  &  \\
\cmidrule(l){2-16}
 & \multirow{2}{*}{SigLIP} & + & 37.69 & 40.17 & \textbf{+2.48} & 61.73 & 64.37 & \textbf{+2.64} & 87.08 & 87.77 & \textbf{+0.69} & \multirow{2}{*}{69.47} & \multirow{2}{*}{72.10} & \multirow{2}{*}{\textbf{+2.63}} & \multirow{2}{*}{\textbf{8.07}} \\
 &                         & -- & 30.15 & 30.32 & +0.17 & 54.23 & 54.66 & +0.43 & 80.45 & 80.22 & \textbf{-0.23} &  &  &  &  \\
\midrule
\multirow{4}{*}{CheferCAM}
 & \multirow{2}{*}{CLIP}   & + & 45.36 & 54.54 & \textbf{+9.18} & 71.41 & 81.17 & \textbf{+9.76} & 90.52 & 91.64 & \textbf{+1.12} & \multirow{2}{*}{83.16} & \multirow{2}{*}{86.16} & \multirow{2}{*}{\textbf{+3.00}} & \multirow{2}{*}{\textbf{13.33}} \\
 &                         & -- & 40.61 & 43.43 & +2.82 & 67.56 & 73.85 & +6.29 & 88.40 & 89.02 & +0.62 &  &  &  &  \\
\cmidrule(l){2-16}
 & \multirow{2}{*}{SigLIP} & + & 39.92 & 45.67 & \textbf{+5.75} & 68.19 & 78.31 & \textbf{+10.12} & 87.33 & 88.70 & \textbf{+1.37} & \multirow{2}{*}{71.77} & \multirow{2}{*}{76.47} & \multirow{2}{*}{\textbf{+4.70}} & \multirow{2}{*}{\textbf{7.54}} \\
 &                         & -- & 37.86 & 42.27 & +4.41 & 66.39 & 75.65 & +9.26 & 85.56 & 86.29 & +0.73 &  &  &  &  \\
\midrule
\multirow{4}{*}{Att.CAM}
 & \multirow{2}{*}{CLIP}   & + & 44.65 & 54.93 & \textbf{+10.28} & 78.10 & 81.77 & \textbf{+3.67} & 85.13 & 88.73 & \textbf{+3.60} & \multirow{2}{*}{53.84} & \multirow{2}{*}{80.51} & \multirow{2}{*}{\textbf{+26.67}} & \multirow{2}{*}{\textbf{56.57}} \\
 &                         & -- & 40.24 & 37.89 & \textbf{-2.35} & 74.66 & 70.86 & \textbf{-3.80} & 83.49 & 77.29 & \textbf{-6.20} &  &  &  &  \\
\cmidrule(l){2-16}
 & \multirow{2}{*}{SigLIP} & + & 45.68 & 47.88 & \textbf{+2.20} & 71.28 & 73.09 & \textbf{+1.81} & 90.16 & 89.01 & -1.15 & \multirow{2}{*}{81.01} & \multirow{2}{*}{81.79} & \multirow{2}{*}{\textbf{+0.78}} & \multirow{2}{*}{\textbf{6.41}} \\
 &                         & -- & 33.50 & 32.89 & \textbf{-0.61} & 60.86 & 60.50 & \textbf{-0.36} & 81.45 & 79.65 & \textbf{-1.80} &  &  &  &  \\
\midrule
\multirow{2}{*}{DAAM}
 & \multirow{2}{*}{SD 2}   & + & 43.58 & 56.17 & \textbf{+12.59} & 66.72 & 80.72 & \textbf{+14.00} & 91.43 & 92.50 & \textbf{+1.07} & \multirow{2}{*}{86.47} & \multirow{2}{*}{88.05} & \multirow{2}{*}{\textbf{+1.58}} & \multirow{2}{*}{\textbf{16.25}} \\
 &                         & -- & 37.86 & 41.80 & +3.94 & 61.47 & 70.16 & +8.69 & 87.10 & 87.46 & +0.36 &  &  &  &  \\
\bottomrule
\end{tabular}}
\end{table}

\begin{table}[h]
\centering
\caption{\textbf{Full quantitative results on MS COCO 2017 using the dictionary from Strategy 3 (Gemini Generated).} All four metrics are shown before (Base) and after (+OSP) applying Orthogonal Semantic Projection, along with deltas ($\Delta$). The Gap Improvement is the cumulative change across metrics. Favorable changes are highlighted in \textbf{bold}.}
\label{tab:full_results_coco_gemini}
\resizebox{\textwidth}{!}{%
\setlength{\tabcolsep}{2.5pt}
\scriptsize
\begin{tabular}{ll@{\hskip 3pt}c@{\hskip 4pt}rr@{\hskip 2pt}r@{\hskip 6pt}rr@{\hskip 2pt}r@{\hskip 6pt}rr@{\hskip 2pt}r@{\hskip 6pt}rr@{\hskip 2pt}r@{\hskip 6pt}r}
\toprule
\textbf{Method} & \textbf{Model} & \textbf{Pr.} & \multicolumn{3}{c}{\textbf{mIoU}} & \multicolumn{3}{c}{\textbf{Accuracy}} & \multicolumn{3}{c}{\textbf{mAP}} & \multicolumn{3}{c}{\textbf{AUROC}} & \textbf{Gap} \\
\cmidrule(lr){4-6} \cmidrule(lr){7-9} \cmidrule(lr){10-12} \cmidrule(lr){13-15}
 & & & B & +O & $\Delta$ & B & +O & $\Delta$ & B & +O & $\Delta$ & B & +O & $\Delta$ & \textbf{Impr.} \\
\midrule
\multirow{4}{*}{LeGrad}
 & \multirow{2}{*}{CLIP}   & + & 59.40 & 57.85 & -1.55 & 86.68 & 83.45 & -3.23 & 90.78 & 90.54 & -0.24 & \multirow{2}{*}{87.45} & \multirow{2}{*}{86.04} & \multirow{2}{*}{-1.41} & \multirow{2}{*}{-3.04} \\
 &                         & -- & 39.04 & 37.96 & \textbf{-1.08} & 75.06 & 70.26 & \textbf{-4.80} & 77.00 & 79.49 & +2.49 &  &  &  &  \\
\cmidrule(l){2-16}
 & \multirow{2}{*}{SigLIP} & + & 47.30 & 50.02 & \textbf{+2.72} & 85.16 & 75.91 & -9.25 & 89.37 & 89.68 & \textbf{+0.31} & \multirow{2}{*}{77.81} & \multirow{2}{*}{79.85} & \multirow{2}{*}{\textbf{+2.04}} & \multirow{2}{*}{\textbf{21.72}} \\
 &                         & -- & 41.82 & 34.65 & \textbf{-7.17} & 82.57 & 62.82 & \textbf{-19.75} & 80.21 & 81.23 & +1.02 &  &  &  &  \\
\midrule
\multirow{4}{*}{GradCAM}
 & \multirow{2}{*}{CLIP}   & + & 46.98 & 52.58 & \textbf{+5.60} & 75.45 & 78.88 & \textbf{+3.43} & 88.69 & 91.07 & \textbf{+2.38} & \multirow{2}{*}{77.82} & \multirow{2}{*}{82.84} & \multirow{2}{*}{\textbf{+5.02}} & \multirow{2}{*}{\textbf{14.12}} \\
 &                         & -- & 35.75 & 36.25 & +0.50 & 66.50 & 66.40 & \textbf{-0.10} & 81.91 & 83.82 & +1.91 &  &  &  &  \\
\cmidrule(l){2-16}
 & \multirow{2}{*}{SigLIP} & + & 37.69 & 41.28 & \textbf{+3.59} & 61.73 & 66.66 & \textbf{+4.93} & 87.08 & 87.03 & -0.05 & \multirow{2}{*}{69.47} & \multirow{2}{*}{70.07} & \multirow{2}{*}{\textbf{+0.41}} & \multirow{2}{*}{\textbf{2.81}} \\
 &                         & -- & 30.15 & 32.27 & +2.12 & 54.23 & 58.17 & +3.94 & 80.45 & 80.46 & +0.01 &  &  &  &  \\
\midrule
\multirow{4}{*}{CheferCAM}
 & \multirow{2}{*}{CLIP}   & + & 45.36 & 52.33 & \textbf{+6.97} & 71.41 & 77.74 & \textbf{+6.33} & 90.52 & 91.67 & \textbf{+1.15} & \multirow{2}{*}{83.16} & \multirow{2}{*}{85.79} & \multirow{2}{*}{\textbf{+2.63}} & \multirow{2}{*}{\textbf{10.33}} \\
 &                         & -- & 40.61 & 43.10 & +2.49 & 67.56 & 71.04 & +3.48 & 88.40 & 89.18 & +0.78 &  &  &  &  \\
\cmidrule(l){2-16}
 & \multirow{2}{*}{SigLIP} & + & 39.92 & 45.55 & \textbf{+5.63} & 68.19 & 78.10 & \textbf{+9.91} & 87.33 & 88.81 & \textbf{+1.48} & \multirow{2}{*}{71.77} & \multirow{2}{*}{76.58} & \multirow{2}{*}{\textbf{+4.81}} & \multirow{2}{*}{\textbf{7.19}} \\
 &                         & -- & 37.86 & 42.47 & +4.61 & 66.39 & 75.53 & +9.14 & 85.56 & 86.45 & +0.89 &  &  &  &  \\
\midrule
\multirow{4}{*}{Att.CAM}
 & \multirow{2}{*}{CLIP}   & + & 44.65 & 57.90 & \textbf{+13.25} & 78.10 & 87.09 & \textbf{+8.99} & 85.13 & 90.16 & \textbf{+5.03} & \multirow{2}{*}{53.84} & \multirow{2}{*}{81.48} & \multirow{2}{*}{\textbf{+27.64}} & \multirow{2}{*}{\textbf{55.36}} \\
 &                         & -- & 40.24 & 40.72 & +0.48 & 74.66 & 77.73 & +3.07 & 83.49 & 79.49 & \textbf{-4.00} &  &  &  &  \\
\cmidrule(l){2-16}
 & \multirow{2}{*}{SigLIP} & + & 45.68 & 51.90 & \textbf{+6.22} & 71.28 & 78.69 & \textbf{+7.41} & 90.16 & 89.66 & -0.50 & \multirow{2}{*}{81.01} & \multirow{2}{*}{80.10} & \multirow{2}{*}{-0.91} & \multirow{2}{*}{\textbf{1.95}} \\
 &                         & -- & 33.50 & 37.01 & +3.51 & 60.86 & 67.72 & +6.86 & 81.45 & 81.35 & \textbf{-0.10} &  &  &  &  \\
\midrule
\multirow{2}{*}{DAAM}
 & \multirow{2}{*}{SD 2}   & + & 43.58 & 55.83 & \textbf{+12.25} & 66.72 & 84.08 & \textbf{+17.36} & 91.43 & 91.42 & -0.01 & \multirow{2}{*}{86.47} & \multirow{2}{*}{86.24} & \multirow{2}{*}{-0.23} & \multirow{2}{*}{\textbf{6.39}} \\
 &                         & -- & 37.86 & 44.82 & +6.96 & 61.47 & 77.46 & +15.99 & 87.10 & 87.13 & +0.03 &  &  &  &  \\
\bottomrule
\end{tabular}}
\end{table}

\begin{table}[h]
\centering
\caption{\textbf{Full quantitative results on MS COCO 2017 using the dictionary from Strategy 4 (GPT Generated).} All four metrics are shown before (Base) and after (+OSP) applying Orthogonal Semantic Projection, along with deltas ($\Delta$). The Gap Improvement is the cumulative change across metrics. Favorable changes are highlighted in \textbf{bold}.}
\label{tab:full_results_coco_gpt}
\resizebox{\textwidth}{!}{%
\setlength{\tabcolsep}{2.5pt}
\scriptsize
\begin{tabular}{ll@{\hskip 3pt}c@{\hskip 4pt}rr@{\hskip 2pt}r@{\hskip 6pt}rr@{\hskip 2pt}r@{\hskip 6pt}rr@{\hskip 2pt}r@{\hskip 6pt}rr@{\hskip 2pt}r@{\hskip 6pt}r}
\toprule
\textbf{Method} & \textbf{Model} & \textbf{Pr.} & \multicolumn{3}{c}{\textbf{mIoU}} & \multicolumn{3}{c}{\textbf{Accuracy}} & \multicolumn{3}{c}{\textbf{mAP}} & \multicolumn{3}{c}{\textbf{AUROC}} & \textbf{Gap} \\
\cmidrule(lr){4-6} \cmidrule(lr){7-9} \cmidrule(lr){10-12} \cmidrule(lr){13-15}
 & & & B & +O & $\Delta$ & B & +O & $\Delta$ & B & +O & $\Delta$ & B & +O & $\Delta$ & \textbf{Impr.} \\
\midrule
\multirow{4}{*}{LeGrad}
 & \multirow{2}{*}{CLIP}   & + & 59.40 & 56.66 & -2.74 & 86.68 & 82.31 & -4.37 & 90.78 & 90.32 & -0.46 & \multirow{2}{*}{87.45} & \multirow{2}{*}{87.70} & \multirow{2}{*}{\textbf{+0.25}} & \multirow{2}{*}{\textbf{2.89}} \\
 &                         & -- & 39.04 & 36.06 & \textbf{-2.98} & 75.06 & 67.83 & \textbf{-7.23} & 77.00 & 77.00 & 0.00 &  &  &  &  \\
\cmidrule(l){2-16}
 & \multirow{2}{*}{SigLIP} & + & 47.30 & 52.46 & \textbf{+5.16} & 85.16 & 83.34 & -1.82 & 89.37 & 89.12 & -0.25 & \multirow{2}{*}{77.81} & \multirow{2}{*}{78.48} & \multirow{2}{*}{\textbf{+0.67}} & \multirow{2}{*}{\textbf{10.38}} \\
 &                         & -- & 41.82 & 40.62 & \textbf{-1.20} & 82.57 & 76.17 & \textbf{-6.40} & 80.21 & 81.19 & +0.98 &  &  &  &  \\
\midrule
\multirow{4}{*}{GradCAM}
 & \multirow{2}{*}{CLIP}   & + & 46.98 & 53.79 & \textbf{+6.81} & 75.45 & 80.18 & \textbf{+4.73} & 88.69 & 91.20 & \textbf{+2.51} & \multirow{2}{*}{77.82} & \multirow{2}{*}{83.65} & \multirow{2}{*}{\textbf{+5.83}} & \multirow{2}{*}{\textbf{15.19}} \\
 &                         & -- & 35.75 & 37.16 & +1.41 & 66.50 & 67.75 & +1.25 & 81.91 & 83.94 & +2.03 &  &  &  &  \\
\cmidrule(l){2-16}
 & \multirow{2}{*}{SigLIP} & + & 37.69 & 42.56 & \textbf{+4.87} & 61.73 & 67.57 & \textbf{+5.84} & 87.08 & 87.58 & \textbf{+0.50} & \multirow{2}{*}{69.47} & \multirow{2}{*}{72.63} & \multirow{2}{*}{\textbf{+3.16}} & \multirow{2}{*}{\textbf{6.18}} \\
 &                         & -- & 30.15 & 33.44 & +3.29 & 54.23 & 59.21 & +4.98 & 80.45 & 80.37 & \textbf{-0.08} &  &  &  &  \\
\midrule
\multirow{4}{*}{CheferCAM}
 & \multirow{2}{*}{CLIP}   & + & 45.36 & 54.68 & \textbf{+9.32} & 71.41 & 81.18 & \textbf{+9.77} & 90.52 & 91.60 & \textbf{+1.08} & \multirow{2}{*}{83.16} & \multirow{2}{*}{86.34} & \multirow{2}{*}{\textbf{+3.18}} & \multirow{2}{*}{\textbf{13.17}} \\
 &                         & -- & 40.61 & 43.67 & +3.06 & 67.56 & 74.06 & +6.50 & 88.40 & 89.02 & +0.62 &  &  &  &  \\
\cmidrule(l){2-16}
 & \multirow{2}{*}{SigLIP} & + & 39.92 & 44.96 & \textbf{+5.04} & 68.19 & 78.26 & \textbf{+10.07} & 87.33 & 88.17 & \textbf{+0.84} & \multirow{2}{*}{71.77} & \multirow{2}{*}{74.88} & \multirow{2}{*}{\textbf{+3.11}} & \multirow{2}{*}{\textbf{4.50}} \\
 &                         & -- & 37.86 & 42.04 & +4.18 & 66.39 & 76.06 & +9.67 & 85.56 & 86.27 & +0.71 &  &  &  &  \\
\midrule
\multirow{4}{*}{Att.CAM}
 & \multirow{2}{*}{CLIP}   & + & 44.65 & 58.14 & \textbf{+13.49} & 78.10 & 83.09 & \textbf{+4.99} & 85.13 & 91.13 & \textbf{+6.00} & \multirow{2}{*}{53.84} & \multirow{2}{*}{86.52} & \multirow{2}{*}{\textbf{+32.68}} & \multirow{2}{*}{\textbf{69.53}} \\
 &                         & -- & 40.24 & 37.65 & \textbf{-2.59} & 74.66 & 69.15 & \textbf{-5.51} & 83.49 & 79.22 & \textbf{-4.27} &  &  &  &  \\
\cmidrule(l){2-16}
 & \multirow{2}{*}{SigLIP} & + & 45.68 & 53.90 & \textbf{+8.22} & 71.28 & 80.83 & \textbf{+9.55} & 90.16 & 89.81 & -0.35 & \multirow{2}{*}{81.01} & \multirow{2}{*}{81.23} & \multirow{2}{*}{\textbf{+0.22}} & \multirow{2}{*}{\textbf{7.81}} \\
 &                         & -- & 33.50 & 36.85 & +3.35 & 60.86 & 68.73 & +7.87 & 81.45 & 80.06 & \textbf{-1.39} &  &  &  &  \\
\midrule
\multirow{2}{*}{DAAM}
 & \multirow{2}{*}{SD 2}   & + & 43.58 & 45.36 & \textbf{+1.78} & 66.72 & 68.18 & \textbf{+1.46} & 91.43 & 92.12 & \textbf{+0.69} & \multirow{2}{*}{86.47} & \multirow{2}{*}{86.71} & \multirow{2}{*}{\textbf{+0.22}} & \multirow{2}{*}{\textbf{14.72}} \\
 &                         & -- & 37.86 & 33.41 & \textbf{-4.45} & 61.47 & 56.53 & \textbf{-4.94} & 87.10 & 85.94 & \textbf{-1.16} &  &  &  &  \\
\bottomrule
\end{tabular}}
\end{table}

Tables~\ref{tab:all_hyperparameters} and~\ref{tab:all_hyperparameters_coco} show the hyperparameters used for Orthogonal Semantic Projection across all datasets and dictionary strategies evaluated in this work. For each method--model pair, we report the heatmap binarization threshold ($\tau_{\text{act}}$), number of OMP atoms, (= number of semantics), ($T$), and maximum cosine similarity threshold ($\tau_{\text{cos}}$). DAAM additionally uses the OMP substitution weight ($\beta_{\text{omp}}$), introduced in Section~\ref{sec:extended_quantitative_results}, reported below each table.

\begin{table}[htbp]
\centering
\caption{\textbf{Consolidated hyperparameters on ImageNet-Segmentation} across all four dictionary strategies.}
\label{tab:all_hyperparameters}
\resizebox{\textwidth}{!}{%
\setlength{\tabcolsep}{3pt}
\scriptsize
\begin{tabular}{ll ccc ccc ccc ccc}
\toprule
 & & \multicolumn{3}{c}{\textbf{S1: Classes}} & \multicolumn{3}{c}{\textbf{S2: WordNet}} & \multicolumn{3}{c}{\textbf{S3: Gemini}} & \multicolumn{3}{c}{\textbf{S4: GPT-OSS}} \\
\cmidrule(lr){3-5} \cmidrule(lr){6-8} \cmidrule(lr){9-11} \cmidrule(lr){12-14}
\textbf{Method} & \textbf{Model} & $\tau_{\text{act}}$ & $T$ & $\tau_{\text{cos}}$ & $\tau_{\text{act}}$ & $T$ & $\tau_{\text{cos}}$ & $\tau_{\text{act}}$ & $T$ & $\tau_{\text{cos}}$ & $\tau_{\text{act}}$ & $T$ & $\tau_{\text{cos}}$ \\
\midrule
\multirow{2}{*}{LeGrad} & CLIP & 0.4 & 3 & 0.55 & 0.4 & 5 & 0.5 & 0.4 & 27 & 0.6 & 0.35 & 29 & 0.7 \\
 & SigLIP & 0.375 & 6 & 0.5 & 0.375 & 11 & 0.55 & 0.4 & 23 & 0.65 & 0.425 & 15 & 0.7 \\
\midrule
\multirow{2}{*}{CheferCAM} & CLIP & 0.15 & 4 & 0.55 & 0.125 & 5 & 0.5 & 0.175 & 4 & 0.7 & 0.175 & 10 & 0.85 \\
 & SigLIP & 0.1 & 6 & 0.9 & 0.1 & 5 & 0.9 & 0.1 & 4 & 0.85 & 0.125 & 19 & 0.8 \\
\midrule
\multirow{2}{*}{Att.CAM} & CLIP & 0.1 & 17 & 0.5 & 0.1 & 6 & 0.5 & 0.1 & 20 & 0.8 & 0.1 & 28 & 1.0 \\
 & SigLIP & 0.125 & 5 & 0.55 & 0.1 & 8 & 0.55 & 0.125 & 28 & 0.7 & 0.15 & 7 & 0.75 \\
\midrule
\multirow{2}{*}{GradCAM} & CLIP & 0.1 & 5 & 0.5 & 0.1 & 17 & 0.5 & 0.1 & 26 & 0.65 & 0.125 & 32 & 0.85 \\
 & SigLIP & 0.5 & 20 & 0.5 & 0.5 & 16 & 0.5 & 0.475 & 26 & 0.75 & 0.45 & 23 & 0.75 \\
\midrule
DAAM & SD 2 & 0.225 & 17 & 0.6 & 0.325 & 20 & 0.95 & 0.4 & 24 & 0.95 & 0.4 & 19 & 0.7 \\
\bottomrule
\end{tabular}
}
\vspace{2pt}
\par\noindent{\scriptsize DAAM $\beta_{\text{omp}}$: S1\,=\,0.275, S2\,=\,0.15, S3\,=\,0.1, S4\,=\,0.1.}
\end{table}

\begin{table}[htbp]
\centering
\caption{\textbf{Consolidated hyperparameters on MS COCO 2017} across all four dictionary strategies.}
\label{tab:all_hyperparameters_coco}
\resizebox{\textwidth}{!}{%
\setlength{\tabcolsep}{3pt}
\scriptsize
\begin{tabular}{ll ccc ccc ccc ccc}
\toprule
 & & \multicolumn{3}{c}{\textbf{S1: Classes}} & \multicolumn{3}{c}{\textbf{S2: WordNet}} & \multicolumn{3}{c}{\textbf{S3: Gemini}} & \multicolumn{3}{c}{\textbf{S4: GPT-OSS}} \\
\cmidrule(lr){3-5} \cmidrule(lr){6-8} \cmidrule(lr){9-11} \cmidrule(lr){12-14}
\textbf{Method} & \textbf{Model} & $\tau_{\text{act}}$ & $T$ & $\tau_{\text{cos}}$ & $\tau_{\text{act}}$ & $T$ & $\tau_{\text{cos}}$ & $\tau_{\text{act}}$ & $T$ & $\tau_{\text{cos}}$ & $\tau_{\text{act}}$ & $T$ & $\tau_{\text{cos}}$ \\
\midrule
\multirow{2}{*}{LeGrad} & CLIP & 0.4 & 13 & 0.8 & 0.425 & 27 & 0.75 & 0.425 & 25 & 0.75 & 0.4 & 2 & 0.55 \\
 & SigLIP & 0.275 & 8 & 0.9 & 0.425 & 27 & 0.75 & 0.275 & 31 & 0.8 & 0.425 & 27 & 0.85 \\
\midrule
\multirow{2}{*}{GradCAM} & CLIP & 0.125 & 28 & 0.75 & 0.15 & 19 & 0.65 & 0.15 & 9 & 0.75 & 0.15 & 21 & 0.75 \\
 & SigLIP & 0.25 & 18 & 0.8 & 0.15 & 15 & 0.65 & 0.15 & 19 & 0.6 & 0.225 & 24 & 0.8 \\
\midrule
\multirow{2}{*}{CheferCAM} & CLIP & 0.1 & 14 & 0.8 & 0.125 & 11 & 0.65 & 0.1 & 18 & 0.8 & 0.125 & 17 & 0.75 \\
 & SigLIP & 0.1 & 8 & 0.85 & 0.1 & 14 & 0.85 & 0.1 & 16 & 0.9 & 0.1 & 2 & 0.85 \\
\midrule
\multirow{2}{*}{Att.CAM} & CLIP & 0.4 & 20 & 0.7 & 0.425 & 14 & 0.6 & 0.55 & 19 & 0.75 & 0.4 & 12 & 0.75 \\
 & SigLIP & 0.25 & 12 & 0.8 & 0.275 & 22 & 0.5 & 0.3 & 6 & 0.85 & 0.35 & 18 & 0.8 \\
\midrule
DAAM & SD 2 & 0.25 & 2 & 0.7 & 0.3 & 27 & 0.85 & 0.425 & 2 & 1.0 & 0.125 & 7 & 0.8 \\
\bottomrule
\end{tabular}
}
\vspace{2pt}
\par\noindent{\scriptsize DAAM $\beta_{\text{omp}}$: S1\,=\,0.5, S2\,=\,0.1, S3\,=\,0.1, S4\,=\,0.1.}
\end{table}